\def\year{2020}\relax
		\documentclass[letterpaper]{article} 
		\usepackage{aaai20}  
		\usepackage{times}  
		\usepackage{helvet} 
		\usepackage{courier}  
		\usepackage[hyphens]{url}  
		\usepackage{graphicx} 
		\usepackage{array}
		\usepackage{graphicx}  
		\usepackage{longtable}
		\usepackage[shortlabels]{enumitem}
		
		\usepackage{dsfont}
		\usepackage{amsmath,amsthm, amssymb}
		\usepackage{mathrsfs}
		\usepackage{graphicx}
		\usepackage[colorinlistoftodos]{todonotes}
		\newcommand{\citet}[1]{\citeauthor{#1}\shortcite{#1}}
		\newcommand{\citep}{\cite}

		\newcommand{\sX}{\mathscr{X}}
		\newcommand{\sY}{\mathscr{Y}}
		\newcommand{\sZ}{\mathscr{Z}}
		\renewcommand{\Pr}{\mathbb{P}}
		\newcommand{\Exp}{\mathbb{E}}
		\newcommand{\dReal}{\Real^d}
		\newcommand{\Int}{\mathbb{Z}}
		\newcommand{\indc}{\mathds{1}}
		\newcommand{\Id}{I}
		\newcommand{\cW}{\mathcal{W}}
		\newcommand{\RInf}{\mathbb{R}^\infty_{+}}
		\newcommand{\cI}{\mathcal{I}}
		\newcommand{\Real}{\mathbb{R}}
		\newcommand{\df}{\textnormal{d}}
		\newcommand{\tr}{\top}
		
		\newcommand{\uw}{u}
		\newcommand{\thS}{\theta^*}
		\newcommand{\wS}{w^*}
		\newcommand{\Rnt}{R_{n + 1}^{(\theta)}}
		\newcommand{\Rntp}{R_{n}^{(\theta)}}
		\newcommand{\Lnt}{L_{n + 1}^{(\theta)}}
		\newcommand{\Dnt}{\Delta_{n + 1}^{(\theta)}}
		\newcommand{\Rnw}{R_{n + 1}^{(w)}}
		\newcommand{\Lnw}{L_{n + 1}^{(w)}}
		\newcommand{\Dnw}{\Delta_{n + 1}^{(w)}}
		\newcommand{\Rproj}{R_{\text{proj}}}
		\newcommand{\Rprojth}{R_{\text{proj}}^\theta}
		\newcommand{\Rprojw}{R_{\text{proj}}^w}
		\newcommand{\cE}{\mathcal{E}}
		
		\newcommand{\Mt}{M^{(1)}}
		\newcommand{\Mw}{M^{(2)}}
		\newcommand{\cF}{\mathcal{F}}

		\newcommand{\mw}{m_2}
		\newcommand{\mt}{m_1}
		\newcommand{\norm}[1]{\left\lVert#1\right\rVert}
		\newcommand{\cU}{\mathcal{U}}
		\newcommand{\cG}{\mathcal{G}}
		\newcommand{\cA}{\mathcal{A}}
		\newcommand{\cL}{\mathcal{L}}
		\newcommand{\Gp}{\cG^\prime}
		\newcommand{\Crth}{C_R^\theta}
		\newcommand{\Crw}{C_R^w}
		\newcommand{\cZ}{\mathcal{Z}}
		\newcommand{\bE}{\mathbb{E}}

		\newcommand{\Tt}{\Gamma_1}
		\newcommand{\Tw}{\Gamma_2}
		\newcommand{\Wt}{W_1}
		\newcommand{\Ww}{W_2}
		\newcommand{\Xt}{X_1}
		\usepackage{mathtools}
		\DeclarePairedDelimiter{\ceil}{\lceil}{\rceil}
		
		\newcommand{\epsth}{\epsilon^{(\theta)}}
		\newcommand{\epsw}{\epsilon^{(w)}}

		\newcommand{\st}{\alpha}
		\newcommand{\sw}{\beta}
		
		\newcommand{\assumMatrices}{\noindent \pmb{$\cA_1$}}
		\newcommand{\assumStepsize}{\noindent \pmb{$\cA_2$}}
		\newcommand{\mokkademConvergence}{\noindent \pmb{$\cA_3$}}

		\newtheorem{thm}{Theorem}
		
		\newtheorem{prop}[thm]{Proposition}
		\newtheorem{theorem}[thm]{Theorem}
		\newtheorem{corollary}[thm]{Corollary}
		\newtheorem{remark}[thm]{Remark}
		\newtheorem{lemma}[thm]{Lemma}

		\usepackage{xcolor}

		\newtheorem{definition}[thm]{Definition}
		
		\title{A Tale of Two-Timescale Reinforcement Learning with the Tightest Finite-Time Bound}
		\author{Gal Dalal,\textsuperscript{\rm 1} Bal\'azs Sz\"or\'enyi,\textsuperscript{\rm 2} and Gugan Thoppe\textsuperscript{\rm 3}\thanks{Research supported by NSF grants DEB-1840223 and DMS 17-13012.}\\ 
		\textsuperscript{\rm 1} Technion, Israel Institute of Technology, Haifa, Israel; gald@technion.ac.il
		\\
		\textsuperscript{\rm 2} Yahoo! Research, New York, NY, USA; szorenyi.balazs@gmail.com
		\\
		\textsuperscript{\rm 3} Duke University, Durham, NC,  USA; gugan.thoppe@gmail.com
		}

\begin{document}
		\allowdisplaybreaks
		
		\maketitle
		
		\begin{abstract}
		    Policy evaluation in reinforcement learning is often conducted using two-timescale stochastic approximation, which results in various gradient temporal difference methods such as GTD(0), GTD2, and TDC. Here, we provide convergence rate bounds for this suite of algorithms. Algorithms such as these have two iterates, $\theta_n$ and $w_n,$ which are updated using two distinct stepsize sequences, $\alpha_n$ and $\beta_n,$ respectively. Assuming $\alpha_n = n^{-\alpha}$ and $\beta_n = n^{-\beta}$ with $1 > \alpha > \beta > 0,$ we show that, with high probability, the two iterates converge to their respective solutions $\theta^*$ and $w^*$ at rates given by $\|\theta_n - \theta^*\| = \tilde{O}( n^{-\alpha/2})$ and $\|w_n - w^*\| = \tilde{O}(n^{-\beta/2});$ here, $\tilde{O}$ hides logarithmic terms. Via comparable lower bounds, we show that these bounds are, in fact, tight. To the best of our knowledge, ours is the first finite-time analysis which achieves these rates. While it was known that the two timescale components decouple asymptotically, our results depict this phenomenon more explicitly by showing that it in fact happens from some finite time onwards. Lastly, compared to existing works, our result applies to a broader family of stepsizes, including non-square summable ones.
		\end{abstract}
		
		\section{Introduction}
		Stochastic Approximation (SA) \cite{kushner1997stochatic} is the name given to algorithms useful for finding optimal points or zeros of a function for which only noisy access is available. This makes SA theory vital to machine learning and, specifically, to Reinforcement Learning (RL). Here, we obtain tight convergence rate estimates for the special class of linear two-timescale SA, which involves two interleaved update rules with distinct stepsize sequences. In the context of RL, the analysis here applies to \emph{policy evaluation} schemes with function approximation. 
		
		A generic linear two-timescale SA has the form:
		\begin{eqnarray}
		\theta_{n + 1} &  = & \theta_n + \st_n [h_1(\theta_n, w_n) + \Mt_{n + 1}] \enspace, \label{eqn:tIter}\\
		w_{n + 1} & = & w_n + \sw_n [h_2(\theta_n, w_n) + \Mw_{n + 1}] \enspace,  \label{eqn:wIter}
		\end{eqnarray}
		where $\st_n, \sw_n \in \Real$ are stepsizes and $M^{(i)}_n \in \dReal$ denotes noise. Further, $h_i: \dReal \times \dReal \to \dReal$ has the form
		\begin{equation}
		\label{eqn:lht}
		h_i(\theta, w) = v_i - \Gamma_i \theta - W_i w
		\end{equation}
		for a vector $v_i \in \dReal$ and matrices $\Gamma_i, W_i \in \mathbb{R}^{d \times d}.$ 
		
		Within RL, this class of algorithms mainly concerns the suite of gradient Temporal Difference (TD) methods, which was introduced in \cite{sutton2009convergent} and has gradually gained increasing attention since then. That work presented a gradient descent variant of TD(0), called GTD(0). As it supports off-policy learning, GTD(0) is advantageous over TD(0). More recently, additional variants were introduced such as GTD2 and TDC \cite{sutton2009fast}; while being better than TD(0), these are also faster than GTD(0). The above gradient TD methods have been shown to  converge asymptotically in the case of linear and non-linear function approximation \cite{sutton2009convergent,sutton2009fast,bhatnagar2009convergent}.
		Separately, there are also a few convergence rate results for altered versions of the GTD family \cite{liu2015finite} and sparsely-projected variants \cite{2TSdalal2018}. Both works apply projections to keep the iterates in a confined region around the solutions. However, in \cite{liu2015finite}, the learning rates are set to a fixed ratio which makes the altered algorithms single-timescale variants of the original ones.
		 
		To place our work in the landscape of the existing literature on generic two-timescale SA, we now briefly review a few seminal papers. The first well-known use of the two-timescale idea is the  Polyak-Ruppert averaging scheme \cite{ruppert1988efficient,polyak1990new}. There, iterate averaging is used to improve the convergence rate of a one-timescale algorithm, which is especially beneficial when the driving matrices have poor conditioning. The general two-timecale SA scheme is formulated in \cite{borkar1997stochastic}; this work provided conditions for convergence. Since then, 
		relatively little work has been published on the topic; the main results obtained so far include
		weak convergence and asymptotic convergence rates \cite{gerencser1997rate,konda2004convergence,mokkadem2006convergence}, and stability \cite{lakshminarayanan2017stability}.
		
		We now discuss two specific 
		papers from the above list that are the closest to our work. Denote by $\thS$ and $\wS$ the respective solutions of \eqref{eqn:tIter} and \eqref{eqn:wIter}; i.e., $h_1(\thS, \wS)=h_2(\thS, \wS)=0.$ In \cite{konda2004convergence}, it was shown that both, $ (\theta_n - \thS)/\sqrt{\alpha_n}$ and $(w_n - \wS)/\sqrt{\beta_n},$ are asymptotically normal. This result surprisingly tells us that eventually the two components do not influence each other's convergence rates. However, one of the assumptions there is that the noise sequence is independent of its past values, and their variance-covariance matrices are constant across the iterations. This make their results inapplicable to the RL methods of our interest. In \cite{mokkadem2006convergence}, a similar weak convergence result has been derived in the context of nonlinear SA under the assumptions that the stepsizes are square summable. This result also explicitly establishes asymptotic independence (see (5) there) between the two components. A separate result in this last work is that of almost-sure asymptotic convergence rate. 
		The issue with this last result is that it cannot be used to obtain explicit form for the constants.
		In fact, by its very nature, the constants involved depend on the sample paths. 
		
		In this work, we revisit the convergence rate question for two-timescale RL methods with a focus on finite-time behaviour. In order to highlight the merits of this work over existing literature, we first classify common types of convergence results. The first class is of asymptotic convergence, which is beneficial for the rudimentary verification that an algorithm converges after an infinite amount of time. The second class is asymptotic convergence rates; these are stronger in the sense of telling us that an algorithm would asymptotically converge at a certain rate, but again they have little practical implications; even given exact knowledge of all parameters of the problem, with these results one cannot numerically compute a bound on the distance from the solution with a corresponding numerical probability value. The third class, to which the results in this work belong, are finite time bounds. These contain explicit constants --- both controllable such as stepsize parameters and uncontrollable such as eigenvalues --- as well as finite-time rates, thereby revealing intriguing dependencies among such parameters that crucially affect convergence rates (e.g., $1/q_i;$ see Table~\ref{tab: constants contd}). Moreover, the constants are trajectory-independent and thus 
		can be of help in obtaining stopping time theorems. We consider this a significant step forward in obtaining practical results that would enable 
		to assuredly adapt algorithm parameters so as to maximize their efficiency. 
		
		\subsubsection{Our Contributions}   
		In \cite{2TSdalal2018}, the first finite time bound for the GTD family was proved. Here, we significantly strengthen it and, in fact, obtain a tight rate.
		Specifically, our key result (Theorem~\ref{thm:Rates Proj Iterates}) is that the iterates $\theta'_n$ and $w'_n,$ obtained by sparsely projecting $\theta_n$ and $w_n,$ respectively, satisfy $\|\theta'_n - \thS\| = \tilde{O}( n^{-\alpha/2})$ and $\|w'_n - \wS\| = \tilde{O}(n^{-\beta/2})$ with high probability. Here, $\tilde{O}$ hides logarithmic terms and $\alpha$ and $\beta$ originate in the stepsize choice
		 $\alpha_n = n^{-\alpha}$ and $\beta_n = n^{-\beta}$ with $1 > \alpha > \beta > 0.$ We establish the tightness of this upper bound by deriving a matching lower bound.
		
		We emphasize that we have explicit formulas for the constants hidden in these order notations and also bounds on the iteration index from where these rates apply. In particular, our bound shows how the convergence rate of a given GTD method depends on the parameters of the MDP itself; e.g., the eigenvalues of the driving matrix.
		
		As in \cite{2TSdalal2018} which dealt with single-timescale algorithms, the bounds in this work are applicable for both square-summable and non-square-summable stepsizes. This was indeed also the case in \cite{konda2004convergence}; however, as pointed earlier, the noise assumptions there are significantly stronger than ours.
		
		The sparse projection scheme used here is novel but is similar in spirit to the one used in \cite{2TSdalal2018}. There, the iterates were only projected when the iteration indices were powers of $2,$ whereas here we project whenever the iteration index is of the form $k^k = 2^{k \log_2 k},$ $k \geq 0.$ 
		The motivation for using projections is to keep
		the iterates bounded. However, projections also
		modify the original algorithm by introducing non-linearity. This highly complicates the analysis. Evidently, the literature almost doesn't contain analyses of projected algorithms at all. Moreover,  projections are often empirically found to be unnecessary. The advantages of using a sparse projection scheme is that we effectively almost 
		never project and, more importantly, it makes the analysis 
		oblivious to its non-linearity.
		
		An additional
		novelty of this paper is its  proof technique. At its heart lie two induction tricks--one inspired from \cite{thoppe2019concentration} and the other, being rather non-standard, from \cite{mokkadem2006convergence}. The first
		induction is on the iteration 
		index $n$; together with projections it enables us to show that both $\theta'_n$ and $w'_n$ iterates are $O(1),$ i.e.,  bounded, with high probability. On each sample path where the iterates are bounded, we then use the second 
		induction to show that the convergence rate of the $w'_n$ iterates can be improved from $\tilde{O}(n^{-\beta/2} \indc[\ell \neq 0] + n^{-(\alpha - \beta)\ell})$ to $\tilde{O}(n^{-\beta/2} \indc[\ell \neq 0] + n^{-(\alpha - \beta)(\ell + 1)})$ for all suitable $\ell.$ In particular, we use this to show that the bound on the behaviour of  $w'_n$ iterates can be incrementally improved from $O(1),$ established above, to the desired $\tilde{O}(n^{
		-\beta/2}).$
		%
		%
		Finally, we use this latter result to show that $\|\theta'_n - \thS\| = \tilde{O}(n^{-\alpha/2}).$
		
		We end this section by describing the key insights that our main result in Theorem~\ref{thm:Rates Proj Iterates} provides.
		
		\textbf{Decoupling after Finite Time}: Even though both $\theta'_n$ and $w'_n$ influence each other, our result shows that, from some finite time onwards, their convergence rates do not depend on $\beta$ and $\alpha,$ respectively.  While from the results in \cite{konda2004convergence} and \cite{mokkadem2006convergence}, one would expect the two-timescale components to indeed decouple asymptotically, our result shows that this in fact happens from some finite time 
		that can conceptually be numerically evaluated. All of this is in sharp contrast to the former state-of-the-art finite-time result given in \cite{2TSdalal2018} which showed that the convergence rate is $\tilde{O}(n^{-\min\{\alpha- \beta, \beta/2\}}).$   
		
		\textbf{One vs Two-Timescale}: A natural question 
		for an RL practitioner is 
		whether to run the algorithm given in \eqref{eqn:tIter} and \eqref{eqn:wIter} in the one-timescale mode, i.e., with $\alpha_n/\beta_n$ being constant, or in the two-timescale mode, i.e., with $\alpha_n/\beta_n \to 0.$
		Judging solely on the convergence rate order -- based on this work and on single-timescale results from, e.g., \cite{liu2015finite}, the answer\footnote{The $\alpha_n=\beta_n=1/n$ case above would bring the condition number of the driving matrices into picture {\cite{dalal2018finite}}. To overcome this, one could use Polyak-Ruppert iterate averaging for two-timescale SA \cite{mokkadem2006convergence}.} is to pick the  single timescale mode with
		$\alpha_n = \beta_n \approx 1/n.$
		This then brings forth 
		an imperative question for future work: ``what indeed are the provable benefits of two-timescale RL methods?" A comparison to  recent gradient descent literature suggests that this question can be better answered via iteration complexity, i.e., the the number of iterations required to hit some  $\epsilon-$ball around the solution. In particular, we believe the eigenvalues of the driving matrices  
		%
		--- hiding in the constants --- can have dramatic influence on the actual rate. A predominant recent example is how the heavy-ball method, which is similar in nature to a two-timescale algorithm, has an $O(\sqrt{\kappa}\ln(1/\epsilon))$ iteration complexity as compared to the usual stochastic gradient descent which has $O(\kappa\ln(1/\epsilon))$ \cite{loizou2017momentum}; here, $\kappa$ is the condition number. Thus, we believe that finite-time analyses of two-timescale methods are crucial for understanding their potential merits over one-single variants.
		
		\section{Main Result}
		\label{sec: Main Result}
		We state our main convergence rate result here. It applies to the iterates $\theta'_n$  and $w'_n$ which are obtained by sparsely-projecting $\theta_n$ and $w_n$ from \eqref{eqn:tIter} and \eqref{eqn:wIter}. We begin by stating our assumptions and defining the projection operator.
		
		\noindent \assumMatrices \; (Matrix Assumptions). $W_2$ and $X_1 = \Gamma_1 - W_1 W_2^{-1}\Gamma_2$ are positive definite (not necessarily symmetric).
		
		\noindent \assumStepsize \; (Stepsize Assumption).
		$\alpha_n = (n+1)^{-\alpha}$ and $\beta_n = (n+1)^{-\beta},$ where $1 > \alpha > \beta >0.$
		
		\begin{definition}[Noise Condition] \label{assum:Noise} $\{\Mt_n\}$ and $\{\Mw_n\}$  are said to be \emph{$(\theta_n,w_n)$-dominated martingale differences with parameters $\mt$ and $\mw$}, if they are martingale difference sequences w.r.t. the family of $\sigma-$fields $\{\cF_n\},$ where
		$
		\cF_n = \sigma(\theta_0, w_0, \Mt_1, \Mw_1, \ldots, \Mt_n, \Mw_n),
		$
		and
		\begin{eqnarray*}
		\norm{\Mt_{n + 1}} &\leq& \mt(1 + \norm{\theta_n} + \norm{w_n}) 
		\\
		\norm{\Mw_{n + 1}} &\leq& \mw(1 + \norm{\theta_n} + \norm{w_n})
		\end{eqnarray*}
		for all $n \geq 0.$
		\end{definition}
		
		\begin{definition}[Sparse Projection]
		For $R > 0$, let $\Pi_R(x) = \min\{1,R/\|x\|\} \cdot x$ be the projection into the ball with radius $R$ around the origin. The \emph{sparse projection} operator 
		%
		\begin{equation} \label{def: projection}
		\Pi_{n, R} =
		    \begin{cases}
		    \Pi_R, & \text{ if $n = k^k - 1$ for some $k \in \Int_{>0}$}, \\
		    \Id, & \text{otherwise}.
		    \end{cases}
		\end{equation}
		We call it sparse as it projects only on specific indices that are exponentially far apart.
		\end{definition}
		
		Pick an arbitrary $p > 1.$ Fix some constants $\Rprojth>0$ and $\Rprojw >0$ for the radius of the projection balls. 
		Further, let 
		\begin{equation*}
		    \thS = X_1^{-1}b_1, \quad \wS = W_2^{-1}(v_2 - \Gamma_2 \thS)
		\end{equation*}
		with $b_1 = v_1 - W_1 W_2^{-1}v_2$.
		Using \cite{borkar2009stochastic} and \cite{lakshminarayanan2017stability}, it can be shown that $(\theta_n, w_n) \to (\thS, \wS)$ a.s.

		\begin{theorem}[Main Result]
		\label{thm:Rates Proj Iterates}
		Assume \assumMatrices and \assumStepsize.
		Let $\theta'_0, w'_0 \in \Real^d$ be arbitrary. 
		Consider the update rules
		\begin{align}
		\theta_{n + 1}' & = \Pi_{n + 1,\Rprojth}\Big(\theta_n' + \alpha_n [h_1(\theta_n', w_n') + M_{n + 1}^{(1')}]\Big), \label{eq:thetap_iter} \\
		w_{n + 1}' & = \Pi_{n + 1, \Rprojw}\Big(w_n' + \beta_n [h_2(\theta_n', w'_n) + M_{n + 1}^{(2')}]\Big), \label{eq:wp_iter}
		\end{align}
		where $\{M_{n}^{(1')}\}$ and $\{M_{n}^{(2')}\}$ are 
		$(\theta_n',w_n')$-dominated martingale differences with parameters $\mt$ and $\mw$ (see Def.~\ref{assum:Noise}). Then, with probability larger than $1 - \delta,$ \hspace{0.05em} for all $n \geq N_{\ref{thm:Rates Proj Iterates}}$
		%
		%
		\begin{align} 
		\|\theta_n' - \thS\| \leq {} &   
		C_{\ref{thm:Rates Proj Iterates},\theta} \frac{\sqrt{\ln{(4d^2(n+1)^p/\delta)}}
			}{(n+1)^{\alpha/2}} \label{eqn:proj theta rate} \\
		\|w'_n - \wS\| \leq {} & C_{\ref{thm:Rates Proj Iterates},w} \frac{\sqrt{\ln{(4d^2(n+1)^p/\delta)}}
			}{(n+1)^{\beta/2}} . \label{eqn:proj w rate}
		\end{align}
		%
		%
		Refer to Tables~\ref{tab: Nconstants} and \ref{tab: constants contd} for the constants.
		\end{theorem}
		
		\subsubsection{Comments on Main Result}
		\begin{enumerate}
		    \item Our analysis goes through even if  $\theta_n \in \mathbb{R}^{d_1}, w_n \in \mathbb{R}^{d_2}$ with $d_1 \neq d_2.$ For brevity, we work with $d_1=d_2=d.$
		
		    \item The constants in the above result equal infinity when $\alpha = \beta.$ This is because the algorithm then ceases to be two-timescale, thereby making our analysis invalid. 
		\end{enumerate}
		
		\begin{table}
		\begin{center}{
		\small
		 \begin{tabular}{|c| l |} 
		 \hline
		 & \\[-2ex]
		 Constant & Definition \\ [0.5ex]
		 \hline
		    $N_{\ref{thm:Rates Proj Iterates}}$
		     & 
		     $\min\left\{ n \geq N'_{\ref{thm:Rates Proj Iterates}}: n=k^k-1 \mbox{ for some integer } k \right\}$
		    \\ [1.5ex]
		    $N'_{\ref{thm:Rates Proj Iterates}}$  
		    & 
		    $\max\big\{N_{\ref{thm:Main Res wo Proj}}, K_{\ref{lemma: A4' A5'},a}, K_{\ref{lemma: A4' A5'},b}, K_{\ref{thm:Rates Proj Iterates}, w}, K_{\ref{thm:Rates Proj Iterates}, \theta}, 
		    e^{1/\beta}, (2/\beta)^{2/\beta} \big\}$ 
		    \\ [1.5ex]
		    $N_{\ref{thm:Main Res wo Proj}}$ & $\max\{N_{\ref{thm: En0 prob bound}}, N_{\ref{thm: wn rate}}\}$ \\ [1.5ex]
		    $N_{\ref{thm: En0 prob bound}}$ & $\max\{K_{\ref{lem: small eigenvalues},\alpha}, K_{\ref{lem: small eigenvalues},\beta}, K_{\ref{eq: conditions for cond 2},\alpha}(0), K_{\ref{lem: const bound on eps},\beta},K_{\ref{lem:Large Theta n}}, (p-1)^{-1/(p-1)}\}$\\ [1.5ex]
		    $N_{\ref{thm: wn rate}}$ & $\max\big\{K_{\ref{eq: conditions for cond 2},\alpha}(\beta/2), k_{\beta}(\beta/2),e^{1/\beta}\delta^{1/p}/(4d^2)^{1/p},$ \\ [1ex]
		    & \hspace{6em} $K_{\ref{lem:IntComputation},a}, K_{\ref{lem:IntComputation},b}, K_{\ref{lem: epsilon n domination}, a}, K_{\ref{lem: epsilon n domination},b}\big\} + 1$ \\[1ex]
		 \hline
		\end{tabular} }
		\end{center}
		\caption{\label{tab: Nconstants} A summary of all $n_0$ lower bounds}
		\end{table}
		
		
		
		
		\subsection{Tightness}
		\label{sec:Tightness_Demonstration}
		%
		Here, we accompany our 
		upper bound 
		by a lower bound. 
		This bound is asymptotic and holds for unprojected algorithms. Nonetheless, a coupling argument as in the proof of  Theorem~\ref{thm:Rates Proj Iterates} can be used to obtain a similar bound for projected ones. 
		We thus establish
		the tightness (up to logarithmic terms) of the result in Theorem~\ref{thm:Rates Proj Iterates}.
		
		\begin{prop}[Lower Bound]\label{prop:lower bound}
		Assume \assumMatrices \;and \assumStepsize. Consider \eqref{eqn:tIter} and \eqref{eqn:wIter} with $\{\Mt_{n}\}$ and $\{\Mw_{n}\}$ being 
		$(\theta_n,w_n)$-dominated martingale differences 
		(see Def.~\ref{assum:Noise}). Then, there exists an algorithm for which 
		\[
		    \|\theta_n - \thS\| = \Omega_p(n^{-\alpha/2}) \quad \text{ and } \quad \|w_n - \wS\| = \Omega_p(n^{-\beta/2}),
		\]
		    where 
		$X_n = \Omega_p(\gamma_n)$ means that for any $\epsilon > 0,$ there are constants $c$ and $K$ so that $    \Pr\{|X_n|/\gamma_n < c\}\leq \epsilon, \, \forall n \geq K.$
		\end{prop}
		\begin{proof}
		See Appendix~\ref{sec:lowerBd}.
		\end{proof}

		\section{Applications to Reinforcement Learning}
		\label{sec:RL}
		Here, we apply our results on the general linear two-timescale setup to the specific RL use case. Namely, we apply Theorem ~\ref{thm:Rates Proj Iterates} to derive the tightest existing finite sample bound for the
		GTD family.
		This section relies on a similar procedure as in Section 5, \cite{2TSdalal2018}. Nonetheless, we reiterate it here for completeness.
		
		\subsection{Background}
		A Markov Decision Processes (MDP) is 
		a
		tuple $(\mathcal{S}, \mathcal{A},P,R,\gamma)$ \cite{sutton1988learning}, where $\mathcal{S}$ is the state space, $\mathcal{A}$ is the action space, $P$ is the transition kernel, $R$ is the reward function, and $\gamma$ the discount factor.
		A policy $\pi:{\mathcal S} \rightarrow {\mathcal A}$ is a stationary mapping from states to actions and $V^\pi(s) = \mathbb{E}^\pi[\sum_{n=0}^\infty\gamma^nr_n|s_0 = s]$ is the value function at state $s$ w.r.t $\pi$.
		
		As mentioned above, our results apply to GTD, which is a suite of policy evaluation algorithms. These algorithms are used to estimate the value function $V^\pi(s)$ with respect to a given $\pi$ using linear regression, i.e., $V^\pi(s)\approx\theta^\top\phi(s)$, where $\phi(s) \in \mathbb{R}^d$ is a feature vector at state $s$, and $\theta \in \mathbb{R}^d$ is a parameter vector. For brevity, we omit the notation $\pi$ and denote $\phi(s_n),~\phi(s_n')$  by $\phi_n,~\phi_n'$. Finally, let $\delta_n = r_n +\gamma\theta_n^\top\phi_n' - \theta_n^\top\phi_n~,A = \bE[\phi(\phi-\gamma\phi')^\top]~,C=\bE[\phi\phi^\top]$, and $b=\bE[r\phi]$, where the expectations are w.r.t. the stationary distribution of the induced chain \footnote{Here, the samples $\{(\phi_n,\phi'_n)\}$ are drawn iid. This assumption is standard when dealing with convergence bounds in RL \citep{liu2015finite,sutton2009convergent,sutton2009fast}.}  
		.

		We assume all rewards $r(s)$ and feature vectors $\phi(s)$ are bounded:
		\(
		|r(s)| \leq 1 ,	\|\phi(s)\| \leq 1 ~ \forall s \in S.
		\)
		Also, it is assumed that the feature matrix $\Phi$ is full rank, so $A$ and $C$ are full rank. This assumption is standard \citep{maei2010toward,sutton2009convergent}. Therefore, due to its structure, $A$ is also positive definite \citep{bertsekas2012dynamic}. Moreover, by construction, $C$ is positive semi-definite; thus, by the full-rank assumption, it is actually positive definite.
		
		
		\subsection{The GTD(0) Algorithm}
		\label{sec: GTD0}
		First introduced in \cite{sutton2009convergent}, GTD(0) is designed to minimize the objective function 
		$
		J^{\rm NEU}(\theta)
		=
		\tfrac{1}{2}(b-A\theta)^\top (b-A\theta).
		$
		Its update rule is
		\begin{align*}
		\theta_{n + 1} = & \theta_n + \st_n \left(\phi_n - \gamma \phi_n'\right)\phi_n^\top w_n,\\
		w_{n + 1} = & w_n + \sw_n r_n\phi_n + \phi_n[\gamma\phi_n'-\phi_n]^\top\theta_n .
		\end{align*}
		It thus takes the form of \eqref{eqn:tIter} and \eqref{eqn:wIter} with
		$h_1(\theta,w) = 
		A^\top w \enspace,
		h_2(\theta,w) = 
		b-A\theta - w ~,
		\Mt_{n+1} = \left(\phi_n - \gamma \phi_n'\right)\phi_n^\top w_n - A^\top w_n \enspace, \Mw_{n+1} 
		= r_n\phi_n + \phi_n[\gamma\phi_n'-\phi_n]^\top\theta_n  - \left(b-A\theta_n \right) \enspace.
		$
		That is, in case of GTD(0), the relevant matrices in the update rules are
		$\Tt = 0$, $\Wt = -A^\top$, $v_1 = 0$, and $\Tw=A$, $\Ww=\Id$, $v_2=b$.
		Additionally, $\Xt = \Tt - \Wt\Ww^{-1}\Tw = A^\top A$.
		By our assumption above, both $\Ww$ and $\Xt$ are symmetric positive definite matrices, and thus the real parts of their eigenvalues are also positive.
		Also,
		$
		\|\Mt_{n+1}\| \leq (1+\gamma+\|A\|) \|w_n\| ,
		$
		$
		\|\Mw_{n+1}\|
		\leq  1+\|b\| + (1+\gamma+\|A\|)\|\theta_n\|.
		$
		Hence, the noise condition in Defn.~\ref{assum:Noise} is satisfied with constants $\mt = (1+\gamma+\|A\|)$ and $\mw = 1 + \max(\|b\|,\gamma+\|A\|)$.
		
		We can now apply Theorem~\ref{thm:Rates Proj Iterates}  to get the following result.
		
		{
		\begin{corollary} \label{cor: RL}
			Consider the Sparsely Projected variant of GTD(0) as in \eqref{eq:thetap_iter} and \eqref{eq:wp_iter}. Then, for  $\st_n = 1/(n+1)^{\alpha}$, $\sw_n = 1/(n+1)^{\beta},$ with probability larger than $1 - \delta,$ \hspace{0.05em} for all $n \geq N_{\ref{thm:Rates Proj Iterates}},$ we have
		    \begin{align} 
		    \|\theta_n' - \thS\| \leq {} &   
		    C_{\ref{thm:Rates Proj Iterates},\theta} \frac{\sqrt{\ln{(4d^2(n+1)^p/\delta)}}
			}{(n+1)^{\alpha/2}}  \label{eqn:proj theta rate gtd} \\
		\|w'_n - \wS\| \leq {} & C_{\ref{thm:Rates Proj Iterates},w} \frac{\sqrt{\ln{(4d^2(n+1)^p/\delta)}}
			}{(n+1)^{\beta/2}}. \label{eqn:proj w rate gtd}
		\end{align}
		\end{corollary}
		}
		For GTD2 and TDC \citep{sutton2009fast}, the above result can be similarly reproduced. The detailed derivation and relevant constants are provided in Appendix~\ref{sec: gtd2 and tdc}.
		
		\section{Outline of Proof of the Main Result}
		\label{sec: proof outline}
		    
		    
		    
		
		Here, we first state an intermediary result in Thereom~\ref{thm:Main Res wo Proj} and 
		using that we sketch a proof of Theorem~\ref{thm:Rates Proj Iterates}. The full proof is in Appendix~\ref{sec:Proof Main Result}. 
		
		Assume \assumMatrices and \assumStepsize. Consider \eqref{eqn:tIter} and \eqref{eqn:wIter} with $\{\Mt_{n}\}$ and $\{\Mw_{n}\}$ being   $(\theta_n,w_n)$-dominated martingale differences with parameters $\mt$ and $\mw$ (see Def.~\ref{assum:Noise}). Let $\Gp_{n_0}$ be the event given by 
		\begin{equation*} \label{def: G_n_0'}
		\cG^\prime_{n_0} = \{\|\theta_{n_0} - \thS\| \leq \Rprojth, \|w_{n_0} - \wS\| \leq \Rprojw\}
		\end{equation*}
		and let $
		\nu(n; \gamma) = (n+1)^{-\gamma/2} \sqrt{\ln{(4d^2(n+1)^p/\delta)}}.$

		\begin{theorem}
		\label{thm:Main Res wo Proj}
		Let $\delta \in (0, 1).$ Suppose that $n_0 \geq N_{\ref{thm:Main Res wo Proj}}$ and that the event $\Gp_{n_0}$  holds. Then, with probability larger than $1- \delta,$
		%
		\begin{align}
		\|\theta_{n} - \thS\| & \leq A_{5,n_0} \, \nu(n, \alpha) 
		\label{eq: theta rate}\\
		\|w_{n} - w^*\| & \leq  A_{4,n_0} \, \nu(n, \beta) 
		\label{eq: w rate}
		\end{align}
		for all $n \geq n_0.$
		\end{theorem}
		\begin{proof}[Sketch of Proof for Theorem~\ref{thm:Rates Proj Iterates}]
		Our idea is to use a coupling argument to show that the projected iterates, given in \eqref{eq:thetap_iter} and \eqref{eq:wp_iter}, and the unprojected iterates, given in \eqref{eqn:tIter} and \eqref{eqn:wIter}, are identically distributed from some time on. This then allows us to use Theorem~\ref{thm:Main Res wo Proj} to conclude Theorem~\ref{thm:Rates Proj Iterates}. 
		
		The key steps in our argument are as follows.
		\begin{enumerate}
		    \item First we note that, for the projected algorithm, the event $\Gp_{n_0}$ holds whenever $n_0$ is of the form $k^k - 1.$ 
		    
		    \item Further, recalling \eqref{def: projection}, we observe that, for any $k \geq 0$, between 
		    projection steps $k^k - 1$ and $(k+1)^{k+1} -1$, the projected iterates $\{\theta_n', w'_n\}$ behave exactly as the unprojected iterates $\{\theta_n, w_n\}$ that are initiated at $(\theta'_{k^k - 1}, w'_{k^k - 1}).$
		    
		    \item It then follows from Theorem~\ref{thm:Main Res wo Proj} that if $k$ is large enough so that $n_0 = k^k - 1 \geq N_{\ref{thm:Main Res wo Proj}},$ then 
		    \eqref{eq: theta rate} and \eqref{eq: w rate} apply to $\{\theta'_n, w'_n\}$  for $k^k - 1 \leq n < (k + 1)^{k + 1} - 1.$
		    
		    \item In fact, if $k$ is enlarged a bit more so that $n_0 = k^k - 1 \geq N_{\ref{thm:Rates Proj Iterates}} \geq N_{\ref{thm:Main Res wo Proj}},$ then not only does the above claim hold, it is also true that the RHSs in \eqref{eq: theta rate} and \eqref{eq: w rate} are less than $\Rprojth$ and $\Rprojw,$ respectively, for $n \geq (k + 1)^{k + 1} - 1.$ 
		    \item In turn, the latter implies that the projected iterates and unprojected iterates, starting from $(\theta'_n, w'_n),$ behave exactly the same $\forall n \geq k^k - 1.$ Consequently, \eqref{eq: theta rate} and \eqref{eq: w rate} hold for the projected iterates $\forall n \geq N_{\ref{thm:Rates Proj Iterates}}.$ Substituting $n_0 = N_{\ref{thm:Rates Proj Iterates}}$ then establishes Theorem~\ref{thm:Rates Proj Iterates}. 
		\end{enumerate}
		
		
		    
		    

		
		%
		\noindent See Appendix~\ref{sec:Proof Main Result} for the actual proof. 
		%
		\end{proof}
		
		Next, we discuss the proof of Theorem~\ref{thm:Main Res wo Proj}; note that this result only concerns the unprojected iterates.
		First, we introduce some further notations.
		
		Fix any $p > 1$ and let $\cU(n_0)$ be the event given by
		\begin{multline} \label{def: E_n0}
		\cU(n_0) : = \bigcap_{n \geq n_0}\big\{\|\theta_n - \thS\| \leq \Crth \Rprojth, \|L_{n + 1}^{(\theta)}\| \leq \epsth_n,\\  \|w_n - \wS\| \leq \Crw \Rprojw, \|L_{n + 1}^{(w)}\| \leq \epsw_n\big\},
		\end{multline}
		where 
		\begin{align}
		\label{eq: epsilon n new def}
		\epsth_n = {} & \sqrt{d^3 L_\theta C_{\ref{lem: an bn upper bounds},\theta}} \, \nu(n, \alpha), \\
		\epsw_n = {} & \sqrt{d^3 L_w C_{\ref{lem: an bn upper bounds},w}} \, \nu(n, \beta).
		\end{align}
		Further, let $\Lnt$ and $\Lnw$ be appropriate aggregates of the martingale noise terms 
		given by
		\begin{align}
		    \Lnw &= \sum_{k=n_0}^n  \left[\prod_{j=k+1}^n[I-\beta_j W_2]\right] \beta_k  M_{k + 1}^{(2)}, \label{def: Lnw Main}
		\end{align}
		\begin{multline}
		\Lnt = \sum_{k=n_0}^n  \left[\prod_{j=k+1}^n[I-\alpha_j X_1]\right] \alpha_k \\
		\times\left[- W_1 W_2^{-1} M_{k + 1}^{(2)} + M_{k + 1}^{(1)}\right]. \label{def: Lnt Main}
		\end{multline}
		For the definition of the constants above, see Table~\ref{tab: constants contd}.
		

		As a first step in proving Theorem~\ref{thm:Main Res wo Proj}, we show that the co-occurrence of the events $\Gp_{n_0}$ and $\cU(n_0)$ has small probability if $n_0$ is large enough. The proof, inspired from \cite{thoppe2019concentration}, uses induction on the iteration index $n.$ Specifically, we show that if, at time $n,$ the iterates are bounded and the aggregate noise is well-behaved (respectively bounded by $\epsth_n$ and $\epsw_n$), then the iterates continue to remain bounded at time $n + 1$ as well w.h.p.
		%
		\begin{theorem}
		\label{thm: En0 prob bound}
		Let $\delta \in (0,1)$ and $n_0 \geq N_{\ref{thm: En0 prob bound}}.$ 
		Then, 
		$$\Pr\{\cU^c(n_0) | \Gp_{n_0}\} \leq \delta.$$
		\end{theorem}
		
		Next, we  show that, on the event $\cU(n_0),$ the convergence rates of $\{\theta_n\}$ and $\{w_n\}$ are  $\tilde{O}(n^{-\alpha/2})$ and $\tilde{O}(n^{-\beta/2}),$ respectively. The proof proceeds as follows. By refining  an induction trick from \cite{mokkadem2006convergence}, we first show that the convergence rate estimate for the $\{w_n\}$ iterates can be improved from $O(1)$ to $\tilde{O}(n^{-\beta/2}).$ Using this, we then show that $\|\theta_n - \thS\| = \tilde{O}(n^{-\alpha/2})$.
		We emphasize that these results are
		deterministic.

		\begin{theorem}
		\label{thm: wn rate}
		Let $n_0 \geq N_{\ref{thm: wn rate}}.$
		%
		Then,
		\begin{equation}\label{eq: w subset}
		\cU(n_0)  \subseteq \{\|w_n - \wS\| \leq A_{4,n_0} \, \nu(n; \beta), \forall n \geq n_0\}
		\end{equation}
		and 
		\begin{equation}
		\label{eq: theta subset}
		 \cU(n_0) \subseteq \{ \|\theta_{n} - \thS\| \leq A_{5,n_0} \,
		\nu(n; \alpha), \forall n \geq n_0\}.
		\end{equation}
		\end{theorem}

		\begin{proof}[Proof of Theorem~\ref{thm:Main Res wo Proj}]
		Theorems~\ref{thm: En0 prob bound} and \ref{thm: wn rate} together establish Theorem~\ref{thm:Main Res wo Proj}.
		\end{proof}
		
		The next two subsections highlight the key steps in the proofs of these last two results. 
		
		\subsection{Proof of Theorem~\ref{thm: En0 prob bound}}
		%
		    

		    
		
		Let
		$    \Crth  = 3$ and 
		$    \Crw  = 3/2 + (e^{q_2} /q_2 \|\Gamma_2\| C_{\ref{lem: Dn bounds},w}) \Crth \frac{\Rprojth}{\Rprojw}.$ 
		%
		Further, let $\cG_n, \cL_n,$ and $\cA_n$ be the events given by
		\begin{multline} \label{def: Gn}
		\cG_n = \bigcap_{k = n_0}^n\{\|\theta_k - \thS\| \leq \Crth\Rprojth, \\ \|w_k - \wS\| \leq \Crw\Rprojw\},
		\end{multline}
		\begin{equation} \label{def: Ln}
		\cL_n = \bigcap_{k = n_0}^{n} \left\{\|L_{k + 1}^{(\theta)}\| \leq \epsth_k,  \|L_{k + 1}^{(w)}\| \leq \epsw_k\right\},
		\end{equation}
		and $\cA_n = \cG_n \cap \cL_n.$ %
		Using \eqref{def: E_n0}, note that $\cU(n_0) = \lim_{n \to \infty}\cA_n = \bigcap_{n \geq n_0}\cA_n.$
		Lastly, define
		\begin{multline}
		\label{eq: Zn defn}
		\cZ_n = \{\|\theta_n - \thS\| \leq  \Crth \Rprojth, \|w_n - \wS\| \leq \Crw \Rprojw, \\
		\|L_{n + 1}^{(\theta)}\| \leq \epsth_n, \|L_{n + 1}^{(w)}\| \leq \epsw_n\}.
		\end{multline}
		
		\begin{proof}[Proof of Theorem~\ref{thm: En0 prob bound}]
		By adopting ideas from  \cite{thoppe2019concentration}, we first 
		decompose the event $\Gp_{n_0} \cap \cU^c_{n_0}.$ From \eqref{eq: U complement reformulated} - \eqref{eqn:Gpn0 intersect An intersect Zn1 compl app}
		 in the appendix, we have
		\begin{multline}
		\label{eqn:Gpnp interesect Un0 compl}
		    \Gp_{n_0} \cap \cU^c(n_0) = (\Gp_{n_0} \cap \cZ_{n_0}^c) \cup (\Gp_{n_0} \cap \cA_{n_0} \cap \cZ_{n_0 + 1}^c) \\\cup (\Gp_{n_0} \cap \cA_{n_0 + 1} \cap \cZ_{n_0 + 2}^c) \cup \ldots,
		\end{multline}
		\begin{multline}
		\label{eqn:Gpn0 intersect Zn0 complement}
		    \Gp_{n_0} \cap \cZ_{n_0}^c \subseteq \cG_{n_0} \cap \big(\{\|L_{n_0 + 1}^{(\theta)}\| > \epsth_{n_0}\} \\ \cup \{ \|L_{n_0 + 1}^{(w)}\| > \epsw_{n_0}\}\big),
		\end{multline}
		and
		\begin{align}
		\Gp_{n_0}& \cap \cA_n \cap  \cZ^c_{n + 1} \\
		\subseteq {} & \Gp_{n_0} \cap \cA_{n} \cap \Big[\{\|\theta_{n + 1} - \thS\| > \Crth \Rprojth\} \notag\\
		&\hspace{2.2cm}\cup \{\|w_{n + 1} - \wS\| > \Crw \Rprojw\} \Big] \\ 
		& \cup \Big(\cG_{n + 1} \cap \big[\{\|L_{n + 2}^{(\theta)}\| > \epsth_{n + 1}\} \notag \\
		&\hspace{3cm}\cup \{\|L_{n + 2}^{(w)}\| > \epsw_{n + 1}\}\big]\Big).
		\label{eq: before removing the empty event}
		\end{align}
		
		With regards to \eqref{eq: before removing the empty event}, we also have the following fact. 
		\begin{lemma}
		\label{lem:Large Theta n}
		Let 
		$n \geq n_0 \geq \max\{K_{\ref{lem: small eigenvalues},\alpha}, K_{\ref{lem: small eigenvalues},\beta}, K_{\ref{eq: conditions for cond 2},\alpha}(0), K_{\ref{lem: const bound on eps}, \beta}, K_{\ref{lem:Large Theta n}} \}.
		$
		Then,
		\begin{multline*}
		\Gp_{n_0} \cap \cA_{n} \cap \Big[\{\|\theta_{n + 1} - \thS\| > \Crth \Rprojth\} \\
		\hspace{2.2cm}\cup \{\|w_{n + 1} - \wS\| > \Crw \Rprojw\} \Big] = \emptyset.
		\end{multline*}
		\end{lemma}
		\begin{proof} See Appendix~\ref{sec:Proof Lemma Large Theta n}.
		\end{proof}
		
		Therefore, it follows that for $n \geq n_0$ 
		\begin{multline}
		    \Gp_{n_0}  \cap \cA_n \cap  \cZ^c_{n + 1} 
		    \subseteq 
		    \cG_{n + 1} \\ \cap \big[\{\|L_{n + 2}^{(\theta)}\| > \epsth_{n + 1}\}  
		 \cup \{\|L_{n + 2}^{(w)}\| > \epsw_{n + 1}\}\big]
		 \label{eqn:Gpn0 intersect An intersect Zn1 compl}.
		\end{multline}
		
		Equations \eqref{eqn:Gpnp interesect Un0 compl}, \eqref{eqn:Gpn0 intersect Zn0 complement} and \eqref{eqn:Gpn0 intersect An intersect Zn1 compl}  together imply  
		\begin{multline}
		\label{eq: superset of Gn0 cap Uc}
		\Gp_{n_0} \cap \cU^c(n_0) \subseteq \bigcup_{n \geq n_0} \Big( \cG_n \cap \big[\{|L_{n + 1}^{(\theta)}\| > \epsth_n\} \\ \cup \{\|L_{n + 1}^{(w)}\| > \epsw_n\}\big]\Big).
		\end{multline}

		The usefulness of this decomposition lies in the fact that each 
		term in the union
		contains the event $\cG_n$ which ensures that the iterates are bounded. This, along with our noise assumption in Definition~\ref{assum:Noise}, implies that the Martingale differences are in turn bounded and the Azuma-Hoeffding inequality can now be invoked (see Lemma~\ref{eq: azuma-hoeffding}). 
		%
		%
		Applying this on \eqref{eq: superset of Gn0 cap Uc} after using the union bound gives
		\begin{align}
		\Pr&\{ \cU^c(n_0) \,| \, \Gp_{n_0} \} 
		=
		\Pr\{ \cU^c(n_0) \cap \Gp_{n_0} \,|\, \Gp_{n_0} \}
		\\
		&\leq 
		\sum_{n \geq n_0}
		\Pr\Big( \cG_n \cap    \{\|L_{n + 1}^{(w)}\| > \epsw_n\} \, |\, \Gp_{n_0}\Big)
		\\&\hspace{0.5cm}+ 
		\sum_{n \geq n_0}\Pr\Big( \cG_n \cap  \{\|L_{n + 1}^{(w)}\| > \epsw_n\}\, |\, \Gp_{n_0}\Big)
		\\
		& \leq  \sum_{n \geq n_0} 2d^2 \exp\left(-\tfrac{{(\epsth_n)}^2}{d^3 L_\theta a_{n + 1}}\right) 
		\notag\\
		&\hspace{1.5cm} + \sum_{n \geq n_0} 2d^2 \exp\left(-\tfrac{{(\epsw_n)}^2}{d^3 L_w b_{n + 1}}\right). \label{eq: prob e comp bound}
		\end{align}
		Additionally, due to Lemma~\ref{lem: an bn upper bounds} in the Appendix,
		\begin{equation}
		 \label{eq: an bn bounds}
		\begin{aligned}
		 a_{n + 1} \leq 
		 C_{\ref{lem: an bn upper bounds},\theta} (n+1)^{-\alpha},
		 \\
		 b_{n + 1} \leq 
		C_{\ref{lem: an bn upper bounds},w} (n+1)^{-\beta}.
		\end{aligned}
		\end{equation}
		Substituting \eqref{eq: an bn bounds} and \eqref{eq: epsilon n new def} in \eqref{eq: prob e comp bound} gives
		\begin{align}
		\Pr\left\{\Gp_{n_0} \cap \cU^c(n_0) \, \middle|\, \Gp_{n_0}\right\} \leq {}   \sum_{n \geq n_0} \frac{\delta}{(n+1)^p} 
		\leq {} & 
		%
		\delta \frac{n_0^{-(p-1)}}{p-1} \label{eq: delta bound for U}.
		\end{align}
		%
		%
		Now, since 
		\begin{equation}
		\label{eq: n_0 bound}
		n_0 \geq (p-1)^{-1/(p-1)},
		\end{equation}
		it eventually follows that  \eqref{eq: delta bound for U} $\leq \delta,$ as desired.
		\end{proof}

		\subsection{Proof of Theorem~\ref{thm: wn rate}}
		
		For a sequence $u \in \RInf,$ let
		\begin{equation} \label{def: wn}
		\cW_n(u) := \{\|w_k - \wS\| \leq u_k \; \forall n_0 \leq k \leq n\}.
		\end{equation}
		
		\begin{definition}\label{def: wn behavior 2}
		We say that $u \in \RInf$ is \emph{$\alpha$-moderate from $k_0$ onwards} if
		%
		\begin{equation*}
		    \frac{\uw_k}{\uw_{k + 1}} \leq \frac{\alpha_{k + 1}}{\alpha_k} \frac{\beta_k}{\beta_{k + 1}} e^{q_1/2 \; \alpha_{k + 1}}, \; \quad \forall k \geq k_0.
		\end{equation*}
		\end{definition}
		
		\begin{definition}
		\label{def: wn behavior 3}
		We say that $u \in \RInf$ is \emph{$\beta$-moderate from $k_0$ onwards} if
		\begin{equation*}
		\frac{\uw_k}{\uw_{k + 1}} \leq \frac{\alpha_{k + 1}}{\alpha_k} \frac{\beta_k}{\beta_{k + 1}} e^{q_2/2 \; \beta_{k + 2}},\; \quad \forall k \geq k_0.
		\end{equation*}
		\end{definition}
		
		We consider these definitions to be part of the novelty of this work.
		They characterize a sequence
		via 
		the ratio of its consecutive terms.
		Ratios in a decaying sequence (such as the ones used in this paper) satisfying Defs. \ref{def: wn behavior 2} or \ref{def: wn behavior 3} will converge to $1$.
		Examples of sequences satisfying these definitions are
		constant sequences and those that decay at an inverse polynomial rate. On the other hand, sequences that decay exponentially fast do not satisfy these conditions. 
		These definitions play a crucial role in enabling our induction; i.e., 
		they help us show that the estimates on the rate of convergence of $\|w_n - \wS\|$ 
		can be incrementally improved. One quick way to see this is via \eqref{eqn: wn Bound eps and u} given later; it shows that if the bound on $\|w_n - \wS\|$ was $u_n,$ then it can be improved via induction to $O(\epsilon_{n}) + O\left(\frac{\alpha_n}{\beta_n}u_n\right).$ These definitions are motivated by Definitions~1 and 2 in \cite{mokkadem2006convergence}. However, there they are expressed as a certain asymptotic behavior, while ours provide the exact sequence, including constants, and thereby enable finite time analysis.


		%
		For $\ell \geq 0,$ let $\cE(n_0; \ell) := \bigcap_{n \geq n_0}\{\|w_n - \wS\| \leq  \uw_n(\ell)\},$
		where 
		\begin{equation}
		\label{eq: hypothesis}
		u_n(\ell) :=  \left[A_{1,n_0}  \sum_{i=0}^{\ell-1} A_2^i\right] \epsw_n  + \left[A_3 A_2^\ell\right] {\left[\frac{\alpha_n}{\beta_n}\right]}^\ell;
		\end{equation}
		all the constants are given in Table~\ref{tab: constants contd}.
		
		\begin{proof}[Proof of Theorem~\ref{thm: wn rate}]

		Our proof idea inspired by \cite{mokkadem2006convergence} is as follows. We use induction to show that whenever  $\cU(n_0)$ holds, the rate of convergence of $w_n$ is bounded by \eqref{eq: hypothesis} for all $\ell \leq \ell^*,$ where the latter is as in \eqref{def: lS}. Notice that there are two terms in \eqref{eq: hypothesis} that depend on $n,$ one is $\epsilon_n$ and the other is $\alpha_n/\beta_n.$ As $\ell$ increases, $(\alpha_n/\beta_n)^\ell$ decays faster. Thus, eventually, for $\ell = \ell^*,$ the convergence rate of $w_n$ would be dictated by $\epsilon_n,$ thereby giving us our desired result.
		
		Formally, we begin with proving the following claim.  
		
		\textbf{Claim:} Let 
		\begin{equation} \label{def: lS}
		    \ell^* = \left\lceil\frac{\beta}{2(\alpha-\beta)} \right\rceil;
		\end{equation}
		i.e., let $\ell^*$ be the smallest integer $\ell$ such that $(\alpha - \beta) \ell \geq \beta/2$. Then, for $0 \leq \ell \leq \ell^*,$ 
		\begin{equation}
		    \label{eqn:Ind_Claim}
		    \cU(n_0) \subseteq \cE(n_0; \ell).
		\end{equation}

		\textbf{Induction Base:} By definition, $\cU(n_0) \subseteq \cE(n_0, 0).$ 
		
		\textbf{Induction Hypothesis}: Suppose \eqref{eqn:Ind_Claim} holds for some $\ell$ such that $0 \leq \ell < \ell^*.$
		    
		\textbf{Induction Step:} For the $\ell$ defined in the hypothesis above, we have $(\alpha - \beta) \ell < \beta/2.$ Making use of this, we now show that $\cU(n_0) \subseteq \cE(n_0, \ell + 1).$ 
		
		From the induction hypothesis, on $\cU(n_0)$, for $n \geq n_0 - 1,$
		\begin{equation}
		\label{eqn:wBound}
		\|w_{n + 1} - \wS\| \leq  \uw_{n + 1}(\ell).
		\end{equation}
		A useful result for improving this bound is the following. 
		
		\begin{lemma} \label{lem: bound on wn}
		Let $n_0 \in \mathbb{N}.$ 
		Let $u \in \RInf$ be a monotonically decreasing sequence that is both $\alpha$-moderate and $\beta$-moderate from $n_0 - 1$ onwards. Let $n \geq n_0-1.$ Suppose that the event $\cW_{n}(u)$ holds, $\|L_n^{(\theta)}\| \indc[n \geq n_0 + 1] \leq \epsth_{n-1}$, and $\|L_n^{(w)}\| \indc[n \geq n_0 + 1] \leq \epsw_{n-1}$. If $n \geq n_0,$ then 
		\begin{align}
		    \label{eqn:thBd Main}
		    &\|\theta_n - \thS\| \leq  C_{\ref{lemma: R_n w bound},b} \frac{\alpha_{n - 1}}{\beta_{n - 1}}u_{n - 1} +  \epsth_{n-1} \\ &+  C_{\ref{lemma: R_n w bound},a} \left[\|\theta_{n_0} - \thS\|
		    + \frac{\alpha_{n_0}}{\beta_{n_0}}\|w_{n_0} - \wS\|\right] e^{-q_1\sum_{j = n_0 + 1}^{n - 1} \alpha_j}.\notag
		\end{align}
		%
		Additionally, if $n_0 \geq \max\{K_{\ref{lem:IntComputation},a}, K_{\ref{lem:IntComputation},b}, K_{\ref{lem: epsilon n domination},a}, K_{\ref{lem: epsilon n domination},b}, $ $ K_{\ref{eq: conditions for cond 2},\alpha}(\beta/2)\} +  1$ and $n \geq n_0 - 1,$
		then
		\begin{equation}
		    \label{eqn: wn Bound eps and u}
		    \|w_{n+1} - w^*\| \leq  A_{1,n_0} \epsw_{n + 1} +  A_2 \frac{\alpha_{n+1}}{\beta_{n+1}} \uw_{n + 1}.
		\end{equation}

		All the constants are as in Table~\ref{tab: constants contd}.
		\end{lemma}
		\begin{proof}
		See Appendix~\ref{sec: Proof of wn Bound}.
		\end{proof}
		
		We now verify the conditions necessary to apply this result. After substituting the value of $\epsw_n$ from \eqref{eq: epsilon n new def}, and those of $\alpha_n, \beta_n$ into \eqref{eq: hypothesis}, and then pulling out $p$ from \eqref{eq: epsilon n new def} to the constants, observe that $u_n(\ell)$ is of the form
		\begin{multline} \label{eq: wn leq u}
		    u_n(\ell) = B_1 (n+1)^{-\beta/2} \sqrt{\ln [B_2 (n + 1)]} \\ + B_3 (n+1)^{-(\alpha - \beta) \ell}
		\end{multline}
		for some suitable constants $B_1, B_2$ and $B_3.$ 
		Clearly, $B_1$ and $B_3$ are strictly positive, while $B_2 = (4d^2/\delta)^{1/p} \geq 1.$ 
		 Lemma~\ref{lem: un satisfies condition 2} then shows $\{u_n(\ell)\}$ is $\alpha$-moderate, $\beta$-moderate, and monotonically decreasing from $n_0-1$ onwards.
		
		Additionally, notice that due to \eqref{eqn:wBound} the event $\cW_{n}(u)$  holds for $u = \{u_n(\ell)\}$, while on $\cU(n_0)$ the events $
		    \{\|L_n^{(\theta)}\| \indc[n \geq n_0 + 1] \leq \epsth_{n-1}\}$ and $ \{\|L_n^{(w)}\| \indc[n \geq n_0 + 1] \leq \epsw_{n-1}\}$
		hold. Since $n_0 \geq N_{\ref{thm: wn rate}} \geq \max\{K_{\ref{lem:IntComputation},a}, K_{\ref{lem:IntComputation},b}, K_{\ref{lem: epsilon n domination},a}, K_{\ref{lem: epsilon n domination},b}, K_{\ref{eq: conditions for cond 2},\alpha}(\beta/2)\} +1,$ we can now employ Lemma~\ref{lem: bound on wn} with $\{u_n\} = \{u_n(\ell)\}$ and obtain that, on the event $\cU(n_0),$
		\begin{equation*}
		    \|w_{n+1} - w^*\| \leq  A_{1,n_0} \epsw_{n + 1} +  A_2 \frac{\alpha_{n+1}}{\beta_{n+1}} \uw_{n + 1}(\ell).
		\end{equation*}
		%
		By substituting the value of $\uw_{n + 1}(\ell)$ from \eqref{eq: hypothesis} and making use of the fact that $\alpha_n/\beta_n \leq 1$, we get
		%
		\begin{equation*}
		\label{eqn:IndStep}
		A_{1,n_0} \epsw_{n+1} + A_2 \frac{\alpha_{n+1}}{\beta_{n+1}} \uw_{n + 1}(\ell) \leq \uw_{n + 1}(\ell + 1).
		\end{equation*}
		%
		%
		%
		%
		This completes the proof of the induction step.
		
		When $\ell = \ell^*,$ it now follows that $\cU(n_0) \subseteq \cE(n_0; \ell^*).$ That is, when the event $\cU(n_0)$ holds, 
		\begin{equation*} \label{eq: wn bounded by u_n}
		   \|w_{n + 1} - \wS\| \leq  \uw_{n + 1 }(\ell^*), \quad \forall n \geq n_0 -1.
		\end{equation*}
		%

		%
		We now bound $u_n(\ell^*).$ Since ${\ceil{\frac{\beta}{2(\alpha - \beta})} \geq \frac{\beta}{2(\alpha - \beta)}},$ we have $(\alpha_n/\beta_n)^{\ceil{\frac{\beta}{2(\alpha - \beta})}} \leq (n  + 1)^{-\beta/2}.$ Substituting the value of $\epsw_n$ and using the above relation along with the fact that $4 \geq e$ which implies  $\sqrt{\ln{(4d^2(n+1)^p/\delta)}} \geq 1,$ we have 
		\begin{multline}
		u_n(\ell^*) \leq \bigg[A_{1,n_0}  \sum_{i=0}^{\ceil{\frac{\beta}{2(\alpha-\beta)}}-1} A_2^i \sqrt{d^3 L_w C_{\ref{lem: an bn upper bounds},w}} \\ + A_3 A_2^{\ceil{\frac{\beta}{2(\alpha-\beta)}}}\bigg]  \nu(n; \beta) 
		.
		\end{multline}
		Consequently, for $n \geq n_0 - 1,$
		\begin{equation}
		\label{eqn:Nice un lstar Bd}
		\|w_{n + 1} - \wS\| \leq  u_{n + 1}(\ell^*) \leq A_{4, n_0} \,  \nu(n + 1; \beta) 
		\end{equation}
		which establishes \eqref{eq: w subset}.
		
		We now prove \eqref{eq: theta subset}. 
		On the event $\cU(n_0)$, we can apply \eqref{eqn:thBd Main} from Lemma~\ref{lem: bound on wn} with $\{u_n\} = \{u_n(\ell^*)\}$ and use the fact that $\alpha_{n_0} /\beta_{n_0} \leq 1,$  as well as bound $\|\theta_{n_0} - \thS\|$ and $\|w_{n_0} - \wS\|$ using $\cU(n_0),$ to get
		\begin{multline*}
		\|\theta_{n} - \thS\| \leq C_{\ref{lemma: R_n w bound},b} \frac{\alpha_{n-1}}{\beta_{n-1}}u_{n-1}(\ell^*) \\
		+ C_{\ref{lemma: R_n w bound},a}  \left[\Crth\Rprojth + \Crw \Rprojw \right] e^{-q_1\sum_{j = n_0 + 1}^{n-1} \alpha_j} +  \epsth_{n-1}.
		\end{multline*}
		%
		

		Now, Lemma~\ref{lem: epsilon n domination} (see Appendix~\ref{sec: proof epsilon n domination}) and the fact that $q_1 \geq q_{\min}$ imply (in Lemma~\ref{lem: epsilon n domination} we require $n \geq n_0$ but here we use it from $n_0 - 1$, which is justified since $n_0 \geq K_{\ref{lem: epsilon n domination},b} + 1$), on $\cU(n_0)$,
		\begin{multline*}
		\|\theta_{n} - \thS\| \leq 
		C_{\ref{lemma: R_n w bound},b} \frac{\alpha_{n-1}}{\beta_{n-1}}u_{n-1}(\ell^*)
			\\+ \left[C_{\ref{lemma: R_n w bound},a}  \left[\Crth\Rprojth + \Crw \Rprojw \right] / \epsth_{n_0-1} + 1\right] \epsth_{n-1} .
		\end{multline*}
		Consequently, using \eqref{eq: epsilon n new def}, \eqref{eqn:Nice un lstar Bd} and the facts that $\alpha_{n-1}/\beta_{n-1} = n^{-(\alpha - \beta)}$ and
		$\alpha/2 = \alpha-\alpha/2 \leq \alpha-\beta/2,$
		we have that,  on $\cU(n_0),$
		\begin{multline*}
		\|\theta_{n} - \thS\| \leq  C_{\ref{lemma: R_n w bound},b}[A_{4,n_0} \nu(n - 1, \alpha)] \\ + \left[C_{\ref{lemma: R_n w bound},a}  \left[\Crth\Rprojth + \Crw \Rprojw \right] / \epsth_{n_0-1} + 1\right] \epsth_{n-1}.
		\end{multline*}
		Since $\nu(n-1,\alpha) \leq 2\nu(n,\alpha)$, the theorem follows. 
		\end{proof}

		\section{Discussion}
		Two-timescale SA lies at the foundation of RL in the shape of several popular evaluation and control methods. This work introduces the tightest finite sample analysis for the GTD algorithm suite. We provide it as a general methodology that applies to all linear two-timescale SA algorithms.
		
		Extending our methodology to the case of GTD algorithms with non-linear function-approximation, in similar fashion to \cite{bhatnagar2009convergent}, would be a natural future direction to consider.
		Such a result could be of high interest due to the attractiveness of neural networks. 
		Finite time analysis of
		non-linear SA would also be of use in better understanding actor-critic RL algorithms.
		An additional direction for future research could be finite sample analysis of distributed SA algorithms of the kind discussed in \cite{mathkar2016nonlinear}. 
		
		Lastly, 
		it would also be interesting to see 
		how adaptive stepsizes can help improve sample complexity in all the above scenarios.

		\newpage
		
		\bibliographystyle{aaai}
		\bibliography{References}

\begin{thebibliography}{}

\bibitem[\protect\citeauthoryear{Bertsekas}{2012}]{bertsekas2012dynamic}
Bertsekas, D.~P.
\newblock 2012.
\newblock {\em Dynamic Programming and Optimal Control}.
\newblock Vol II. Athena Scientific, fourth edition.

\bibitem[\protect\citeauthoryear{Bhatnagar \bgroup et al\mbox.\egroup
  }{2009}]{bhatnagar2009convergent}
Bhatnagar, S.; Precup, D.; Silver, D.; Sutton, R.~S.; Maei, H.~R.; and
  Szepesv{\'a}ri, C.
\newblock 2009.
\newblock Convergent temporal-difference learning with arbitrary smooth
  function approximation.
\newblock In {\em Advances in Neural Information Processing Systems},
  1204--1212.

\bibitem[\protect\citeauthoryear{Borkar}{1997}]{borkar1997stochastic}
Borkar, V.~S.
\newblock 1997.
\newblock Stochastic approximation with two time scales.
\newblock {\em Systems \& Control Letters} 29(5):291--294.

\bibitem[\protect\citeauthoryear{Borkar}{2009}]{borkar2009stochastic}
Borkar, V.~S.
\newblock 2009.
\newblock {\em Stochastic approximation: a dynamical systems viewpoint},
  volume~48.
\newblock Springer.

\bibitem[\protect\citeauthoryear{Dalal \bgroup et al\mbox.\egroup
  }{2018a}]{dalal2018finite}
Dalal, G.; Szorenyi, B.; Thoppe, G.; and Mannor, S.
\newblock 2018a.
\newblock Finite sample analyses for td(0) with function approximation.
\newblock In {\em AAAI}.

\bibitem[\protect\citeauthoryear{Dalal \bgroup et al\mbox.\egroup
  }{2018b}]{2TSdalal2018}
Dalal, G.; Thoppe, G.; Sz{\"o}r{\'e}nyi, B.; and Mannor, S.
\newblock 2018b.
\newblock Finite sample analysis of two-timescale stochastic approximation with
  applications to reinforcement learning.
\newblock In Bubeck, S.; Perchet, V.; and Rigollet, P., eds., {\em Proceedings
  of the 31st Conference On Learning Theory}, volume~75 of {\em Proceedings of
  Machine Learning Research},  1199--1233.
\newblock PMLR.

\bibitem[\protect\citeauthoryear{Gerencs{\'e}r}{1997}]{gerencser1997rate}
Gerencs{\'e}r, L.
\newblock 1997.
\newblock Rate of convergence of moments of spall's spsa method.
\newblock In {\em Control Conference (ECC), 1997 European},  2192--2197.
\newblock IEEE.

\bibitem[\protect\citeauthoryear{Konda and
  Tsitsiklis}{2004}]{konda2004convergence}
Konda, V.~R., and Tsitsiklis, J.~N.
\newblock 2004.
\newblock Convergence rate of linear two-time-scale stochastic approximation.
\newblock {\em The Annals of Applied Probability} 14(2):796--819.

\bibitem[\protect\citeauthoryear{Kushner and Yin}{1997}]{kushner1997stochatic}
Kushner, H.~J., and Yin, G.~G.
\newblock 1997.
\newblock {\em Stochastic Approximation Algorithms and Applications}.

\bibitem[\protect\citeauthoryear{Lakshminarayanan and
  Bhatnagar}{2017}]{lakshminarayanan2017stability}
Lakshminarayanan, C., and Bhatnagar, S.
\newblock 2017.
\newblock A stability criterion for two timescale stochastic approximation
  schemes.
\newblock {\em Automatica} 79:108--114.

\bibitem[\protect\citeauthoryear{Liu \bgroup et al\mbox.\egroup
  }{2015}]{liu2015finite}
Liu, B.; Liu, J.; Ghavamzadeh, M.; Mahadevan, S.; and Petrik, M.
\newblock 2015.
\newblock Finite-sample analysis of proximal gradient td algorithms.
\newblock In {\em UAI},  504--513.
\newblock Citeseer.

\bibitem[\protect\citeauthoryear{Loizou and
  Richt{\'a}rik}{2017}]{loizou2017momentum}
Loizou, N., and Richt{\'a}rik, P.
\newblock 2017.
\newblock Momentum and stochastic momentum for stochastic gradient, newton,
  proximal point and subspace descent methods.
\newblock {\em arXiv preprint arXiv:1712.09677}.

\bibitem[\protect\citeauthoryear{Maei \bgroup et al\mbox.\egroup
  }{2010}]{maei2010toward}
Maei, H.~R.; Szepesv{\'a}ri, C.; Bhatnagar, S.; and Sutton, R.~S.
\newblock 2010.
\newblock Toward off-policy learning control with function approximation.
\newblock In {\em ICML},  719--726.

\bibitem[\protect\citeauthoryear{Mathkar and
  Borkar}{2016}]{mathkar2016nonlinear}
Mathkar, A.~S., and Borkar, V.~S.
\newblock 2016.
\newblock Nonlinear gossip.
\newblock {\em SIAM Journal on Control and Optimization} 54(3):1535--1557.

\bibitem[\protect\citeauthoryear{Mokkadem and
  Pelletier}{2006}]{mokkadem2006convergence}
Mokkadem, A., and Pelletier, M.
\newblock 2006.
\newblock Convergence rate and averaging of nonlinear two-time-scale stochastic
  approximation algorithms.
\newblock {\em The Annals of Applied Probability} 16(3):1671--1702.

\bibitem[\protect\citeauthoryear{Polyak}{1990}]{polyak1990new}
Polyak, B.~T.
\newblock 1990.
\newblock New stochastic approximation type procedures.
\newblock {\em Automat. i Telemekh} 7(98-107):2.

\bibitem[\protect\citeauthoryear{Ruppert}{1988}]{ruppert1988efficient}
Ruppert, D.
\newblock 1988.
\newblock Efficient estimations from a slowly convergent robbins-monro process.
\newblock Technical report, Cornell University Operations Research and
  Industrial Engineering.

\bibitem[\protect\citeauthoryear{Sutton \bgroup et al\mbox.\egroup
  }{2009}]{sutton2009fast}
Sutton, R.~S.; Maei, H.~R.; Precup, D.; Bhatnagar, S.; Silver, D.;
  Szepesv{\'a}ri, C.; and Wiewiora, E.
\newblock 2009.
\newblock Fast gradient-descent methods for temporal-difference learning with
  linear function approximation.
\newblock In {\em Proceedings of the 26th Annual International Conference on
  Machine Learning},  993--1000.
\newblock ACM.

\bibitem[\protect\citeauthoryear{Sutton, Maei, and
  Szepesv{\'a}ri}{2009}]{sutton2009convergent}
Sutton, R.~S.; Maei, H.~R.; and Szepesv{\'a}ri, C.
\newblock 2009.
\newblock A convergent o(n) temporal-difference algorithm for off-policy
  learning with linear function approximation.
\newblock In {\em Advances in neural information processing systems},
  1609--1616.

\bibitem[\protect\citeauthoryear{Sutton}{1988}]{sutton1988learning}
Sutton, R.~S.
\newblock 1988.
\newblock Learning to predict by the methods of temporal differences.
\newblock {\em Machine learning} 3(1):9--44.

\bibitem[\protect\citeauthoryear{Thoppe and
  Borkar}{2019}]{thoppe2019concentration}
Thoppe, G., and Borkar, V.
\newblock 2019.
\newblock A concentration bound for stochastic approximation via alekseev’s
  formula.
\newblock {\em Stochastic Systems} 9(1):1--26.

\end{thebibliography}
		
		\onecolumn
		
		\appendix
		
		    
		    
		    
		    
		
		    
		
		    

		\section{Proof of Proposition~\ref{prop:lower bound}: Lower Bound from the CLT}
		\label{sec:lowerBd}
		We first introduce the following necessary assumption.

		\mokkademConvergence~ (\cite{mokkadem2006convergence}[Assumption (A4)(ii)]) There exists a positive definite matrix $\Gamma$ such that
		\begin{equation}
		    \lim_{n \rightarrow \infty} \Exp\left(\begin{bmatrix}
		    \Mt_{n+1} \\
		    \Mw_{n+1} 
		    \end{bmatrix}
		    \left[{\Mt_{n+1}}^\top ~ {\Mw_{n+1}}^\top\right] \,\middle|\, \cF_n \right) = \Gamma
		    =  \begin{bmatrix}
		    \Gamma_{11} ~ \Gamma_{12} \\
		    \Gamma_{21} ~ \Gamma_{22}
		    \end{bmatrix}.
		\end{equation}
		
		\begin{theorem} \label{thm: lower bound}
		    Assume \assumMatrices, \assumStepsize, and \mokkademConvergence. Consider \eqref{eqn:tIter} and \eqref{eqn:wIter} with $\{\Mt_{n}\}$ and $\{\Mw_{n}\}$ being   $\mathbb{R}^{d}$-valued $(\theta_n,w_n)$-dominated martingale differences with parameters $\mt$ and $\mw$ (see Def.~\ref{assum:Noise}). Then, 
		    \begin{equation}
		        \|\theta_n - \thS\| = \Omega_p(n^{-\alpha/2}) \quad \text{ and } \quad \|w_n - \wS\| = \Omega_p(n^{-\beta/2}),
		    \end{equation}
		    where 
		$X_n = \Omega_p(\gamma_n)$ means that for every $\epsilon > 0,$ there are constants $c$ and $K$ such that $    \Pr\{|X_n|/\gamma_n < c\}\leq \epsilon, \, \forall n \geq K.$ As a consequence, for any $C \in (0, \infty)$ and positive sequence $\{g_n\}$ s.t. $\lim_{n\rightarrow{\infty}}g_n=0,$
		\begin{equation}
		\label{eqn:lower bound}
		    \lim_{n \to \infty} \Pr\left\{\|\theta_n - \thS\|\leq C n^{-\alpha/2} g_n\right\} = 0.
		\end{equation}
		A similar expression holds for $\|w_n - \wS\|.$
		\end{theorem}
		\begin{proof}
		The CLT in \cite{mokkadem2006convergence}[Theorem~1] shows that 
		\begin{align}
		    n^{\alpha/2}(\theta_n - \thS)  \Rightarrow {} & N(0, \Sigma_{\theta}), \label{eqn: CLT theta}\\
		    n^{\beta/2}(w_n - \wS)  \Rightarrow {} & N(0, \Sigma_{w})
		\end{align}
		for some covariance matrices $\Sigma_{\theta}$ and $\Sigma_w.$
		
		Let $\epsilon > 0.$ For any $c > 0,$ we have  
		\begin{equation}
		    \Pr\{n^{\alpha/2} \|\theta_n - \thS\|
		    < c\} \leq u(c) + v_n(c),
		\end{equation}
		where 
		\begin{equation}
		    u(c) = \Pr\{\|N(0, \Sigma_{\theta})\| < c\} 
		\end{equation}
		and
		\begin{equation}
		    v_n(c) = |\Pr\{\|N(0, \Sigma_{\theta})\| < c\} - \Pr\{n^{\alpha/2}\|\theta_n - \thS\| < c\}|.
		\end{equation}
		Pick $c \equiv c_{\epsilon}$ so that $u(c) \leq \epsilon/2.$ For this choice of $c,$ pick $K$ so that $v_n(c) \leq \epsilon/2$ for all $n \geq K;$ such a choice is possible because of \eqref{eqn: CLT theta} and the fact that $\|\cdot\|$ is continuous. From this, we can conclude that  that $\|\theta_n - \thS\| = \Omega_p(n^{-\alpha/2}),$ as desired.
		
		Let $C$ and $\{g_n\}$ be as in \eqref{eqn:lower bound}. Then, for any given $c > 0,$ 
		\begin{equation}
		    \{n^{\alpha/2}\|\theta_n - \thS\| \leq Cg_n\} \subseteq \{n^{\alpha/2}\|\theta_n - \thS\| \leq c\}
		\end{equation}
		for all sufficiently large $n.$ From this, it is easy to see that \eqref{eqn:lower bound} holds. 
		
		The statements on $\{w_n\}$ can be proved similarly.
		
		It remains to show that assumptions (A1)-(A4) in \cite{mokkadem2006convergence}[Section~2.1] hold in our setting as well; we do this now.
		\begin{enumerate}
		    \item \label{assum: A1 mokkadem} To show (A1), we first establish the stability of the iterates, i.e., $\sup_n(\|\theta_n\| + \|w_n\|) < \infty.$ For that, we employ \cite{lakshminarayanan2017stability}[Theorem~10] (whose conditions A1-A5 in that work can be easily verified). By invoking \cite{borkar2009stochastic}[Theorem~6.2], one can then see that both, $\{\theta_n\}$ and $\{w_n\},$ converge.  
		    \item Since
		    \begin{equation}
		    \begin{bmatrix}
		    \Gamma_1 & W_1 \\
		    \Gamma_2 & W_2
		    \end{bmatrix}
		    \begin{bmatrix}
		    \thS\\
		    \wS
		    \end{bmatrix}
		    =  \begin{bmatrix}
		    v_1 \\
		    v_2
		    \end{bmatrix},
		    \end{equation}
		    we have that
		    \begin{equation}
		    h_i(\theta, w) = {}  v_i - \Gamma_i \theta - W_i w 
		    = {} \Gamma_i \thS + W_i \wS - \Gamma_i \theta - W_i w 
		    = {} - \Gamma_i (\theta - \thS) - W_i(w - \wS).
		\end{equation}
		This establishes (A2).
		\item (A3) holds due to \assumMatrices in this work.
		\item Lastly, in (A4), (i) holds by definition of $\Mt_{n+1},\Mw_{n+1}$ (Def.~\ref{assum:Noise}); (ii) holds due to \mokkademConvergence; (iii) holds due to Def.~\ref{assum:Noise} and the stability condition in Item~\ref{assum: A1 mokkadem} above; and (iv) holds since $r_n^{(\theta)}=0,r_n^{(w)}=0$ in our linear case.
		
		\end{enumerate}
		\end{proof}
		It is now easy to see that Proposition~\ref{prop:lower bound} is a consequence of Theorem~\ref{thm: lower bound} where the algorithm of choice is one which, in addition to \assumMatrices ~and \assumStepsize, satisfies \mokkademConvergence.

		\section{Preliminaries}
		\subsection{Algebraic Manipulations}
		Using some easy manipulation on  \eqref{eqn:wIter} and \eqref{eqn:lht}, we get
		\begin{equation}
		w_{n} = -W_2^{-1} \left[\frac{w_{n + 1} - w_n}{\beta_n}\right] + W_2^{-1}[v_2 - \Gamma_2 \theta_n + M_{n + 1}^{(2)}]
		\end{equation}
		Substituting this in \eqref{eqn:tIter} gives
		\begin{multline}
		\theta_{n + 1} = \theta_n + \alpha_n \bigg[v_1 - \Gamma_1 \theta_n +	 W_1W_2^{-1} \left[\frac{w_{n + 1} - w_n}{\beta_n}\right] 
		- W_1 W_2^{-1}[v_2 - \Gamma_2 \theta_n] - W_1 W_2^{-1} M_{n + 1}^{(2)} + M_{n + 1}^{(1)}\bigg]
		\end{multline}
		Recall now from Section~\ref{sec: Main Result} that
		\begin{equation}
		    b_1 = v_1 - W_1 W_2^{-1}v_2
		\end{equation}
		and
		\begin{equation}
		    X_1 = \Gamma_1 - W_1 W_2^{-1}\Gamma_2.
		\end{equation}
		Therefore,
		\begin{equation} \label{eq: theta sophisticated update}
		\theta_{n + 1} = \theta_n + \alpha_n\left[b_1 - X_1 \theta_n\right] + \alpha_n\left[ W_1 W_2^{-1} \left[\frac{w_{n + 1} - w_n}{\beta_n}\right]\right] + \alpha_n \left[-W_1 W_2^{-1} M_{n + 1}^{(2)} + M_{n + 1}^{(1)}\right].
		\end{equation}

		Next, let
		\begin{equation} \label{def: rn}
		r_n = W_1 W_2^{-1} \left[\frac{w_{n + 1} - w_n}{\beta_n}\right] - W_1 W_2^{-1} M_{n + 1}^{(2)} + M_{n + 1}^{(1)}
		\end{equation}
		and 
		\begin{equation}
		    \thS = X_1^{-1}b_1.
		\end{equation}
		Then we can rewrite \eqref{eq: theta sophisticated update} as
		\begin{equation}
		\theta_{n + 1} - \thS = \theta_n - \thS + \alpha_n [-X_1( \theta_n - \thS) + r_n] = (I - \alpha_n X_1)(\theta_n - \thS) + \alpha_n r_n.
		\end{equation}
		Rolling out the iterates gives
		\begin{equation} \label{eq: theta prod form}
		\theta_{n + 1} - \thS = \prod_{k = n_0}^{n} (I - \alpha_k X_1) (\theta_{n_0} - \thS) + \sum_{k = n_0}^{n}\left[\prod_{j = k + 1}^{n}[I - \alpha_j X_1]\right] \alpha_k r_k.
		\end{equation}
		
		Similarly, recall that 
		\begin{equation}
		    \wS = W_2^{-1}(v_2 - \Gamma_2 \thS).
		\end{equation}
		It is easy to see from \eqref{eqn:wIter} that
		\begin{equation}
		w_{n + 1} - w^*= w_n - w^* + \beta_n[W_2 w^* + \Gamma_2 \thS - \Gamma_2 \theta_n - W_2 w_n + M_{n + 1}^{(2)}].
		\end{equation}
		This implies that
		\begin{equation}
		w_{n + 1} - w^*= (I - \beta_n W_2) (w_n - w^*) + \beta_n[- \Gamma_2 (\theta_n - \thS) + M_{n + 1}^{(2)}].
		\end{equation}
		Setting 
		\begin{equation} \label{eq: u def}
			s_n = [- \Gamma_2 (\theta_n - \thS) + M_{n + 1}^{(2)}]
		\end{equation}
		and rolling out the iterates gives 
		\begin{equation} \label{eq: w separation}
			w_{n+1} -w^* = \prod_{k=n_0}^n[I-\beta_k W_2] (w_{n_0}-w^*) + \sum_{k=n_0}^n \left[\prod_{j=k+1}^n[I-\beta_j W_2]\right] \beta_k s_k.
		\end{equation}
		
		\subsection{Definitions}
		Recall from \eqref{def: Lnt Main} that
		\begin{equation}
		    \Lnt = \sum_{k=n_0}^n  \left[\prod_{j=k+1}^n[I-\alpha_j X_1]\right] \alpha_k \left[- W_1 W_2^{-1} M_{k + 1}^{(2)} + M_{k + 1}^{(1)}\right], \label{def: Lnt}
		\end{equation}
		and define
		\begin{align}
		\Dnt &= \prod_{k = n_0}^{n} (I - \alpha_k X_1) (\theta_{n_0} - \thS), \label{def: Dnt}\\
		\Rnt &= \sum_{k=n_0}^n  \left[\prod_{j=k+1}^n[I-\alpha_j X_1]\right] \alpha_k \left[W_1 W_2^{-1} \left[\frac{w_{k + 1} - w_k}{\beta_k}\right]\right]. \label{def: Rnt}
		\end{align}
		Then, based on \eqref{def: rn} and \eqref{eq: theta prod form},
		\begin{equation}
		\theta_{n+1} - \thS =  \Dnt + \Lnt + \Rnt. \label{eq: tn components}
		\end{equation}
		
		Similarly, recall from \eqref{def: Lnw Main} that
		\begin{equation}
		    \Lnw = \sum_{k=n_0}^n  \left[\prod_{j=k+1}^n[I-\beta_j W_2]\right] \beta_k  M_{k + 1}^{(2)}, \label{def: Lnw}
		\end{equation}
		and define
		\begin{align}
		\Dnw &= \prod_{k = n_0}^{n} (I - \beta_k W_2) (w_{n_0} - \wS), \label{def: Dnw} \\
		\Rnw &= - \sum_{k=n_0}^n  \left[\prod_{j=k+1}^n[I-\beta_j W_2]\right] \beta_k \left[\Gamma_2 (\theta_k - \thS)\right]. \label{def: Rnw}
		\end{align}
		Then, based on \eqref{eq: u def} and \eqref{eq: w separation},
		\begin{equation}
		\label{eq: wn components}
		w_{n+1} - \wS = \Dnw + \Rnw +\Lnw.
		\end{equation}
		
		Lastly, notice that \eqref{def: Rnt} can also be written as 
		\begin{equation}
		\Rnt  = \sum_{k =n_0}^{n} \left(\prod_{j= k + 1}^{n } [I - \alpha_j X_1]\right)\alpha_k W_1 W_2^{-1} \left[\frac{(w_{k+1}-w^*) - (w_{k} - w^*)}{\beta_k}\right].
		\end{equation}

		From this, we have
		\begin{multline} \label{eq: Rn evolution}
		\Rnt = \frac{\alpha_n}{\beta_n} W_1 W_2^{-1}[w_{n + 1} - \wS] -  \left(\prod_{j = n_0 + 1}^{n} [I - \alpha_j X_1]\right) \frac{\alpha_{n_0}}{\beta_{n_0}} W_1 W_2^{-1}  [w_{n_0} - \wS] - T_{n + 1},
		\end{multline}
		where 
		\begin{equation}
		\label{Defn:Tn}
		T_{n+1} := \sum_{k = n_0 + 1}^{n} \left(\prod_{j= k+ 1}^{n } [I - \alpha_j X_1]\right) \left[\frac{\alpha_k}{\beta_k} I  - \frac{\alpha_{k - 1}}{\beta_{k - 1} } (I - \alpha_k X_1)\right]W_1 W_2^{-1} [w_k - \wS].
		\end{equation}

		\subsection{Technical Results}
		
		\begin{lemma}
		\label{lem: an bn upper bounds}
		Let $p \in (0,1)$ and $\hat{q} > 0.$ Let $K_{\ref{lem: an bn upper bounds}} = K_{\ref{lem: an bn upper bounds}}(p, \hat{q}) \geq 1$ be such that 
		\[
		e^{-\hat{q} \sum_{k = 1}^{n - 1}(k + 1)^{-p}} \leq n^{-p}
		\]
		for all $n \geq K_{\ref{lem: an bn upper bounds}};$ such an $K_{\ref{lem: an bn upper bounds}}$ exists as the l.h.s. is exponentially decaying. Let
		\[
		C_{\ref{lem: an bn upper bounds}} \equiv C_{\ref{lem: an bn upper bounds}}(p, \hat{q}) := \max_{1 \leq i \leq K_{\ref{lem: an bn upper bounds}}} i^p e^{-\hat{q} \sum_{k = 1}^{i - 1} (k+1)^{-p}}.
		\]
		Let $c_n := \sum_{i=0}^{n-1} [i + 1]^{-2p}e^{-2\hat{q}\sum_{k = i + 1}^{n-1}[k+1]^{-p}}.$
		Then,
		\[
		c_n \leq \frac{C_{\ref{lem: an bn upper bounds}}(p, \hat{q}) e^{\hat{q}}}{\hat{q}} n^{-p}.
		\]
		Accordingly,
		$a_n \leq C_{\ref{lem: an bn upper bounds},\theta}n^{-\alpha}$ and $b_n \leq C_{\ref{lem: an bn upper bounds},w}n^{-\beta}$ 
		where 
		$C_{\ref{lem: an bn upper bounds},\theta} =  \frac{C_{\ref{lem: an bn upper bounds}}(\alpha,q_1)e^{q_1}}{q_1}$,
		$C_{\ref{lem: an bn upper bounds},w} =  \frac{C_{\ref{lem: an bn upper bounds}}(\beta,q_2)e^{q_2}}{q_2},$
		$a_n = \sum_{k = 0}^{n - 1} \alpha_k^2 e^{-2q_1 \sum_{j=k+1}^{n - 1} \alpha_j}$ 
		and $b_n = \sum_{k = 0}^{n - 1} \beta_k^2 e^{-2q_2 \sum_{j=k+1}^{n - 1} \beta_j}$.
		\end{lemma}
		\begin{proof}
		The bound follows as in (45) from \cite{dalal2018finite}.
		\end{proof}
		
		Let $\lambda_{\min}$ and $\lambda_{\max}$ of a matrix denote its smallest and largest eigenvalue, respectively. Also, fix 
		\begin{equation}
		q_1 = \lambda_{\min}(X_1 + X_1^\tr)/4\label{eq: q1 defn}
		\end{equation}
		and 
		\begin{equation}
		q_2 = \lambda_{\min}(W_2 + W_2^\tr)/4.\label{eq: q2 defn}
		\end{equation}
		\begin{lemma} \label{lem: small eigenvalues}
		Let $K_{\ref{lem: small eigenvalues},\alpha}$ and $K_{\ref{lem: small eigenvalues},\beta}$ be such that
		\[
		\alpha_n \leq \frac{\lambda_{\min}(X_1 + X_1^\tr) - 2q_1}{\lambda_{\max}(X_1^\tr X_1)}, n \geq K_{\ref{lem: small eigenvalues},\alpha},
		\]
		and
		\[
		\beta_n \leq \frac{\lambda_{\min}(W_2 + W_2^\tr) - 2q_2}{\lambda_{\max}(W_2^\tr W_2)}, n \geq K_{\ref{lem: small eigenvalues},\beta}.
		\]
		Then, for $n \geq K_{\ref{lem: small eigenvalues},\alpha},$ %
		\[
		\|I - \alpha_n X_1\| \leq 1,
		\]
		and, for $n \geq K_{\ref{lem: small eigenvalues},\beta},$ 
		\[
		\|I - \beta_n W_2\| \leq 1.
		\]
		\end{lemma}
		\begin{proof}
		Observe that 
		\[
		\|I - \alpha_n X_1\| = \sqrt{\lambda_{\max}(I - \alpha_n (X_1 + X_1^\tr) + \alpha_n^2 (X_1^\tr X_1))}.
		\]
		Let $\lambda_n := \lambda_{\max}(I - \alpha_n (X_1 + X_1^\tr) + \alpha_n^2 (X_1^\tr X_1)).$
		Then, as in (7) from \cite{dalal2018finite}, we have $\lambda_n \leq e^{-2q_1 \alpha_n} \leq 1$ for $n \geq K_{\ref{lem: small eigenvalues},\alpha}.$ The desired result is now easy to see. The bound on $\|I - \beta_n W_2\|$ similarly holds. \end{proof}
		
		\begin{lemma} \label{lem: Dn bounds}
		For any $i \leq n,$
		\begin{eqnarray}
		\prod_{k=i}^n\|I-\alpha_k X_1\| & \leq & C_{\ref{lem: Dn bounds},\theta} e^{-q_1\sum_{k = i}^n \alpha_k}, 
		\\
		\prod_{k=i}^n\|I-\beta_k W_2\| & \leq & C_{\ref{lem: Dn bounds},w} e^{-q_2\sum_{k = i}^n \beta_k} \label{eq: Kw}.
		\end{eqnarray}
		Here, $C_{\ref{lem: Dn bounds},\theta} = \max\{1, \sqrt{
		\max_{\ell_1 \leq \ell_2 \leq K_{\ref{lem: Dn bounds},1}} \prod_{\ell = \ell_1}^{\ell_2} e^{\alpha_\ell(\mu_1 + 2q_1)}}\}$ with $K_{\ref{lem: Dn bounds},1}=\left\lceil\left(\frac{\lambda_{\max}(X_1^\tr X_1)}{\lambda_{\min}(X_1 + X_1^\tr) - 2q_1}\right)^{1/\alpha}\right\rceil$ and $\mu_1 = -\lambda_{\min}(X_1 + X_1^\tr) + \lambda_{\max}(X_1^\tr X_1) .$ Similarly, $C_{\ref{lem: Dn bounds},w} = \max\{1, \sqrt{
		\max_{\ell_1 \leq \ell_2 \leq K_{\ref{lem: Dn bounds},2}} \prod_{\ell = \ell_1}^{\ell_2} e^{\alpha_\ell(\mu_2 + 2q_2)}}\}$ with $K_{\ref{lem: Dn bounds},2}=\left\lceil\left(\frac{\lambda_{\max}(W_2^\tr W_2)}{\lambda_{\min}(W_2 + W_2^\tr) - 2q_2}\right)^{1/\beta}\right\rceil$ and 
		 $\mu_2 := -\lambda_{\min}(W_2 + W_2^\tr) + \lambda_{\max}(W_2^\tr W_2) .$ 
		\end{lemma}
		\begin{proof}
		For $K_{\ref{lem: Dn bounds},1}$ given in the statement, $\alpha_k \leq \frac{\lambda_{\min}(X_1 + X_1^\tr) - 2q_1}{\lambda_{\max}(X_1^\tr X_1)}$ for all $k \geq K_{\ref{lem: Dn bounds},1}.$  Then, it follows by arguing as in the proof of Lemma~4.1 in \cite{dalal2018finite} that
		\begin{equation}
		\label{Defn:K_theta}
		    \prod_{k=n_0}^n\|I-\alpha_k X_1\| = \sqrt{\prod_{k = n_0}^n \|I-\alpha_k (X_1 + X_1^\tr) + \alpha_k^2 X_1^\tr X_1 \|} \leq C_{\ref{lem: Dn bounds},\theta} e^{-q_1\sum_{k = n_0}^n \alpha_k}.
		\end{equation}
		The second part of the statement is proved analogously. 
		\end{proof}
		
		\begin{lemma}
		\label{lem: Delta bounds}
		Let $q_1, q_2, C_{\ref{lem: Dn bounds},\theta}, C_{\ref{lem: Dn bounds},w}$ be as defined in Lemma~\ref{lem: Dn bounds}. Let $n \geq n_0 - 1 \geq 0.$ Then, 
		\begin{eqnarray}
			\|\Delta_{n+1}^{(\theta)}\| & \leq & C_{\ref{lem: Dn bounds},\theta} e^{-q_1 \sum_{k=n_0}^n\alpha_k} \|\theta_{n_0} - \thS\|,
			\label{eqn:K Theta bound}\\
			\|\Dnw\| & \leq & C_{\ref{lem: Dn bounds},w} e^{-q_2 \sum_{k=n_0}^n\beta_k} \|w_{n_0} - \wS\|.
			\label{eqn: K  w bound}
		\end{eqnarray}
		\end{lemma}
		\begin{proof}
		From \eqref{def: Dnt}, we have
		\begin{equation}
			\|\Delta_{n+1}^{(\theta)}\| \leq \prod_{k=n_0}^n\|I-\alpha_k X_1\| \|\theta_{n_0} - \thS\|.
		\end{equation}
		For $n = n_0 - 1,$ the desired result follows since $C_{\ref{lem: Dn bounds},\theta} \geq 1,$ while, for $n \geq n_0,$ the result holds due to Lemma~\ref{lem: Dn bounds}. The second statement follows similarly.
		\end{proof}

		\begin{remark}
		This trivial lemma gives a much stronger convergence rate for $\|\Delta_{n+1}^{\theta}\|$, compared to Lemma~6 in \cite{mokkadem2006convergence}. It thus raises the following question. On the one hand, in Remark~4,\cite{mokkadem2006convergence}  the linear case is explained to be easier and does not require using the fact that $\Delta_n \rightarrow 0$. On the other hand, that remark refers in this simplified case to Eqs.~27-28 , which are fairly complex and are not decaying exponentially without using sophisticated successive upper bound tricks. These latter Eqs. also recursively consist of $L_{n+1}^{(\theta)} + R_{n+1}^{(\theta)}$. It thus implies that in \cite{mokkadem2006convergence} the derivation above can be tightened.
		\end{remark}

		\begin{lemma}
		\label{lem:Suff Cond. 2}
		Let $B_1,B_3 \geq 0$ with at least one of them being strictly positive, let $x, y \geq 0,$ and let $B_2 \geq 1$ be some constants. Then, for any $n \geq 0$ and $z \geq \max\{x, y\}$
		\begin{equation} \label{eq: unw cond}
		\frac{B_1 (n+1)^{-x}\sqrt{\ln(B_2(n+1))} +B_3 (n+1)^{-y}}{B_1(n+2)^{-x}\sqrt{\ln(B_2(n+2))} + B_3(n+2)^{-y}} \leq \frac{(n+1)^{-z}}{(n+2)^{-z}}.
		\end{equation}
		\end{lemma}
		\begin{proof}
		As $z \geq \max\{x, y\},$ it follows that
		\begin{equation}
		\left(\frac{n+2}{n+1}\right)^x \leq \left(\frac{n+2}{n+1}\right)^z.
		\end{equation}
		Hence, $(n+1)^{-x}(n+2)^{-z} \leq (n+2)^{-x}(n+1)^{-z}.$
		Similarly, $(n+1)^{-y}(n+2)^{-z} \leq (n+2)^{-y}(n+1)^{-z}.$
		Therefore, it is easy to see that
		\begin{multline}
		 B_1(n+1)^{-x}(n+2)^{-z} \sqrt{\ln(B_2(n + 1))} + B_3 (n+1)^{-y}(n+2)^{-z} \\ \leq B_1(n+2)^{-x }(n+1)^{-z} \sqrt{\ln(B_2(n + 2))} + B_3(n+2)^{-y}(n+1)^{-z}.
		\end{multline}
		The desired result now follows. 
		\end{proof}
		\begin{lemma} \label{eq: conditions for cond 2}
		Let $z \in [0, 1-(\alpha - \beta)],$
		\begin{equation} \label{eq: k_alpha def}
		    K_{\ref{eq: conditions for cond 2},\alpha}(z) = \max\bigg\{\ceil[\bigg]{\left(\frac{q_1}{2(\alpha-\beta + z)}\right)^{1/\alpha}}, \ceil[\bigg]{\left(\frac{4(\alpha - \beta + z)}{q_1}\right)^{1/(1-\alpha)}} \bigg\},
		\end{equation}
		and 
		\begin{equation} \label{eq: k_beta def}
		    K_{\ref{eq: conditions for cond 2},\beta}(z) = \max\bigg\{\ceil[\bigg]{\left(\frac{q_2}{\alpha-\beta + z}\right)^{1/\beta}}, \ceil[\bigg]{\left(\frac{4(\alpha - \beta + z)}{q_2}\right)^{1/(1-\beta)}} \bigg\}.
		\end{equation}
		Then,
		\begin{enumerate}
		    \item for $n \geq K_{\ref{eq: conditions for cond 2},\alpha}(z),$ \begin{equation} \label{eq: cond 2 for general case}
		\frac{(n+1)^{-z}}{(n+2)^{-z}} \leq \frac{\alpha_{n + 1}}{\alpha_n} \frac{\beta_n}{\beta_{n + 1}} e^{(q_1/2) \; \alpha_{n + 1}}, \mbox{ and}
		\end{equation}
		\item for $n \geq K_{\ref{eq: conditions for cond 2},\beta}(z),$
		\begin{equation} \label{eq: cond 3 for general case}
		\frac{(n+1)^{-z}}{(n+2)^{-z}} \leq \frac{\alpha_{n + 1}}{\alpha_n} \frac{\beta_n}{\beta_{n + 1}} e^{(q_2/2) \; \beta_{n + 2}}.
		\end{equation}
		\end{enumerate}

		\end{lemma}
		\begin{proof}
		We begin with the first statement. Let us now substitute the stepsizes.
		Let us write \eqref{eq: cond 2 for general case} as
		\begin{equation}
		    1 \leq \frac{{(n+2)}^{-z}}{(n+1)^{-z}}\frac{{(n+2)}^{-\alpha}}{(n+1)^{-\alpha}}\frac{{(n+1)}^{-\beta}}{(n+2)^{-\beta}} e^{q_1/2 \;\alpha_{n + 1}} = \left[ 1 + \frac{1}{n+1} \right]^{\beta - \alpha - z}  e^{q_1/2 \;\alpha_{n + 1}}. \label{eq: bound on stepsize ration}
		\end{equation}
		
		Next, we use a first-order approximation for the exponent. 
		Since $e^{(q_1/2) \alpha_{n + 1}} \geq 1 + (q_1/2) (n+2)^{-\alpha},$ to show \eqref{eq: bound on stepsize ration} it is enough to show that 
		$
		\left[1+\frac{1}{n + 1}\right]^{\beta - \alpha - z} \left[1 + (q_1/2) (n+2)^{-\alpha}\right] \geq 1;
		$
		that is,
		$
		 \left[1 + (q_1/2) (n+2)^{-\alpha}\right] \geq \left[1+\frac{1}{n + 1}\right]^{\alpha - \beta + z}.
		$
		
		For this, we shall show that $\left[1+\frac{1}{n + 1}\right]^{\alpha - \beta + z} \leq 1+\frac{\alpha - \beta + z}{n + 1}$, and later show \eqref{eq: 2nd part}. Denote $f(x) = (1+x)^{\alpha-\beta + z};$ then
		$ \left[1+\frac{1}{n+1}\right]^{\alpha-\beta + z} = f(1/(n+1)).
		$
		From the mean value theorem, $\exists c \in (0,1/(n+1))$ s.t. $f'(c) = \frac{f(1/(n+1)) - f(0)}{1/(n+1) - 0}$. Hence, $(\alpha-\beta + z)(1+c)^{({\alpha-\beta + z})-1} = (n+1)\left[\left(1+\frac{1}{n+1}\right)^{\alpha-\beta + z}-1\right].$ Therefore, $\left[1+\frac{1}{(n+1)}\right]^{\alpha-\beta + z} = 1+ \frac{\alpha-\beta + z}{(n+1)(1+c)^{1-(\alpha-\beta + z)}} \leq 1+\frac{\alpha-\beta + z}{n+1}.$
		The latter inequality holds because
		$
		    (1 + c)^{1 - (\alpha - \beta +z)} \geq 1,
		$
		which can be seen from the fact that 
		 $1 -(\alpha - \beta)\geq  z.$ 

		Now, we are left to show that 
		\begin{equation} \label{eq: 2nd part}
		 1+\frac{\alpha - \beta + z}{n + 1} \leq 1 + (q_1/2) (n+2)^{-\alpha}, 
		\end{equation}
		meaning that 
		\begin{equation} \label{eq: finishing the uk bound}
		\frac{2(\alpha - \beta + z)}{q_1} \leq \frac{n+1}{(n+2)^\alpha} = (n+2)^{1-\alpha} - (n+2)^{-\alpha},
		\end{equation}
		where the last relation holds by adding and subtracting $1$ in the numerator. 
		To show \eqref{eq: finishing the uk bound}, first notice that $(n+2)^{-\alpha} \leq \frac{2(\alpha - \beta + z)}{q_1}$ when $n \geq \left(\frac{q_1}{2(\alpha-\beta + z)}\right)^{1/\alpha} - 2.$ Therefore, \eqref{eq: finishing the uk bound} holds if $(n+2)^{1-\alpha} \geq \frac{4(\alpha - \beta + z)}{q_1}, $ which holds for $n \geq \left(\frac{4(\alpha - \beta + z)}{q_1}\right)^{1/(1-\alpha)} - 2.$ By imposing the condition $n_0 \geq K_{\ref{eq: conditions for cond 2},\alpha}(z)$, we obtain \eqref{eq: cond 2 for general case} and conclude the proof of the first statement.
		
		To show the second statement, it is now enough to show that
		$
		 \left[1 + (q_2/2) (n+3)^{-\beta}\right] \geq \left[1+\frac{1}{n + 1}\right]^{\alpha - \beta + z}.
		$
		The proof of this is very similar; the main difference is that instead of \eqref{eq: finishing the uk bound}, one obtains that
		\begin{equation} 
		\frac{2(\alpha - \beta + z)}{q_2} \leq \frac{n+1}{(n+3)^\beta} = (n+3)^{1-\beta} - 2(n+3)^{-\beta}
		\end{equation}
		holds when 
		$n \geq \left(\frac{q_1}{(\alpha-\beta + z)}\right)^{1/\beta} - 3$
		and
		$n \geq \left(\frac{4(\alpha - \beta + z)}{q_1}\right)^{1/(1-\beta)} - 3$.
		\end{proof}

		\begin{lemma} 
		\label{lem: const bound on eps}
		Given arbitrary constants $a>0$ and $A>0$, it holds that
		$n^{-\gamma}\log(n^p a) \leq A$ for any $n \geq \left[2\frac{p}{\gamma A}\ln\left(2\frac{p}{\gamma A}a^{\gamma/p}\right)\right]^{1/\gamma}$.
		Consequently, for
		$\epsth_n, \epsw_n$ as defined in  \eqref{eq: epsilon n new def}, %
		\begin{equation}
		    \epsth_n \leq \Rprojth / 2 \label{eq: epsth const bound}
		\end{equation}
		for  $n \geq K_{\ref{lem: const bound on eps}, \alpha} := \left[\frac{4d^3 L_\theta C_{\ref{lem: an bn upper bounds},\theta} p }{\alpha (\Rprojth)^2}\right]^{1/\alpha} \left[2\ln\left(2\frac{4d^3 L_\theta C_{\ref{lem: an bn upper bounds},\theta} p }{\alpha (\Rprojth)^2}\left[\frac{4 d^2}{\delta}\right]^{\alpha/p}\right)\right]^{1/\alpha},$ 
		and
		\begin{equation}
		    \epsw_n \leq \Rprojw / 2 \label{eq: epsw const bound}
		\end{equation}
		for  $n \geq K_{\ref{lem: const bound on eps}, \beta} := \left[\frac{4d^3 L_w C_{\ref{lem: an bn upper bounds},w} p }{\beta (\Rprojw)^2}\right]^{1/\beta}\left[2\ln \left(2\frac{4d^3 L_w C_{\ref{lem: an bn upper bounds},w} p }{\beta (\Rprojw)^2}\left[\frac{4 d^2}{\delta}\right]^{\beta/p}\right)\right]^{1/\beta}.$
		\end{lemma}
		
		\begin{proof}
		First note that, for any $C >0$ it holds that $C \ln(x) \leq x$ for $x$ equal to $2C \ln(2C)$ and, since $x$ grows faster than $\ln(x),$ it also holds for any $x$ larger than that.
		The first claim of the lemma follows by substituting $x = [a^{1/p}n]^\gamma$ and $C = \frac{p}{\gamma A}a^{\gamma/p}$. 
		Then \eqref{eq: epsth const bound} simply follows from this first claim by substituting $\gamma = \alpha$, $A = \dfrac{(\Rprojth/2)^2}{d^3 L_\theta C_{\ref{lem: an bn upper bounds},\theta}}$ and $a = \dfrac{4d^2}{\delta}$, and \eqref{eq: epsw const bound} by substituting $\gamma = \beta$, $A = \dfrac{(\Rprojw/2)^2}{d^3 L_w C_{\ref{lem: an bn upper bounds},\theta}}$ and $a = \dfrac{4d^2}{\delta}$.
		
		
		\end{proof}

		\section{Proof of Theorem~\ref{thm:Rates Proj Iterates}}
		\label{sec:Proof Main Result}
		
		Recall from Section~\ref{sec: proof outline} that the analysis is based on Theorem~\ref{thm:Main Res wo Proj}.
		Consequently, what we need to show is that, after the claimed number of iterations, sparse projections ensure $\Gp_{n0}$.
		In particular, we show that,
		after a time, these projections are not needed anymore, as the iterates remain in the close vicinity of $\thS$ and $\wS$ respectively, and the conclusions of the above Theorem take place.
		
		Before we start the proof, we need to analyze briefly the constants in the theorem. Let 
		\begin{equation}
		\label{defn: A3}
		A_3 = \Crw \Rprojw.
		\end{equation}
		\begin{equation}
		\label{Defn:A4}
		     A_{4,n_0} = \left[A_{1,n_0} \sum_{i=0}^{\ceil{\frac{\beta}{2(\alpha-\beta)}}-1} A_2^i\right] \sqrt{d^3 L_w C_{\ref{lem: an bn upper bounds},w}} + \left[A_3 A_2^{\ceil{\frac{\beta}{2(\alpha-\beta)}}}\right],
		\end{equation}
		and 
		\begin{equation}
		\label{Defn:A5}
		    A_{5,n_0} 
		    = 2\left[C_{\ref{lemma: R_n w bound},a}  \left[\Crth\Rprojth + \Crw \Rprojw \right] /\epsth_{n_0 - 1} + 1\right]\sqrt{4d^3 L_\theta C_{\ref{lem: an bn upper bounds},\theta}} 
		    +  2C_{\ref{lemma: R_n w bound},b} A_{4, n_0}.
		\end{equation}

		\begin{lemma} \label{lemma: A4' A5'}
		Assume $\Gp_{n_0}$ holds. Let
		\begin{equation}
		    \label{defn: A'4 and A'5}
		  A_4' = A_{4,C_1} + 1, ~  A_5' = 4 + 2A_{5,C_1} + 2C_{\ref{lemma: R_n w bound},b} A_{4,C_1},
		\end{equation} 
		where
		\begin{equation}
		    \label{defn: A4C1}
		    A_{4,C_1} = d^3 L_w C_{\ref{lem: an bn upper bounds},w}\left(C_{\ref{lem: Dn bounds},w}\|\Gamma_2\| \left[\Rprojth +  C_{\ref{lemma: R_n w bound},a}  e^{q_1} \frac{2}{q_{\min}}    \left(\Rprojth + \Rprojw\right)\right] + C_{\ref{lem: Dn bounds},w} \Rprojw\right)e\sum_{i=0}^{\ceil{\frac{\beta}{2(\alpha-\beta)}}-1} A_2^i,
		\end{equation}
		\begin{equation}
		    \label{defn: A5C1}
		    A_{5,C_1} = C_{\ref{lemma: R_n w bound},a}
		     \left[\Crth\Rprojth + \Crw \Rprojw \right].
		\end{equation}
		Then recalling $A_{4,n_0}$ from \eqref{Defn:A4} and $A_{5,n_0}$ from \eqref{Defn:A5}, 
		\begin{align}
		    A_{4,n_0} &\leq A_4' (n_0+1)^{\beta/2} \left(\ln{(4d^2(n_0+1)^p/\delta)}\right)^{-1/2} \label{eqn:Simplified A4 Expr}\\
		    A_{5,n_0} & \leq  A_5'(n_0+1)^{\alpha/2} \left(\ln{(4d^2(n_0+1)^p/\delta)}\right)^{-1/2} 
		    \; \label{eqn:Simplified A5 Expr}
		\end{align}
		if
		\begin{equation}
		    n_0 \geq \max\{K_{\ref{lemma: A4' A5'},a}, K_{\ref{lemma: A4' A5'},b}\},
		\end{equation}
		where
		\begin{align}
		    \label{defn: CA4'A5'Lemma a}
		    K_{\ref{lemma: A4' A5'},a} & = \left[\frac{p }{\beta (A_{4,C_0})^2}\right]^{1/\beta}\left[2\ln \left(2\frac{p }{\beta (A_{4,C_0})^2}\left[\frac{4 d^2}{\delta}\right]^{\beta/p}\right)\right]^{1/\beta}, \\
		    \label{defn: CA4'A5'Lemma b}
		    K_{\ref{lemma: A4' A5'},b} & = \left[\frac{p }{\alpha (\min\{C_{\ref{lemma: R_n w bound},b}A_{4,C_0},A_{5,C_0}\})^2}\right]^{1/\alpha}\left[2\ln \left(2\frac{p }{\alpha (\min\{C_{\ref{lemma: R_n w bound},b}A_{4,C_0},A_{5,C_0}\})^2}\left[\frac{4 d^2}{\delta}\right]^{\alpha/p}\right)\right]^{1/\alpha},
		\end{align}
		with
		\begin{align}
		    \label{defn:A4C0}
		    A_{4,C_0} & = \left[(e + e^2 A_1'')\left(\sum_{i=0}^{\ceil{\frac{\beta}{2(\alpha-\beta)}}-1} A_2^i \sqrt{d^3 L_w C_{\ref{lem: an bn upper bounds},w}}\right) + A_3 A_2^{\ceil{\frac{\beta}{2(\alpha-\beta)}}}\right], \\
		    \label{defn:A5C0}
		    A_{5,C_0} & = \sqrt{4d^3 L_\theta C_{\ref{lem: an bn upper bounds},\theta}}.
		\end{align}
		\end{lemma}
		
		\begin{proof}
		First, we  upper bound $C_{\ref{lemma: R_n w bound},c}(n_0)$ based on the definition of $\Gp_{n_0}:$
		\begin{align}
		C_{\ref{lemma: R_n w bound},c}(n_0) &= \left[\beta_{n_0} \|\theta_{n_0} - \thS\| +  C_{\ref{lemma: R_n w bound},a}  e^{q_1} \frac{2}{q_{\min}}    \left[\|\theta_{n_0} - \thS\| + \frac{\alpha_{n_0}}{\beta_{n_0}}\|w_{n_0} - \wS\|\right]\right]\\
		& \leq \left[\Rprojth +  C_{\ref{lemma: R_n w bound},a}  e^{q_1} \frac{2}{q_{\min}}    \left(\Rprojth + \Rprojw\right)\right].
		\end{align}
		Next, using the definition of $A_{4,n_0}$ from \eqref{Defn:A4} ,
		\eqref{eqn:Simplified A4 Expr} holds 
		if
		\begin{equation} \label{eq: upper bound on A4c_0}
		    (n_0+1)^{\beta/2} \left(\ln{(4d^2(n_0+1)^p/\delta)}\right)^{-1/2} \geq A_{4,C_0},
		\end{equation}
		Based on Lemma~\ref{lem: const bound on eps}, \eqref{eq: upper bound on A4c_0} holds if
		$    n_0 \geq  K_{\ref{lemma: A4' A5'},a}.
		$
		%
		This completes the proof of \eqref{eqn:Simplified A4 Expr}.
		
		Next, recall from \eqref{Defn:A5} that
		\begin{align}
		    A_{5,n_0}
		    =& 
		    2C_{\ref{lemma: R_n w bound},a}
		    \sqrt{4d^3 L_\theta C_{\ref{lem: an bn upper bounds},\theta}} \left[\Crth\Rprojth + \Crw \Rprojw \right] /\epsth_{n_0 - 1} 
		    +
		    2\sqrt{4d^3 L_\theta C_{\ref{lem: an bn upper bounds},\theta}} 
		    + 
		    2C_{\ref{lemma: R_n w bound},b} A_{4, n_0}
		    \\
		    =&
		    \frac{2C_{\ref{lemma: R_n w bound},a}
		    \sqrt{4d^3 L_\theta C_{\ref{lem: an bn upper bounds},\theta}} \left[\Crth\Rprojth + \Crw \Rprojw \right](n_0+1)^{\alpha/2}}{\sqrt{d^3 L_\theta C_{\ref{lem: an bn upper bounds},\theta}   \ln{(4d^2(n_0+1)^p/\delta)}}}
		    +
		    2\sqrt{4d^3 L_\theta C_{\ref{lem: an bn upper bounds},\theta}} 
		    + 
		    2C_{\ref{lemma: R_n w bound},b} A_{4, n_0}
		    \\
		    =&
		    2\Big[
		    (n_0+1)^{\alpha/2}(   \ln{(4d^2(n_0+1)^p/\delta)})^{-1/2}
		    A_{5,C_1}
		    +
		    A_{5,C_0} 
		    \\
		    &\hspace{.4cm}+ 
		    (n_0+1)^{\beta/2} \left(\ln{(4d^2(n_0+1)^p/\delta)}\right)^{-1/2} C_{\ref{lemma: R_n w bound},b} A_{4,C_1} 
		    + 
		    C_{\ref{lemma: R_n w bound},b} A_{4,C_0}
		    \Big].
		\end{align}
		Therefore, again based on Lemma~\ref{lem: const bound on eps}, \eqref{eqn:Simplified A5 Expr} holds 
		when
		$    n_0
		    \geq
		    K_{\ref{lemma: A4' A5'},b}.
		$
		\end{proof}

		Let
		\begin{equation}
		\label{defn: Rates Proj Iterates Kw}
		K_{\ref{thm:Rates Proj Iterates}, w} =  
		[(A'_4/\Rprojw)^{2/\beta}]^{(A'_4/\Rprojw)^{2/\beta}} 
		\end{equation}
		and 
		\begin{equation}
		\label{defn: Rates Proj Iterates Kth}
		K_{\ref{thm:Rates Proj Iterates}, \theta} = [(A'_5/\Rprojth)^{2/\alpha}]^{(A'_5/\Rprojth)^{2/\alpha}}.
		\end{equation}
		Also, define
		\begin{align}
		    C_{\ref{thm:Rates Proj Iterates},\theta}
		    &=
		    A'_5/\nu(N_{\ref{thm:Rates Proj Iterates}, \alpha}),
		    \label{eq: Rates Proj Iterates Cth}
		    \\
		    C_{\ref{thm:Rates Proj Iterates},w}
		    &=
		    A'_4/\nu(N_{\ref{thm:Rates Proj Iterates}, \beta}).
		    \label{eq: Rates Proj Iterates Cw}
		\end{align}

		We are now ready to prove the theorem.

		\begin{proof}[Proof of Theorem~\ref{thm:Rates Proj Iterates}]
		Recall that whenever $n_0 = k^k -1$ for some $k \in \Int_{> 0},$ then event $\Gp_{n_0}$ holds with probability $1$ for the projected iterates. Let $(\theta_n, w_n)_{n \geq n_0}$ be the iterates obtained by running the unprojected algorithm given in \eqref{eqn:tIter} and \eqref{eqn:wIter} with $\theta_{n_0} = \theta'_{n_0}$ and $w_{n_0} = w'_{n_0}.$
		Define $f(x) = x^x$ and note that if we project in round $n_0$ then, by definition, $n_0 = f(k)-1$ for some positive integer $k$, and the next time we project will be in round $g(n_0) = f(1+k)-1 =  f\left(1+ f^{-1}(n_0+1) \right)-1$. Therefore,
		\begin{eqnarray}
		    \cI & := & \{\|\theta_j - \thS\| \leq \Rprojth, \|w_j - \wS\| \leq \Rprojw, \forall j \geq g(n_0) \} \label{eqn:Rel1 Proj vs UnProj}\\
		    & \subseteq  & \{\theta_j = \Pi_{j, \Rprojth}(\theta_j), w_j = \Pi_{j, \Rprojw}(w_j), \forall j \geq g(n_0)\} \\ 
		    & =  & \{\theta_j = \Pi_{j, \Rprojth}(\theta_j), w_j = \Pi_{j, \Rprojw}(w_j), \forall j \geq n_0\} \label{eqn:Rel2 Proj vs UnProj} .
		\end{eqnarray}
		
		Consider the following coupling: 
		\begin{align}
		(\tilde{\theta}_{n}',\tilde{w}_{n}')
		:=
		\begin{cases}
		({\theta}_n',{w}_n'), & \mbox{ for } 0 \leq n < n_0 \enspace,\\
		({\theta}_n,{w}_n), & \mbox { for } n \geq n_0 \mbox{ on the event $\mathcal{I}$ },\\
		({\theta}_n',{w}_n'), & \mbox{ for } n \geq  n_0 \mbox{ on the complement of the event $\mathcal{I}\enspace.$ }
		\end{cases}
		\end{align}
		Due to \eqref{eqn:Rel1 Proj vs UnProj} - \eqref{eqn:Rel2 Proj vs UnProj}, $(\tilde{\theta}_n',\tilde{w}_n')_{n\geq 0}$ and $(\theta'_n, w'_n)_{n\geq 0}$ are distributed identically.

		
		Consequently, it is easy to see that  Theorem~\ref{thm:Main Res wo Proj} applies to $\{(\theta'_{n}, w'_n)\}$ provided we show that the event $\cI$ holds, i.e.,
		\begin{eqnarray}
		    A_{5, n_0} (n + 1)^{-\alpha/2} \sqrt{\ln{(4d^2(n+1)^p/\delta)}} & \leq & \Rprojth \\
		    A_{4, n_0} (n+1)^{-\beta/2} \sqrt{\ln{(4d^2(n+1)^p/\delta)}} & \leq & \Rprojw 
		\end{eqnarray}
		for all $n \geq g(n_0).$ In fact, using Lemma~\ref{lemma: A4' A5'} together with Theorem~\ref{thm:Main Res wo Proj},
		\begin{align}
		\|\theta'_{n} - \thS\| & \leq \frac{  A_5'}{\nu(n_0, \alpha)} \nu(n, \alpha)  \label{eqn:des Theta Bd} 
		\\
		\|w'_{n} - w^*\| & \leq  \frac{A'_4}{\nu(n_0, \beta)} \nu(n, \beta) \label{eqn:des w Bd}
		\end{align}
		as desired, for $n \geq n_0 \geq \max\{ N_{\ref{thm:Main Res wo Proj}}, K_{\ref{lemma: A4' A5'},a}, K_{\ref{lemma: A4' A5'},b}\}$, provided we show that 
		\begin{align}
		     \frac{A_5'}{\nu(n_0, \alpha)} \nu(n, \alpha) \leq {} & \Rprojth  \label{eq: iterate bounded by RProjth} \\
		     \frac{A'_4}{\nu(n_0, \beta)} \nu(n, \beta) \leq {} & \Rprojw \label{eq: iterate bounded by RProjw}
		\end{align}
		for all $n \geq g(n_0).$
		As we show below, this holds when 
		$n_0 \geq \max\left\{
		K_{\ref{thm:Rates Proj Iterates}, w}, K_{\ref{thm:Rates Proj Iterates}, \theta}, 
		e^{1/\alpha}, e^{1/\beta},  (2/\alpha)^{2/\alpha}, (2/\beta)^{2/\beta}  
		\right\}$.


		

		It is clear that,
		in order to show that \eqref{eq: iterate bounded by RProjw} holds, it suffices to show
		that for $n = g(n_0)$

		\begin{equation}
		\label{eqn: Criteria For Proj Eq Unproj A4}
		     \left(\frac{n + 1}{n_0 + 1}\right)^{\beta} \frac{\ln \left[\left[\frac{4d^2}{\delta}\right]^{\beta/p} \left(n_0 + 1\right)^{\beta}\right]}{\ln \left[\left[\frac{4d^2}{\delta}\right]^{\beta/p} \left(n + 1\right)^{\beta}\right]} \geq \frac{A_4'^2}{\left(\Rprojw \right)^2}.
		\end{equation}
		and that
		\(
		\left(n + 1\right)^{\beta} 
		\left({\ln \left[\left[4d^2/\delta\right]^{\beta/p} \left(n + 1\right)^{\beta}\right]}\right)^{-1}
		\)
		is monotonically decreasing.

		
		%
		Since $n_0 = f(k)-1$ for some positive integer $k$ and $n = f(k+1)-1$, letting $A = \frac{\beta}{p} \ln \left[\frac{4d^2}{\delta}\right] \geq 0,$ we have
		\begin{align}
		    \left(\frac{n+1}{n_0+1}\right)^\beta
		    \frac{A+\beta\ln(n_0+1)}{A+\beta\ln(n+1)}
		    =&
		    \left(\frac{(k+1)^{k+1}}{k^ k}\right)^\beta
		    \frac{A+\beta k \ln k}{A+\beta(k+1)\ln(k+1)}
		    \\
		    \geq&
		    (k+1)^\beta\left(\frac{k+1}{k}\right)^{\beta k}
		    \frac{ k \ln k}{(k+1)\ln(k+1)}
		    \label{eq: dropping the constant from the ratio}
		    \\
		    \geq
		    &
		    (k+1)^\beta\left(\frac{k+1}{k}\right)^{\beta k}
		    \left(\frac{k}{k+1}\right)\frac{\ln k}{1/k + \ln k}
		    \label{eq: bound based on concavity}
		    \\
		    \geq&
		    (k+1)^\beta\left(\frac{k+1}{k}\right)^{\beta k-2}
		    \label{eq: bounding the ratio with the log}
		    \\
		    \geq&
		    (k+1)^\beta
		\end{align}
		where 
		\eqref{eq: dropping the constant from the ratio} follows because $(A+B_1)/(A+B_2) \geq B_1/B_2$ for any $A\geq 0$ and $B_2\geq B_1 > 0$ due to $(A+B_1)B_2 \geq (A+B_2)B_1$,
		\eqref{eq: bound based on concavity} follows because $\ln(x+1) - \ln x \leq 1/x$ due to the fact that $\ln x$ is concave and has derivative $1/x$,
		\eqref{eq: bounding the ratio with the log} holds when $k>1$ due to $\tfrac{\ln k}{1/k + \ln k} = \tfrac{k}{1/(\ln k) + k} \geq \tfrac{k}{1 + k}$
		, and the last inequality holds when 
		$k \geq 2/\beta$, or, equivalently when $n_0 \geq (2/\beta)^{2/\beta}-1$. 
		Consequently, \eqref{eqn: Criteria For Proj Eq Unproj A4} holds if $k \geq (A'_4/\Rprojw)^{2/\beta}$ or,  equivalently, if $n_0 \geq K_{\ref{thm:Rates Proj Iterates}, w}.$
		
		Showing the monotonicity of
		\(
		\left(n + 1\right)^{\beta} 
		\left({\ln \left[\left[4d^2/\delta\right]^{\beta/p} \left(n + 1\right)^{\beta}\right]}\right)^{-1}
		\)
		goes similarly:
		\begin{align}
		    \frac{(n + 1)^\beta}{n^{\beta} }
		    \frac{\ln \left[\left[\frac{4d^2}{\delta}\right]^{\beta/p} \left(n\right)^{\beta}\right]}
		    {\ln \left[\left[\frac{4d^2}{\delta}\right]^{\beta/p} \left(n + 1\right)^{\beta}\right]}
		    =&
		    \left(\frac{n + 1}{n}\right)^{\beta} \frac{\frac{\beta}{p}\ln \left[\frac{4d^2}{\delta}\right] + \beta \ln n}
		    {\frac{\beta}{p}\ln \left[\frac{4d^2}{\delta}\right] + \beta \ln[n+1]}
		    \\
		    \geq&
		    \frac{[n + 1]^{\beta}}{n^{\beta}} \frac{ \ln n}
		    { \ln[n+1]}
		    \\
		    \geq&
		    \frac{n^\beta + \beta(n+1)^{\beta-1}}{n^{\beta}} \frac{ \ln n}
		    { (1/n)+\ln n}
		    \\
		    =&
		    \frac{n^\beta \ln n + \beta(n+1)^{\beta-1} \ln n}{n^{\beta}\ln n + n^{\beta-1}}
		    \\
		    \geq&
		    1
		\end{align}
		where the last inequality holds for $n \geq e^{1/\beta}$.
		
		Using an argument similar to the one above, it is easy to see that  \eqref{eq: iterate bounded by RProjth} holds since $n_0 \geq \max\{e^{1/\alpha}, (2/\alpha)^{2/\alpha}, K_{\ref{thm:Rates Proj Iterates}, \theta} \}$ which is true since $e^{1/\beta} \geq e^{1/\alpha}$ and $(2/\beta)^{2/\beta} \geq (2/\alpha)^{2/\alpha}.$
		\\
		
		Now, substituting $n_0 = N_3$ in \eqref{eqn:des Theta Bd} and \eqref{eqn:des w Bd} gives us the desired result. 
		\end{proof}
		\begin{remark}
		The above result introduces double exponential complexity in $1/\alpha$ and $1/\beta$ via, e.g., $K_{\ref{thm:Rates Proj Iterates}, w}$ and $K_{\ref{thm:Rates Proj Iterates}, \theta}.$ One can try and obtain better bounds by increasing the sparsity of the projections. Nevertheless, we argue that, at least for square-summable step-sizes (i.e., for $\alpha>1/2$ and $\beta>1/2$), this double-exponential bound is not too bad.
		\end{remark}


		\section{Details omitted from the proof of Theorem~\ref{thm: En0 prob bound}}

		Recalling $\cU(n_0)$ from \eqref{def: E_n0} and $\cZ_n$ from \eqref{eq: Zn defn},
		%
		we have
		\begin{align}
		    \cU^c(n_0) & = \bigcup_{n \geq n_0} \cZ_n^c \label{eq: U complement reformulated}\\
		    & = \cZ_{n_0}^c \cup (\cZ_{n_0} \cap \cZ^c_{n_0 + 1}) \cup ([\cZ_{n_0} \cap \cZ_{n_0 + 1}] \cap \cZ_{n_0 + 2}^c) \cup \ldots \\
		    & = \cZ_{n_0}^c \cup (\cA_{n_0} \cap \cZ_{n_0 + 1}^c) \cup (\cA_{n_0 + 1} \cap \cZ_{n_0 + 2}^c) \cup \ldots 
		\end{align}
		This implies that 
		\begin{equation}
		\label{eqnA:Gpnp interesect Un0 compl}
		    \Gp_{n_0} \cap \cU^c(n_0) = (\Gp_{n_0} \cap \cZ_{n_0}^c) \cup (\Gp_{n_0} \cap \cA_{n_0} \cap \cZ_{n_0 + 1}^c) \cup (\Gp_{n_0} \cap \cA_{n_0 + 1} \cap \cZ_{n_0 + 2}^c) \cup \ldots
		\end{equation}
		Recalling that
		\begin{align}
		    \Crth & = 3 \mbox{ and } \label{def: Crth}\\
		    \Crw & = 3/2 + (e^{q_2} /q_2 \|\Gamma_2\| C_{\ref{lem: Dn bounds},w}) \Crth \frac{\Rprojth}{\Rprojw}, \label{def: Crw}
		\end{align}
		one can see that both are lower bounded by 1, and
		hence, 
		\begin{equation}
		%
		    \Gp_{n_0} \cap \cZ_{n_0}^c \subseteq \cG_{n_0} \cap \cZ_{n_0}^c  \subseteq \cG_{n_0} \cap (\{\|L_{n_0 + 1}^{(\theta)}\| > \epsth_{n_0}\} \cup \{ \|L_{n_0 + 1}^{(w)}\| > \epsw_{n_0}\}).
		\end{equation}
		Similarly, observe that %
		\begin{align}
		    \Gp_{n_0} \cap \cA_n \cap \cZ^c_{n + 1} 
		    = & \Gp_{n_0} \cap \cA_{n} \cap \left[\{\|\theta_{n + 1} - \thS\| > \Crth \Rprojth\} \cup \{\|w_{n + 1} - \wS\| > \Crw \Rprojw\} \right] \\
		& \cup \Bigg(\Gp_{n_0} \cap \cA_{n} \cap \left[\{\|\theta_{n + 1} - \thS\| \leq \Crth \Rprojth\} \cap \{\|w_{n + 1} - \wS\| \leq \Crw \Rprojw\} \right] \cap \\
		& \bigg[\{\|L_{n + 2}^{(\theta)}\| > \epsth_{n + 1}\} \cup \{\|L_{n + 2}^{(w)}\| > \epsw_{n + 1}\}\bigg]\Bigg) \\
		\subseteq & \Gp_{n_0} \cap \cA_{n} \cap \left[\{\|\theta_{n + 1} - \thS\| > \Crth \Rprojth\} \cup \{\|w_{n + 1} - \wS\| > \Crw \Rprojw\} \right] \\ 
		& \cup \left(\cG_{n + 1} \cap \left[\{\|L_{n + 2}^{(\theta)}\| > \epsth_{n + 1}\} \cup \{\|L_{n + 2}^{(w)}\| > \epsw_{n + 1}\}\right]\right) \\
		%
		\label{eqn:Gpn0 intersect An intersect Zn1 compl app}
		\end{align}
		%


		\begin{remark}
		One could have obtained an exponentially decaying bound in \eqref{eq: delta bound for U} by defining  $\epsth_n$ to be $\sqrt{d^3 L_\theta C_{\ref{lem: an bn upper bounds},\theta} (n+1)^{-\alpha + p'}  \ln{(4d^2/\delta)}}$ instead of the current definition given in  \eqref{eq: epsilon n new def}. This bound would then be in the same spirit as that of \cite[Chapter 4, Corollary 14]{borkar2009stochastic} (see the 2nd display there). However, the additional $(n + 1)^{p'}$ term means that the new $\epsth_n$ decays at a slower rate and thereby slows down the rate of convergence of the $\{\theta_n\}$ iterates derived in Theorem~\ref{thm:Rates Proj Iterates}. The same discussion applies for $\epsw_n$ and the $\{w_n\}$ iterates as well.
		\end{remark}
		
		\section{A key lemma}
		
		%
		\begin{lemma} \label{lem: Tn bound}
		Let $n \geq n_0 \geq 0.$ Let $u \in \RInf$  be $\alpha$-moderate from $n_0$ onwards (see Def.~\ref{def: wn behavior 2}) and suppose the event $\cW_{n}(u)$ holds (see \eqref{def: wn}). Then,
		\[
		\|T_{n+1}\| \leq \frac{2e^{q_1/2}}{q_1} C_{\ref{lem: Tn bound}} \|W_1 W_2^{-1}\| C_{\ref{lem: Dn bounds},\theta} \frac{\alpha_n}{\beta_n} \uw_n,
		\]
		where
		\begin{equation}
		\label{def: CU}
		C_{\ref{lem: Tn bound}} := [\|X_1\| + 
		2(\alpha - \beta) [1 + \|X_1\|]].
		\end{equation}
		\end{lemma}
		\begin{proof}
		From \eqref{Defn:Tn}, it is easy to see that
		\begin{equation}
		\label{eqn:T_Bd}
		\|T_{n+1}\| \leq  \sum_{k = n_0 + 1}^{n} \left(\prod_{j= k+ 1}^{n } \left\|I - \alpha_j X_1 \right\|\right) \alpha_ k \|U_k\|,
		\end{equation}
		where
		\begin{equation}
		U_k := \frac{1}{\beta_k} \left[I  - \frac{\alpha_{k - 1}}{\alpha_k} \frac{\beta_k}{ \beta_{k - 1} } (I - \alpha_k X_1)\right] W_1 W_2^{-1}[w_k - \wS].
		\end{equation}
		
		Observe that
		\begin{equation}
		U_k = 
		%
		\frac{\alpha_k}{\beta_k} \left[X_1 + \frac{1}{\alpha_k} \left(1  - \frac{\alpha_{k - 1}}{\alpha_k} \frac{\beta_k}{ \beta_{k - 1} }\right) (I - \alpha_k X_1) \right]W_1 W_2^{-1} [w_k - \wS].
		\end{equation}
		We now show 
		\begin{equation}
		\left\|\left[X_1 + \frac{1}{\alpha_k} \left(1  - \frac{\alpha_{k - 1}}{\alpha_k} \frac{\beta_k}{ \beta_{k - 1} }\right) (I - \alpha_k X_1) \right]\right\| 
		\end{equation}
		can be bounded by a constant. In particular, it suffices to show that
		\begin{equation}
		B_k := \frac{1}{\alpha_k} \left(1  - \frac{\alpha_{k - 1}}{\alpha_k} \frac{\beta_k}{ \beta_{k - 1} }\right)
		= (k + 1)^\alpha \left(1 - \left(1 + \frac{1}{k}\right)^{\alpha - \beta} \right)
		\end{equation}
		is bounded by a constant.
		%
		%
		To this end, let $f(x) = (1 + x)^{\alpha - \beta}.$ Then, by the mean value theorem, there is a $c \in (0, 1/k)$ such that
		\begin{equation}
		k[f(1/k) - f(0)] = f'(c).
		\end{equation}
		Noting that 
		\[
		|f'(c)| = \left|\frac{\alpha - \beta}{(1 + c)^{1 - (\alpha - \beta)}}\right| \leq \alpha - \beta,
		\]
		where the inequality follows since $c \in (0,1/k),$
		we obtain
		%
		\begin{equation}
		|B_k| = (k+1)^\alpha|f(0)-f(1/k)| = (k+1)^\alpha\frac{|f'(c)|}{k}\leq (k + 1)^\alpha \frac{\alpha - \beta}{k} \leq 2 (\alpha - \beta).
		\end{equation}
		%
		From this, it follows that 
		%
		\begin{eqnarray}
		\|U_k\| & \leq & \frac{\alpha_k}{\beta_k}[\|X_1\| + |B_k| \; \|I - \alpha_k X_1\| ] \; \|W_1 W_2^{-1}\| \; \|w_k -  \wS\| \\
		& \leq & C_{\ref{lem: Tn bound}} \|W_1 W_2^{-1}\| \; \frac{\alpha_k}{\beta_k} \|w_k - \wS\|. \label{eqn:U_bd}
		\end{eqnarray}
		
		Substituting \eqref{eqn:U_bd} in \eqref{eqn:T_Bd}, we get
		
		%
		\begin{eqnarray}
		\|T_{n+1}\| & \leq &  C_{\ref{lem: Tn bound}} \|W_1 W_2^{-1}\| \sum_{k = n_0 + 1}^{n} \left(\prod_{j= k+ 1}^{n } \left\|I - \alpha_j X_1 \right\|\right) \alpha_ k \frac{\alpha_k} {\beta_k} \|w_k -  \wS\|, \label{eq: bounding Tn} \\
		& \leq & C_{\ref{lem: Tn bound}} \|W_1 W_2^{-1}\| C_{\ref{lem: Dn bounds},\theta} \sum_{k = n_0 + 1}^{n} e^{-q_1\sum_{j = k + 1}^{n} \alpha_j} \alpha_k \frac{\alpha_k}{\beta_k}\|w_k -  \wS\|\label{eq: bounding the product},\\
		& \leq & C_{\ref{lem: Tn bound}} \|W_1 W_2^{-1}\| C_{\ref{lem: Dn bounds},\theta} \left(\sup_{n_0 \leq k \leq n} e^{-q_1/2\sum_{j = k + 1}^{n} \alpha_j} \frac{\alpha_k}{\beta_k}\|w_k -  \wS\|\right) \sum_{k = 0}^{n} e^{-q_1/2\sum_{j = k + 1}^{n} \alpha_j} \alpha_k \label{eq: replacing n_0 with 1}\\
		& \leq & \frac{2e^{q_1/2	}}{q_1} C_{\ref{lem: Tn bound}} \|W_1 W_2^{-1}\| C_{\ref{lem: Dn bounds},\theta} \left(\sup_{n_0 \leq k \leq n} e^{-q_1/2\sum_{j = k + 1}^{n} \alpha_j} \frac{\alpha_k}{\beta_k}\|w_k -  \wS\|\right) \label{eq: bounding wk norm}\\
		& \leq & \frac{2e^{q_1/2	}}{q_1} C_{\ref{lem: Tn bound}} \|W_1 W_2^{-1}\| C_{\ref{lem: Dn bounds},\theta} \left(\sup_{n_0 \leq k \leq n} e^{-q_1/2\sum_{j = k + 1}^{n} \alpha_j} \frac{\alpha_k}{\beta_k}u_k\right) \label{eq: bounding T}\\
		&\leq& \frac{2e^{q_1/2}}{q_1} C_{\ref{lem: Tn bound}} \|W_1 W_2^{-1}\| C_{\ref{lem: Dn bounds},\theta} \frac{\alpha_n}{\beta_n} u_n, \label{eq: Tn final step}
		\end{eqnarray}
		where \eqref{eq: bounding the product} follows from Lemma~\ref{lem: Dn bounds}, while in \eqref{eq: replacing n_0 with 1} we bound the summation from $n_0$ to a summation from $0.$ For \eqref{eq: bounding wk norm}, since $\sup_{n}\alpha_n \leq 1,$ we have $\sum_{k = 0}^{n} e^{-q_1/2\sum_{j = k + 1}^{n} \alpha_j} \alpha_k \leq e^{q_1/2} \sum_{k = 0}^{n} e^{-q_1/2\sum_{j = k}^{n} \alpha_j} \alpha_k;$ hence, by treating this latter sum as a Riemann sum and letting $t_{n + 1} = \sum_{k = 0}^n\alpha_k$, we get $\sum_{k = 0}^{n} e^{-q_1/2\sum_{j = k + 1}^{n} \alpha_j} \alpha_k \leq e^{q_1/2} e^{-q_1/2 t_{n + 1}} \int_{0}^{t_{n + 1}} e^{-(q_1/2)t} \mathnormal{d} t = 2e^
		{q_1/2}/2.$ \eqref{eq: bounding T} holds due to $\cW_n.$  
		Lastly, \eqref{eq: Tn final step} holds because the terms in the sup argument in \eqref{eq: bounding T} monotonically increase with $k$, since 
		\begin{equation}\label{eq: wk monotonicity}
		 \frac{ \frac{\alpha_k}{\beta_k}u_k e^{-q_1/2 \sum_{j = k + 1}^{n} \alpha_j}}
		 {
		\frac{\alpha_{k + 1}}{\beta_{k + 1}}u_{k + 1}e^{-q_1/2 \sum_{j = k + 2}^{n} \alpha_j} }
		=
		 \frac{\frac{\alpha_k}{\beta_k} u_k e^{-q_1/2 \; \alpha_{k + 1}}}{
		\frac{\alpha_{k + 1}}{\beta_{k + 1}} u_{k + 1}}
		=
		 \frac{u_k/u_{k + 1}}{
		\frac{\alpha_{k + 1}}{\alpha_k}\frac{\beta_{k}}{\beta_{k + 1}} e^{q_1/2 \; \alpha_{k + 1}}}
		\end{equation}
		is upper bounded by $1$ due to $u$ being $\alpha$-moderate.
		\end{proof}

		\section{Proof of Lemma~\ref{lem:Large Theta n}}
		\label{sec:Proof Lemma Large Theta n}
		
		\begin{lemma}
		\label{lem: Tn bound with 2Rproj}
		Let $n_0 \geq K_{\ref{eq: conditions for cond 2},\alpha}(0)$ (defined in \eqref{eq: k_alpha def}) and $n \geq n_0.$ Let $T_{n+1}$ be as in \eqref{Defn:Tn}. Then, on $\cA_n,$ we have
		\[
		\|T_{n+1}\| \leq \frac{2e^{q_1/2}}{q_1} C_{\ref{lem: Tn bound}} \|W_1 W_2^{-1}\| C_{\ref{lem: Dn bounds},\theta} \frac{\alpha_n}{\beta_n} \Crw \Rproj^{(w)}
		\]
		where $C_{\ref{lem: Tn bound}}$ is defined in \eqref{def: CU} and $\Crw$  is defined in \eqref{def: Crw}.
		\end{lemma}
		\begin{proof}
		Let $u \in \RInf$ be s.t. $\uw_n = \Crw \Rproj^{(w)}~ \forall n \geq 0.$
		Due to Lemma~\ref{eq: conditions for cond 2} Statement~1 (with $z=0$),
		\begin{equation} 
		\frac{\alpha_{n + 1}}{\alpha_n} \frac{\beta_n}{\beta_{n + 1}} e^{q_1/2 \alpha_{n + 1}} \geq 1 \; \forall n \geq 0.
		\end{equation}
		This implies that $u$ is $\alpha$-moderate from $0$ onwards (see Def.~\ref{def: wn behavior 2}). Further, because $\cA_n$ holds, the event $\cW_{n}(u)$ holds. The desired result now follows from  Lemma~\ref{lem: Tn bound}.
		\end{proof}

		
		
		\begin{lemma}
		\label{lem: wn components bounds}
		Let $n \geq n_0$ and suppose the event $\Gp_{n_0} \cap \cA_n$ holds. The following statements are true.
		\begin{enumerate}
		\item If $n_0 \geq 0,$ then
		\begin{equation}
		\|R_{n+1}^{(w)}\|  \leq  \Crth \|\Gamma_2\| C_{\ref{lem: Dn bounds},w} \Rprojth e^{q_2} /q_2.
		\end{equation}
		\item If $n_0 \geq  K_{\ref{lem: const bound on eps}, \beta},$ then 
		\begin{equation}
		\|L_{n+1}^{(w)}\|  \leq  \Rprojw / 2.
		\end{equation}
		\item If $ n_0\geq K_{\ref{lem: small eigenvalues},\beta},$ then 
		\[
		\|\Delta_{n+1}^{(w)}\| \leq  \Rprojw.
		\]
		%
		%
		\item  \label{item: wn bound} Consequently, if $n_0 \geq \max\{K_{\ref{lem: small eigenvalues},\beta}, K_{\ref{lem: const bound on eps}, \beta}\},$ then
		\[
		\|w_{n + 1} - \wS\| \leq \frac{3}{2} \Rprojw + \Crth \|\Gamma_2\| C_{\ref{lem: Dn bounds},w} \Rprojth e^{q_2} /q_2.
		\]
		\end{enumerate}
		\end{lemma}
		\begin{proof}
		Since $\cA_{n}$ holds, it follows from \eqref{def: Rnw} that 
		\begin{eqnarray*}
		\|R_{n+1}^{(w)}\| & \leq & \Crth \|\Gamma_2\| \Rprojth \sum_{k= n_0}^{n} \left(\prod_{j=k+1}^{n}  \|I - \beta_j W_2\| \right) \beta_k \\
		& \leq & \Crth \|\Gamma_2\|C_{\ref{lem: Dn bounds},w} \Rprojth \sum_{k= n_0}^{n} e^{-q_2\sum_{j=k+1}^{n} \beta_j} \beta_k  \\
		& \leq & \Crth \|\Gamma_2\| C_{\ref{lem: Dn bounds},w} \Rprojth e^{q_2} /q_2,
		\end{eqnarray*}
		where the second relation follows by using Lemma~\ref{lem: Dn bounds}, while the last one follows by arguing in the same way as we did for \eqref{eq: bounding wk norm} above.
		
		The bound on $\|L_{n+1}^{(w)}\|$ follows from the definition of $\cA_{n}$ together with Lemma~\ref{lem: const bound on eps}.
		The bound on $\|\Delta_{n+1}^{(w)}\|$ follows from the definition in \eqref{def: Dnw} along with the facts that $\|I - \beta_k W_2\| \leq 1$ and $\|w_{n_0} - \wS\| \leq \Rprojw,$ which themselves hold due to Lemma~\ref{lem: small eigenvalues} and the event $\Gp_{n_0},$ respectively.
		
		The last statement of the lemma follows from the first three statements.
		\end{proof}

		\begin{lemma}
		\label{lem: Rntheta bound}
		Let $n \geq n_0$ and suppose the event $\Gp_{n_0} \cap \cA_n$ holds. The following statements are true. 
		\begin{enumerate}
		\item If $n_0 \geq \max\{K_{\ref{lem: small eigenvalues},\alpha}, K_{\ref{lem: small eigenvalues},\beta}, K_{\ref{eq: conditions for cond 2},\alpha}(0), K_{\ref{lem: const bound on eps}, \beta} \},$ then
		\[
		\|R_{n+1}^{(\theta)}\| \leq \frac{\alpha_{n_0}}{\beta_{n_0}} \left [C_{\ref{lem: Rntheta bound},\theta} \Rprojth + C_{\ref{lem: Rntheta bound},w} \Rprojw\right],
		\]
		where 
		\begin{align}
		    C_{\ref{lem: Rntheta bound},\theta} =& \|W_1\| \|W_2^{-1}\| \frac{e^{q_2}}{q_2}\Crth \|\Gamma_2\| C_{\ref{lem: Dn bounds},w},  \\
		    C_{\ref{lem: Rntheta bound},w} =&  \|W_1\| \|W_2^{-1}\| \left[\frac{5}{2} + \frac{2e^{q_1/2}}{q_1} C_{\ref{lem: Tn bound}}  C_{\ref{lem: Dn bounds},\theta} \Crw \right].
		\end{align}
		
		\item If $n_0 \geq K_{\ref{lem: const bound on eps}, \alpha},$ then
		\[
		    \|\Lnt\| \leq \Rprojth/2.
		\]

		\item If $n_0 \geq k'_{\alpha},$ then $$\|\Dnt\| \leq \Rprojth.$$
		\item \label{item: theta n bound} Consequently, if $n_0 \geq \max\{K_{\ref{lem: small eigenvalues},\alpha}, K_{\ref{lem: small eigenvalues},\beta}, K_{\ref{eq: conditions for cond 2},\alpha}(0), K_{\ref{lem: const bound on eps}, \alpha}, K_{\ref{lem: const bound on eps}, \beta} \},$ then 
		\[
		\|\theta_{n + 1} - \thS\| \leq \frac{3}{2} \Rprojth + \frac{\alpha_{n_0}}{\beta_{n_0}} \left[C_{\ref{lem: Rntheta bound},\theta} \Rprojth + C_{\ref{lem: Rntheta bound},w} \Rprojw\right].
		\]
		\end{enumerate}
		
		\end{lemma}
		\begin{proof}
		On the event $\Gp_{n_0} \cap \cA_{n},$ we have
		\begin{align}
		\|R_{n+1}^{(\theta)} &\| \\
		\leq {} & \frac{\alpha_{n}}{\beta_{n}} \|W_1\| \|W_2^{-1}\| \|w_{n+1} - \wS\| +  \|W_1\| \|W_2^{-1}\| \frac{\alpha_{n_0}}{\beta_{n_0}} \Rprojw + \|T_{n+1}\| \\
		\leq {} & \frac{\alpha_{n}}{\beta_{n}} \|W_1\| \|W_2^{-1}\| \left[\frac{3}{2} \Rprojw + \frac{e^{q_2}}{q_2}\Crth \|\Gamma_2\| C_{\ref{lem: Dn bounds},w} \Rprojth \right] \label{eq: first term in bound on Rn theta}\\
		& \phantom{{} = 1} + \frac{\alpha_{n_0}}{\beta_{n_0}} \|W_1\| \|W_2^{-1}\| \Rprojw  + \frac{2e^{q_1/2}}{q_1} C_{\ref{lem: Tn bound}} \|W_1 W_2^{-1}\| C_{\ref{lem: Dn bounds},\theta}  \frac{\alpha_{n}}{\beta_{n}} \Crw \Rprojw \label{eq: third term in bound on Rn theta}  \\
		\leq {}&  \frac{\alpha_{n_0}}{\beta_{n_0}} \|W_1\| \|W_2^{-1}\| \left[\frac{5}{2} \Rprojw + \frac{e^{q_2}}{q_2}\Crth \|\Gamma_2\| C_{\ref{lem: Dn bounds},w} \Rprojth  + \frac{2e^{q_1/2}}{q_1} C_{\ref{lem: Tn bound}}  C_{\ref{lem: Dn bounds},\theta} \Crw \Rprojw \right] \\
		= {} & \frac{\alpha_{n_0}}{\beta_{n_0}} \left [C_{\ref{lem: Rntheta bound},\theta} \Rprojth + C_{\ref{lem: Rntheta bound},w} \Rprojw\right],
		\end{align}
		where the first relation holds due to Definitions~\eqref{eq: Rn evolution} and \eqref{Defn:Tn} along with the facts that $\|I - \alpha_j X_j\| \leq 1$ and $\|w_{n_0} - \wS\| \leq \Rprojw$ which themselves hold because of  Lemma~\ref{lem: small eigenvalues} and the event $\Gp_{n_0},$ respectively. The second relation holds due to Lemma~\ref{lem: wn components bounds}, Statement~\ref{item: wn bound} and Lemma~\ref{lem: Tn bound with 2Rproj}. The third relation holds because $\alpha_{n}/\beta_{n} \leq \alpha_{n_0}/\beta_{n_0}$ for $n \geq n_0.$
		
		The bound on $\|\Lnw\|$ follows from the definition of $\cA_{n}$ together with Lemma~\ref{lem: const bound on eps}.
		The bound on $\|\Dnt\|$ follows from the definition in \eqref{def: Dnt} along with the facts that $\|I - \alpha_k X_1\| \leq 1$ and $\|\theta_{n_0} - \thS\| \leq \Rprojth,$ which themselves hold due to Lemma~\ref{lem: small eigenvalues} and the event $\Gp_{n_0},$ respectively.
		
		The last statement of the lemma follows from the first three statements. 
		\end{proof}

		

		%
		Let us now define
		\begin{equation}
		\label{defn: Large Theta n constant}
		K_{\ref{lem:Large Theta n}} = \left[\frac{2}{3}C_{\ref{lem: Rntheta bound},\theta} + \frac{2}{3}C_{\ref{lem: Rntheta bound},w}\frac{\Rprojw}{\Rprojth}
		      \right]^{1/(\alpha - \beta)}.
		\end{equation}
		Now we are ready to prove Lemma~\ref{lem:Large Theta n}.
		
		\begin{proof}[Proof of Lemma~\ref{lem:Large Theta n}]
		To get the desired result, it suffices to show that $\|\theta_{n+1} - \thS\| \leq \Crw \Rprojw$ and $\|w_{n+1} - \wS\| \leq \Crth \Rprojth$ on the event $\Gp_{n_0} \cap \cA_{n}.$ 
		
		Assume the event $\Gp_{n_0} \cap \cA_n$ holds. Since $n_0 \geq K_{\ref{lem:Large Theta n}},$ we have $\alpha_{n_0}/\beta_{n_0} \leq \left[\frac{2}{3}C_{\ref{lem: Rntheta bound},\theta} + \frac{2}{3}C_{\ref{lem: Rntheta bound},w}\frac{\Rprojw}{\Rprojth}\right]^{-1}.$ Using this along with the  bound on $\|\theta_{n + 1} - \thS\|$ from Lemma~\ref{lem: Rntheta bound}, item \ref{item: theta n bound}, it is easy to see that $\|\theta_{n + 1} - \thS\| \leq \Crth \Rprojth,$ as desired. The bound on $\|w_{n + 1} - \wS\|$ is straightforward from Lemma~\ref{lem: wn components bounds}, item \ref{item: wn bound}. 
		\end{proof}

		\section{
		Azuma-Hoeffding inequality to bound $L_{n+1}^{(\theta)}$ and $L_{n+1}^{(w)}$
		}
		\label{sec:Proof Azuma Hoeffding}
		
		
		\begin{lemma}
		\label{eq: azuma-hoeffding}
		Let $L_\theta, L_w, a_n$ and $b_n$ be as in Table~\ref{tab: constants contd}. Then,
		\begin{equation}
		\label{eq: azzuma for Lnt}
		\mathbb{P}\left\{\cG_n,  \|{L_{n+1}^{(\theta)}\|} \geq \epsilon \,\middle|\,\Gp_{n_0}\right\} \leq 2d^2 \exp\left(-\frac{\epsilon^2}{d^3 L_\theta a_{n + 1}}\right)
		\end{equation}
		and
		\begin{equation}
		\label{eq: azzuma for Lnw}
		\mathbb{P}\left\{\cG_n,  \|{L_{n+1}^{(w)}\|} \geq \epsilon \,\middle|\, \Gp_{n_0}\right\} \leq 2d^2 \exp\left(-\frac{\epsilon^2}{d^3 L_w b_{n + 1}}\right).
		\end{equation}
		\end{lemma}
		%
		
		\begin{proof}
		Recall the definitions of $\cG_n$ from \eqref{def: Gn} and $\Lnt$ from \eqref{def: Lnt}.
		Let $A_{k,n} \equiv \alpha_k \prod_{j=k+1}^n[I-\alpha_j X_1]$. Then
		
		\begin{eqnarray}
		& & \Pr\left\{\cG_n,  \|{L_{n+1}^{(\theta)}\|} \geq \epsilon \, \middle|\, \Gp_{n_0} \right\} \\
		& = & \Pr\left\{\cG_n, \norm{\sum_{k = n_0}^{n }A_{k, n} \left[- W_1 W_2^{-1} M_{k + 1}^{(2)} + M_{k + 1}^{(1)}\right]} \geq \epsilon \, \middle|\, \Gp_{n_0}\right\}\\
		&  = & \Pr\left\{\cG_n, \norm{\sum_{k = n_0}^{n }A_{k, n} \left[- W_1 W_2^{-1} M_{k + 1}^{(2)} + M_{k + 1}^{(1)}\right] 1_{\cG_k \cap \Gp_{n_0}}} \geq \epsilon \, \middle|\, \Gp_{n_0}\right\}\\
		& \leq & \Pr\left\{\norm{\sum_{k = n_0}^{n }A_{k, n} \left[- W_1 W_2^{-1} M_{k + 1}^{(2)} + M_{k + 1}^{(1)}\right]  1_{\cG_k \cap \Gp_{n_0}}} \geq \epsilon \, \middle|\, \Gp_{n_0}\right\} \label{eqn:PreConcIneqApplVec}\\
		& \leq &  \sum_{i = 1}^{d}\sum_{j = 1}^{d} \Pr\left\{\left|\sum_{k = n_0}^{n } A^{ij}_{k, n} \left[- W_1 W_2^{-1} M_{k + 1}^{(2)} + M_{k + 1}^{(1)}\right]_j 1_{\cG_k \cap \Gp_{n_0}}\right| \geq \frac{\epsilon}{d \sqrt{d}} \, \middle|\, \Gp_{n_0}\right\} \label{eqn:PreConcIneqApplScalar},
		\end{eqnarray}
		where $x_j$ denotes the $j$-th element of the vector $x,$ while $A_{k, n}^{ij}$ is the $ij-$th entry of the matrix $A_{k, n}.$ Our arguments for the last inequality are as follows. First, the term within $\|\cdot\|$ in \eqref{eqn:PreConcIneqApplVec} is a vector, call it $\sX;$ clearly, $\|\sX\| \geq \epsilon$ implies $|\sX_i| \geq \epsilon/\sqrt{d}$ for atleast one coordinate $\sX_i$ of $\sX.$ Next, each $\sX_i$ is itself of the form $\sum_{j = 1}^d \sY_{ij},$ where the scalar $\sY_{ij} = \sum_{k = n_0}^{n} \sZ^{ij}_{k + 1}$ with $\sZ^{ij}_{k + 1} = A^{ij}_{k, n} \left[- W_1 W_2^{-1} M_{k + 1}^{(2)} + M_{k + 1}^{(1)}\right]_j 1_{\cG_k \cap \Gp_{n_0}}.$ Consequently, $|\sX_i| \geq \epsilon/\sqrt{d}$ implies $|\sY_{ij}| \geq \epsilon/(d\sqrt{d})$ for at least one $j.$ Using the union bound, it is now easy to see that \eqref{eqn:PreConcIneqApplScalar} holds, as desired.
		
		
		Let $\Pr'$ denote the probability measure obtained by conditioning $\Pr$ on $\Gp_{n_0}$; that is, $\Pr'(A) = \Pr(A|\Gp_{n_0})$. Then,
		\begin{equation}
		\label{eqn:Change of Measure}
		    \Pr\left\{|\sY_{ij}| \geq \frac{\epsilon}{d\sqrt{d}} \, \middle|\,\Gp_{n_0}\right\} = \Pr'\left\{|\sY_{ij}| \geq \frac{\epsilon}{d\sqrt{d}} \right\}.
		\end{equation}
		We want to bound the RHS using the Azuma-Hoeffding inequality.
		To this end, let $\Exp'$ denote the expectation with respect to $\Pr'$ and $\cF_k':= \cF_k \cap \Gp_{n_0}.$ 
		%
		We now show that $\{M_k\}_{k \geq n_0}$ is Martingale difference sequence w.r.t. $\{\cF'_k\}$ under $\Pr'$.
		Observe that, for all $F \in \cF_k',$
		\begin{align}
		    \int_{F}\Exp'[M_{k + 1}|\cF_k'] \, \df \Pr' 
		    = {} & \Exp'[1_F M_{k+1}] = \Exp[1_{F} M_{k+1}|\Gp_{n_0}]  =  \Exp[1_{F\cap\Gp_{n_0}} M_{k+1}]/\Pr[\Gp_{n_0}] 
		    \\
		    = {} & \frac{1}{\Pr\{\Gp_{n_0}\}} \int_{F \cap \Gp_{n_0}} M_{k + 1} \df \Pr 
		    \\
		    = {} & \frac{1}{\Pr\{\Gp_{n_0}\}} \int_{F \cap \Gp_{n_0}} \Exp[M_{k + 1}|\cF_k] \df \Pr \label{eqn: Relating to E}\\
		    = {} & 0, \label{eqn: Exp under P' equals 0}
		\end{align}
		where \eqref{eqn: Relating to E} follows because $F \cap \Gp_{n_0} \in \cF_k.$ Since \eqref{eqn: Exp under P' equals 0} holds true for all $F \in \cF_k',$ it follows that 
		\begin{equation}
		    \Exp'[M_{k + 1}|\cF_k \cap \Gp_{n_0}] = 0 \quad a.s.
		\end{equation}
		This implies that 
		\begin{multline}
		    \Exp'[\sZ^{ij}_{k + 1}|\cF'_k] = \Exp'\left[A^{ij}_{k, n} \left[- W_1 W_2^{-1} M_{k + 1}^{(2)} + M_{k + 1}^{(1)}\right]_j 1_{\cG_k \cap \Gp_{n_0}} \, \middle| \, \cF'_k\right] \\ = A^{ij}_{k, n} 1_{\cG_k \cap \Gp_{n_0}} \Exp\left[ \left[- W_1 W_2^{-1} M_{k + 1}^{(2)} + M_{k + 1}^{(1)}\right]_j \, \middle| \, \cF_k'\right] = 0. 
		\end{multline}
		Further, observe that 
		\begin{align}
		|\sZ^{ij}_{k + 1}| \\
		= {} & \left| A^{ij}_{k, n}  \left[- W_1 W_2^{-1} M_{k + 1}^{(2)} + M_{k + 1}^{(1)}\right]_j 1_{\cG_k \cap \Gp_{n_0}}\right|\\
		\leq {} & |A^{ij}_{k, n}| \left| \left[- W_1 W_2^{-1} M_{k + 1}^{(2)} + M_{k + 1}^{(1)}\right]_j \right|\\
		\leq {} & \|A_{k, n}\| \left\|- W_1 W_2^{-1} M_{k + 1}^{(2)} + M_{k + 1}^{(1)}\right\|  \label{eq: norm bound}\\
		\leq {} & \|A_{k, n}\|  \left(1 + \Crth \Rprojth + \Crw \Rprojw + \|\thS\| + \|\wS\|\right) \left(m_2 + m_1 \|W_1\| \|W_2^{-1}\|\right) \label{eq: mart. bound}\\
		\leq {} & \alpha_k C_{\ref{lem: Dn bounds},\theta} e^{-q_1 \sum_{j=k+1}^n \alpha_j} \left(1 + \Crth \Rprojth + \Crw \Rprojw + \|\thS\| + \|\wS\|\right) \left(m_2 + m_1 \|W_1\| \|W_2^{-1}\|\right), \label{eq: bound on Akn}
		\end{align}
		where \eqref{eq: norm bound} holds because for matrix $A$, $\max_{i,j}|A^{ij}| \equiv \|A\|_{\max} \leq \|A\|_2$, \eqref{eq: mart. bound} follows from the noise condition (see Defn.~\ref{assum:Noise}) 


		

		Now, applying the Azuma-Hoeffding inequality to the RHS of \eqref{eqn:Change of Measure} and using the fact that $\sum_{k = n_0}^{n} \alpha_k^2 e^{-2q_1 \sum_{j = k + 1}^n \alpha_j} \leq a_{n + 1},$ we obtain \eqref{eq: azzuma for Lnt}.
		
		Repeating the same steps above for $\Lnw$ (see \eqref{def: Lnw}), we obtain the bound in \eqref{eq: azzuma for Lnw}.
		\end{proof}
		
		\section{Proof of Lemma~\ref{lem: bound on wn}}
		\label{sec: Proof of wn Bound}
		
		
		\begin{lemma}
		\label{lem:IntComputation}
		Fix some $n_0 \in \mathbb{N}.$ The following holds for $n \geq n_0 + 1$:
		\begin{enumerate}
		
		\item $\sum_{k = n_0 + 1}^{n}e^{- q_2\sum_{j = k + 1}^{n} \beta_j} \beta_k e^{-q_1\sum_{j = n_0 + 1}^{k}\alpha_j} \leq \frac{2}{q_{\min}} e^{-q_{\min} \sum_{j = n_0+1}^n\alpha_j},$ where $q_{\min} = \min\{q_1,q_2\},$ for $n_0 \geq K_{\ref{lem:IntComputation},a}$ where $K_{\ref{lem:IntComputation},a} = 2^{1/(\alpha - \beta)}.$
		
		\item
		For any $u \in \RInf$ that is $\beta$-moderate from $n_0$ onwards (see Def.~\ref{def: wn behavior 3}),\\
		$\sum_{k = n_0 + 1}^n e^{- q_2\sum_{j = k + 1}^{n} \beta_j} \beta_k  \frac{\alpha_{k - 1}}{\beta_{k - 1}}u_{k - 1} \leq \frac{2e^{q_2/2}}{q_2} \frac{\alpha_{n - 1}}{\beta_{n - 1}}u_{n - 1} .$
		
		\item $
		\sum_{k = n_0 + 1}^{n}e^{- q_2\sum_{j = k + 1}^{n} \beta_j} \beta_k  \epsth_{k - 1} \leq C_{\ref{lem:IntComputation}} \epsth_{n - 1}
		$  for $n_0 \geq K_{\ref{lem:IntComputation},b}$ where $K_{\ref{lem:IntComputation},b} = (3\alpha/q_2)^{1/(1-\beta)} - 2$ and
		$C_{\ref{lem:IntComputation}} = 2e^{q_2/2}/q_2.$
		
		\end{enumerate}
		\end{lemma}
		\begin{proof}
		For the first claim, denote $t_{n + 1} = \sum_{j = 0}^n\alpha_j$ and $s_{n + 1} = \sum_{j = 0}^n\beta_j$. Hence, $t_{n + 1} - t_{n_0} = \sum_{j = n_0}^n\alpha_j$ and $s_{n + 1} - s_{n_0} = \sum_{j=n_0}^n\beta_j$. Clearly,
		
		
		\begin{align}
		&\sum_{k=n_0 + 1}^n e^{-q_2 \left((s_{n + 1} - s_{n_0}) - (s_{k + 1} - s_{n_0})\right) - q_1 (t_{k + 1} - t_{n_0 + 1})} \beta_k \\
		&\leq \sum_{k=n_0 + 1}^n e^{-q_{\min}[(s_{n + 1} - s_{k + 1}) +  (t_{k + 1} - t_{n_0 + 1}) ]} \beta_k \\
		&= e^{-q_{\min} (t_{n + 1} - t_{n_0 + 1})}  \sum_{k = n_0+1}^n e^{-q_{\min}[(s_{n + 1} - s_{k + 1}) - (t_{n + 1} - t_{k + 1}) ]} \beta_k \\
		&\leq e^{-q_{\min} (t_{n + 1} - t_{n_0 + 1})}  \sum_{k = n_0 + 1}^n e^{-q_{\min}(s_{n + 1} - s_{k + 1})/2} \beta_k, \label{eqn:using K6}\\
		&\leq e^{-q_{\min} (t_{n + 1} - t_{n_0 + 1})}  e^{q_{\min}/2} \sum_{k = n_0 + 1}^n e^{-q_{\min}(s_{n + 1} - s_{k})/2} \beta_k, \\
		& \leq e^{-q_{\min} (t_{n + 1} - t_{n_0 + 1})} e^{q_{\min}/2} \int_{s_{n_0 + 1}}^{s_{n + 1}} e^{-q_{\min}(s_{n + 1}  - \tau)/2} \df \tau \label{eqn:RiemannSum}\\
		& \leq e^{-q_{\min} (t_{n + 1} - t_{n_0 + 1})} e^{q_{\min}/2} \left[\frac{1 - e^{-q_{\min}(s_{n + 1}  - s_{n_0 + 1})/2}}{q_{\min}/2}\right] \\
		& \leq \frac{2e^{q_{\min}/2}}{q_{\min}} e^{-q_{\min} (t_{n + 1} - t_{n_0 + 1})}
		\end{align}
		where \eqref{eqn:using K6} holds since, for all $j \geq K_{\ref{lem:IntComputation},a},$  $(s_{n + 1} - s_{k + 1}) - (t_{n + 1} - t_{k + 1}) \geq (s_{n + 1} - s_{k + 1})/2$ which itself holds because $\beta_j/2 \geq \alpha_j,$
		and \eqref{eqn:RiemannSum} follows by treating the sum as a left Riemann sum.

		
		For the second claim, observe that 
		\begin{align}
		\sum_{k = n_0 + 1}^n & e^{- q_2\sum_{j = k + 1}^{n} \beta_j} \beta_k  \frac{\alpha_{k - 1}}{\beta_{k - 1}}u_{k - 1} \\
		& \leq \left(\sup_{n_0 + 1 \leq k \leq n}e^{- (q_2/2)\sum_{j = k + 1}^{n} \beta_j} \frac{\alpha_{k - 1}}{\beta_{k - 1}}u_{k - 1}  \right) \sum_{k = n_0 + 1}^{n} e^{- (q_2/2)\sum_{j = k + 1}^{n} \beta_j} \beta_k \\
		& \leq \left(\sup_{n_0 + 1 \leq k \leq n}e^{- (q_2/2)\sum_{j = k + 1}^{n} \beta_j} \frac{\alpha_{k - 1}}{\beta_{k - 1}}u_{k - 1}  \right) e^{q_2/2} \sum_{k = n_0 + 1}^{n} e^{- (q_2/2)\sum_{j = k}^{n} \beta_j} \beta_k \label{eq: separate the first beta term}\\
		& \leq \left(\sup_{n_0 + 1 \leq k \leq n}e^{- (q_2/2)\sum_{j = k + 1}^{n} \beta_j} \frac{\alpha_{k - 1}}{\beta_{k - 1}}u_{k - 1}  \right) e^{q_2/2} \int_{s_{n_0 + 1}}^{s_{n + 1}} e^{-(q_2/2)(s_{n + 1} - \tau)}\df \tau \label{eq: replace sum by intgral}\\
		& \leq \left(\sup_{n_0 + 1 \leq k \leq n}e^{- (q_2/2)\sum_{j = k + 1}^{n} \beta_j} \frac{\alpha_{k - 1}}{\beta_{k - 1}}u_{k - 1}  \right) \frac{2e^{q_2/2}}{q_2} ,
		\end{align}
		where \eqref{eq: separate the first beta term} follows because $\beta_k \leq 1$, and \eqref{eq: replace sum by intgral} follows by treating the sum as a left Riemann sum.
		In order to obtain the claim, it is now enough to show that the term in the supremum is monotonically increasing. For that we need to show that
		\begin{equation}
		e^{-(q_2/2)\beta_{k + 1}} \frac{\alpha_{k - 1}}{\beta_{k - 1}}u_{k - 1} \leq  \frac{\alpha_{k}}{\beta_{k}}u_{k}.
		\end{equation}
		But this is exactly the $\beta$-moderate behavior, which is assumed true here.
		
		For the third term, observe that 
		\begin{align}
		& \sum_{k = n_0 + 1}^{n} \left[e^{- q_2\sum_{j = k + 1}^{n} \beta_j}\right] \beta_k  \epsth_{k - 1} \\
		& \leq \left(\sup_{n_0 + 1 \leq k \leq n} e^{- q_2/2\sum_{j = k + 1}^{n} \beta_j} \epsth_{k - 1} \right) \sum_{k = n_0 + 1}^n e^{- q_2/2\sum_{j = k + 1}^{n} \beta_j} \beta_k \\
		& \leq \left(\sup_{n_0 + 1 \leq k \leq n} e^{- q_2/2\sum_{j = k + 1}^{n} \beta_j} \epsth_{k - 1} \right) \frac{2e^{q_2/2}}{q_2}, \label{eqn: Term3 Bd}
		\end{align}
		where the last relation follows as in the proof of the second claim.

		As in the second claim, in order to get the desired result, we show that the terms in the supremum expression are monotonically increasing. For this, we only need to verify if 
		\begin{equation}
		\label{eqn:epsMon}
		    \frac{ \epsth_{k - 1}}{\epsth_{k}} \leq e^{(q_2/2) \beta_{k + 1}}.
		\end{equation}
		But this is true since 
		\begin{align}
		\frac{\epsth_{k - 1}}{\epsth_{k}} & \leq \left(\frac{k + 1}{k}\right)^{\alpha/2}  
		\leq e^{\alpha/(2k)} 
		\leq e^{q_2 \beta_{k + 1} /2 }.
		\end{align}
		The first relation follows from \eqref{eq: epsilon n new def} by cancelling out the constants and by dropping the ratio of log terms since the latter is bounded from above by $1.$ The last relation follows from the fact that 
		\begin{equation}
		    \frac{(k + 2)^\beta}{k} = \frac{k + 2}{k} \frac{1}{(k + 2)^{1 - \beta}} \leq 3 \frac{1}{(k + 2)^{1 - \beta}} \leq \frac{q_2}{\alpha},
		\end{equation}
		in which the rightmost inequality itself is true since $n_0 \geq [3\alpha/q_2]^{1/(1 - \beta)} - 2$ and $k \geq n_0$ together imply $(k + 2)^{1 - \beta} \geq [3\alpha/q_2].$
		
		By exploiting the  monotonicity of  \eqref{eqn:epsMon} in \eqref{eqn: Term3 Bd}, the desired result is now easy to see. 
		\end{proof}

		\begin{lemma}
		\label{lem:IdxPullUp}
		The following statements hold. 
		\begin{enumerate}
		    \item $\dfrac{\epsw_n}{\epsw_{n + 1}} \leq e$ for $n \geq 1.$
		    
		    \item Suppose $u \in \RInf$ is $\alpha-$moderate from some $k_0$ onwards and $n \geq k_0 + 1.$ Then, $\dfrac{u_{n - 1}}{u_{n + 1}} \leq e^{q_1}.$

		    \item $ \dfrac{\alpha_{n - 1}}{\beta_{n - 1}} \leq e^{2(\alpha - \beta)} \dfrac{\alpha_{n + 1}}{\beta_{n + 1}}$ for $n \geq 1.$
		\end{enumerate}
		\end{lemma}
		\begin{proof}
		Employing Lemma~\ref{lem:Suff Cond. 2}, with $B_1 = 1, B_2 = [4d^2/\delta]^{1/p}, B_3 = 0,$ $x = - \beta/2, y = 0$ and $z = \beta/2,$ we have
		\begin{equation}
		\frac{\epsw_n}{\epsw_{n + 1}} = \frac{ (n+1)^{-\beta/2}\sqrt{\ln [B_2 (n + 1)]} }{ (n + 2)^{-\beta/2} \sqrt{\ln [B_2 (n + 2)]}} \leq \frac{(n+1)^{-\beta/2}}{(n+2)^{-\beta/2}} = \left[1 + \frac{1}{n + 1}\right]^{\beta/2} \leq e^{\beta/2} \leq e,
		\end{equation}
		where the last relation follows since $\beta/2 \leq 1.$ This proves the first statement.
		
		Now, consider the second statement. Since $u$ is $\alpha$-moderate, $\frac{\alpha_{k + 1}}{\beta_{k + 1}}/\frac{\alpha_k}{\beta_k} \leq 1,$ and $\sup \alpha_{k} \leq 1$, we have $u_n/u_{n + 1} \leq e^{q_1/2}$ for $n \geq k_0.$ The desired result is now easy to see. 
		
		Finally, since
		\begin{equation}
		    \frac{\alpha_{n - 1}}{\beta_{n - 1}} \Big/ \frac{\alpha_{n}}{\beta_{n}} 
		    =  \left[\frac{n+1}{n}\right]^{\alpha - \beta} \leq e^{\alpha - \beta},
		\end{equation}
		it is easy to see that the third statement holds. 
		\end{proof}


		\begin{lemma} \label{lemma: R_n w bound}
		Fix some $n_0 \in \mathbb{N}$ and let $n \geq n_0 - 1.$ Let $u \in \RInf$ be a decreasing sequence that is both $\alpha$-moderate and $\beta$-moderate from $n_0$ onwards (see Defs.~\ref{def: wn behavior 2} and \ref{def: wn behavior 3}).  
		Suppose that the event $\cW_{n}(u)$  holds (see \eqref{def: wn}).
		Then 
		\begin{equation}
		\label{eqn:RntpBd}
		\| \Rntp \| \leq \|W_1 W_2^{-1}\|\left[\left(1+\frac{2e^{q_1/2}}{q_1} C_{\ref{lem: Tn bound}} C_{\ref{lem: Dn bounds},\theta}\right)\frac{\alpha_{n-1}}{\beta_{n-1}}  \uw_{n-1} +  \frac{\alpha_{n_0}}{\beta_{n_0}} \|w_{n_0} - \wS\| C_{\ref{lem: Dn bounds},\theta} e^{-q_1\sum_{j = n_0+1}^{n-1} \alpha_j} \right]
		\end{equation}
		if $n \geq n_0.$
		If, additionally, $\|L_n^{(\theta)}\| \indc[n \geq n_0 + 1] \leq \epsth_{n-1},$ then
		\begin{equation}\label{eqn:thBd}
		    \|\theta_n - \thS\| \leq C_{\ref{lemma: R_n w bound},a} \left[\|\theta_{n_0} - \thS\| + \frac{\alpha_{n_0}}{\beta_{n_0}}\|w_{n_0} - \wS\|\right] e^{-q_1\sum_{j = n_0 + 1}^{n - 1} \alpha_j} + C_{\ref{lemma: R_n w bound},b} \frac{\alpha_{n - 1}}{\beta_{n - 1}}u_{n - 1} +  \epsth_{n-1}
		\end{equation}
		if $n \geq n_0,$ and
		\begin{equation}
		    \|R_{n+1}^{(w)}\| \leq   C_{\ref{lem: Dn bounds},w} \|\Gamma_2\| \left[C_{\ref{lemma: R_n w bound},c}(n_0) e^{-q_{\min} \sum_{j = n_0 + 1}^n\alpha_j} +  C_{\ref{lemma: R_n w bound},b} \frac{2e^{q_2/2}}{q_2} \frac{\alpha_{n - 1}}{\beta_{n - 1}}u_{n - 1} +  C_{\ref{lem:IntComputation}} \epsth_{n-1} \right] \label{eq: Rnw bound}
		\end{equation}
		if $n \geq n_0 - 1\geq\max\{K_{\ref{lem:IntComputation},a}, K_{\ref{lem:IntComputation},b}\}-1.$
		Here, $K_{\ref{lem:IntComputation},a}, K_{\ref{lem:IntComputation},b}, C_{\ref{lem:IntComputation}}$ are defined in the statement of Lemma~\ref{lem:IntComputation},
		$$
		C_{\ref{lemma: R_n w bound},c}(n_0) = \left[\beta_{n_0} \|\theta_{n_0} - \thS\| +  C_{\ref{lemma: R_n w bound},a}  e^{q_1} \frac{2}{q_{\min}}    \left[\|\theta_{n_0} - \thS\| + \frac{\alpha_{n_0}}{\beta_{n_0}}\|w_{n_0} - \wS\|\right]\right],
		$$
		\(C_{\ref{lemma: R_n w bound},a} =  C_{\ref{lem: Dn bounds},\theta} \max\{\|W_1 W_2^{-1}\|, 1\}
		\) and
		\(C_{\ref{lemma: R_n w bound},b} =  \|W_1 W_2^{-1}\|\left(1 + \frac{2e^{q_1/2}}{q_1} C_{\ref{lem: Tn bound}} C_{\ref{lem: Dn bounds},\theta} \right)\).
		\end{lemma}
		{
		\begin{remark}
		Difference between Lemma~\ref{lemma: R_n w bound} and Lemma~\ref{lem: Rntheta bound}:
		\begin{itemize}
		    \item In Lemma~\ref{lemma: R_n w bound}, we assume that $\|w_k - \wS\| \leq u_k$ for all $n \geq k \geq n_0.$ Using this, we try and obtain better rates of convergence for $\|w_k - \wS\|.$ In other words, this is part of our inductive proof where we are showing the $(\ell + 1)$ statement assuming the $\ell-$th step to be true.
		    \item In Lemma~\ref{lem: Rntheta bound}, we establish the base case of the above induction. In particular, we try and show that the iterates are bounded with high probability. In order to prove this, we use another induction on the iterate index which reads as: if the iterates are bounded until time $n,$ what is the bound at the $n + 1-$th step.
		\end{itemize}
		\end{remark}
		}
		
		\begin{proof}[Proof of Lemma~\ref{lemma: R_n w bound}]
		We first establish \eqref{eqn:RntpBd}. Notice from \eqref{def: Rnt} that $R_{n_0}^{(\theta)}=0$ and hence  \eqref{eqn:RntpBd} trivially holds for $n = n_0$. As for $n\geq n_0 + 1$, from \eqref{eq: Rn evolution}
		we have 
		\begin{align}  \label{eq: Rnt bound}
		\| \Rntp \|& \leq \|W_1 W_2^{-1}\|\left[\frac{\alpha_{n - 1}}{\beta_{n - 1}}  \uw_{n} + C_{\ref{lem: Dn bounds},\theta} \frac{\alpha_{n_0}}{\beta_{n_0}} \|w_{n_0} - \wS\|  e^{-q_1\sum_{j = n_0+1}^{n -1} \alpha_j} + \|T_{n}\|\right] \\
		& \leq  
		\|W_1 W_2^{-1}\|\left[\frac{\alpha_{n - 1}}{\beta_{n - 1}}  \uw_{n} + C_{\ref{lem: Dn bounds},\theta} \frac{\alpha_{n_0}}{\beta_{n_0}} \|w_{n_0} - \wS\|  e^{-q_1\sum_{j = n_0+1}^{n -1} \alpha_j} + \frac{2e^{q_1/2}}{q_1} C_{\ref{lem: Tn bound}} C_{\ref{lem: Dn bounds},\theta}  \frac{\alpha_{n-1}}{\beta_{n-1}} \uw_{n - 1}\right] \\
		&\leq \|W_1 W_2^{-1}\|\left[\left(1+\frac{2e^{q_1/2}}{q_1} C_{\ref{lem: Tn bound}} C_{\ref{lem: Dn bounds},\theta}\right)\frac{\alpha_{n - 1}}{\beta_{n - 1}}  \uw_{n - 1} +  \frac{\alpha_{n_0}}{\beta_{n_0}} \|w_{n_0} - \wS\| C_{\ref{lem: Dn bounds},\theta} e^{-q_1\sum_{j = n_0+1}^{n - 1} \alpha_j} \right],
		\end{align}
		where the first relation follows using Lemma~\ref{lem: Dn bounds} and the fact that the event $\cW_{n}(u)$ holds, the second relation is due to Lemma~\ref{lem: Tn bound} (recall that $\uw_n$ is $\alpha$-moderate), while the third relation is due to the fact that $u_n$ monotonically decreases.
		
		We now derive the bound \eqref{eqn:thBd} for $\|\theta_n - \thS\|.$ Since $C_{\ref{lem: Dn bounds},\theta} \geq 1$ implies $C_{\ref{lemma: R_n w bound},a} \geq 1,$ it follows that \eqref{eqn:thBd} trivially holds for $n = n_0.$ As for $n \geq n_0 + 1,$
		\begin{align}
		 \|\theta_n - \thS\|  \leq &\|\Delta_n^{(\theta)}\| + \|R_n^{(\theta)}\| + \|L_n^{(\theta)}\| \\
		 \leq &C_{\ref{lem: Dn bounds},\theta} \|\theta_{n_0} - \thS\| e^{-q_1\sum_{j = n_0}^{n - 1} \alpha_j} + \|W_1 W_2^{-1}\|\left[\left(1 + \frac{2e^{q_1/2}}{q_1} C_{\ref{lem: Tn bound}} C_{\ref{lem: Dn bounds},\theta} \right) \frac{\alpha_{n- 1}}{\beta_{n - 1}} \uw_{n - 1} \right. \nonumber \\
		& \left.+ \frac{\alpha_{n_0}}{\beta_{n_0}} \|w_{n_0} - \wS\|  C_{\ref{lem: Dn bounds},\theta} e^{-q_1\sum_{j = n_0 + 1}^{n - 1} \alpha_j}\right] +  \epsth_{n-1}\\
		 \leq &C_{\ref{lemma: R_n w bound},a}  \left[\|\theta_{n_0} - \thS\| + \frac{\alpha_{n_0}}{\beta_{n_0}} \|w_{n_0} - \wS\| \right] e^{-q_1\sum_{j = n_0 + 1}^{n - 1} \alpha_j} +C_{\ref{lemma: R_n w bound},b} \frac{\alpha_{n - 1}}{\beta_{n - 1}}u_{n - 1} +  \epsth_{n-1}, \label{eq: w.h.p.}
		\end{align}
		where the first relation follows by \eqref{eq: tn components},
		the second one holds on account of \eqref{eqn:K Theta bound} of Lemma~\ref{lem: Delta bounds},  \eqref{eqn:RntpBd}, and our assumption that $\|L_n^{(\theta)}\| \leq \epsth_{n-1},$ while the third relation is obtained by dropping $\alpha_{n_0} = (n_0 + 1)^{-\alpha}$ term from the exponent multiplying $\|\theta_{n_0} - \thS\|.$
		
		Lastly, for the third statement, \eqref{eq: Rnw bound} trivially holds for $n=n_0-1$ since  $R_{n_0}^{(w)} =0 $ by definition \eqref{def: Rnw}. Similarly, for $n=n_0,$ it follows from \eqref{def: Rnw} that
		\begin{equation}
		    \|R^{(w)}_{n_0 + 1}\| \leq \beta_{n_0} \|\Gamma_2\|  \|\theta_{n_0} - \thS\|.
		\end{equation}
		From this and the fact that $C_{\ref{lem: Dn bounds},w} \geq 1,$ it is easy to see that \eqref{eq: Rnw bound} holds again. 
		
		For $n\geq n_0 + 1,$ we break the summation in \eqref{def: Rnw} into the first and the rest of terms; thus, 
		\begin{align}
		\|\Rnw\|& \nonumber \\
		\leq &  \prod_{j = n_0 + 1}^n \|I - \beta_j W_2\| \beta_{n_0} \|\Gamma_2\| \|\theta_{n_0} - \thS\|  + \left\|\sum_{k = n_0 + 1}^n \left[\prod_{j=k+1}^n[I-\beta_j W_2]\right] \beta_k \left[\Gamma_2 (\theta_k - \thS)\right]\right\| \label{eq: Rnw split}\\
		\leq &  \prod_{j = n_0 + 1}^n \|I - \beta_j W_2\| \beta_{n_0} \|\Gamma_2\| \|\theta_{n_0} - \thS\| \nonumber \\
		& + \|\Gamma_2\|\sum_{k = n_0 + 1}^n \left\| \prod_{j=k+1}^n[I-\beta_j W_2]\right\| \beta_k \bigg\{ C_{\ref{lemma: R_n w bound},a}  \left[\|\theta_{n_0} - \thS\| + \frac{\alpha_{n_0}}{\beta_{n_0}}\|w_{n_0} - \wS\|\right] e^{-q_1\sum_{j = n_0 + 1}^{k - 1} \alpha_j}   \nonumber \\
		& \hspace{5cm}+ \, C_{\ref{lemma: R_n w bound},b} \frac{\alpha_{k - 1}}{\beta_{k - 1}}u_{k - 1} +  \epsth_{k-1}\bigg\},  \label{eq: thetan bound} \\
		\leq & C_{\ref{lem: Dn bounds},w} \|\Gamma_2\|\beta_{n_0} \|\theta_{n_0} - \thS\| e^{- q_2\sum_{j = n_0+ 1}^{n} \beta_j}
		\label{eq: subst with Kw 1}\\
		& + C_{\ref{lem: Dn bounds},w}\|\Gamma_2\|\sum_{k = n_0 + 1}^{n}e^{- q_2\sum_{j = k + 1}^{n} \beta_j} \beta_k\bigg\{ C_{\ref{lemma: R_n w bound},a}  \left[\|\theta_{n_0} - \thS\| + \frac{\alpha_{n_0}}{\beta_{n_0}}\|w_{n_0} - \wS\|\right] e^{-q_1\sum_{j = n_0 + 1}^{k - 1} \alpha_j} 
		\label{eq: subst with Kw 2}\\
		& \hspace{5cm}+ \, C_{\ref{lemma: R_n w bound},b} \frac{\alpha_{k - 1}}{\beta_{k - 1}}u_{k - 1} +  \epsth_{k-1}\bigg\} \nonumber \\
		\leq & C_{\ref{lem: Dn bounds},w} \|\Gamma_2\|\beta_{n_0} \|\theta_{n_0} - \thS\| e^{- q_2\sum_{j = n_0+ 1}^{n} \beta_j}\\
		& + C_{\ref{lem: Dn bounds},w}\|\Gamma_2\| \sum_{k = n_0 + 1}^{n}e^{- q_2\sum_{j = k + 1}^{n} \beta_j} \beta_k\bigg\{e^{q_1} C_{\ref{lemma: R_n w bound},a}  \left[\|\theta_{n_0} - \thS\| + \frac{\alpha_{n_0}}{\beta_{n_0}}\|w_{n_0} - \wS\|\right] e^{-q_1\sum_{j = n_0 + 1}^{k } \alpha_j} 
		\label{eq: insert missing alpha}\\
		& \hspace{5cm}+ \, C_{\ref{lemma: R_n w bound},b} \frac{\alpha_{k - 1}}{\beta_{k - 1}}u_{k - 1}  +  \epsth_{k-1}\bigg\}\\
		\leq &  C_{\ref{lem: Dn bounds},w} \|\Gamma_2\|\beta_{n_0} \|\theta_{n_0} - \thS\| e^{- q_2\sum_{j = n_0+ 1}^{n} \beta_j} \label{eq: first term to bound}\\
		& + C_{\ref{lem: Dn bounds},w} \|\Gamma_2\|C_{\ref{lemma: R_n w bound},a} e^{q_1} \frac{2}{q_{\min}}    \left[\|\theta_{n_0} - \thS\| + \frac{\alpha_{n_0}}{\beta_{n_0}}\|w_{n_0} - \wS\|\right]e^{-q_{\min} \sum_{j = n_0 + 1}^n\alpha_j}
		\label{eq: Rnw bound derivation last step}\\
		& + C_{\ref{lem: Dn bounds},w}\|\Gamma_2\|\sum_{k = n_0 + 1}^{n}e^{- q_2\sum_{j = k + 1}^{n} \beta_j} \beta_k\Big[\, C_{\ref{lemma: R_n w bound},b} \frac{\alpha_{k - 1}}{\beta_{k - 1}}u_{k - 1}  +  \epsth_{k-1}\Big] \label{eq: two last terms}
		\end{align}
		where \eqref{eq: thetan bound} follows from  \eqref{eqn:thBd}, \eqref{eq: subst with Kw 1} and \eqref{eq: subst with Kw 2} follow by applying Lemma~\ref{lem: Dn bounds}, \eqref{eq: insert missing alpha} follows because $e^{q_1} e^{-q_1\alpha_k} \geq 1$, and finally \eqref{eq: Rnw bound derivation last step} follows recalling the first statement from Lemma~\ref{lem:IntComputation}. Now, using the second and third statements in Lemma~\ref{lem:IntComputation} (recall that $\uw_n$ is $\beta$-moderate), it is easy to see that the expression in \eqref{eq: two last terms} can be bounded by 
		\begin{equation}
		    C_{\ref{lem: Dn bounds},w} \|\Gamma_2\|\left[ C_{\ref{lemma: R_n w bound},b} \frac{2e^{q_2/2}}{q_2} \frac{\alpha_{n - 1}}{\beta_{n - 1}}u_{n - 1} +  C_{\ref{lem:IntComputation}} \epsth_{n-1} \right].
		\end{equation}
		Since $q_2 \geq q_{\min}$ and $\beta_j \geq \alpha_j$, the term in \eqref{eq: first term to bound} can be bounded by
		\begin{equation}
		    C_{\ref{lem: Dn bounds},w} \|\Gamma_2\| \beta_{n_0} \|\theta_{n_0} - \thS\| e^{-q_{\min} \sum_{j = n_0 + 1}^n \alpha_j}.
		\end{equation}
		Hence, \eqref{eq: first term to bound} to  \eqref{eq: two last terms} can be bounded by 
		\begin{align}
		     C_{\ref{lem: Dn bounds},w} \|\Gamma_2\| \left[C_{\ref{lemma: R_n w bound},c}(n_0) e^{-q_{\min} \sum_{j = n_0 + 1}^n\alpha_j} + C_{\ref{lemma: R_n w bound},b} \frac{2e^{q_2/2}}{q_2} \frac{\alpha_{n - 1}}{\beta_{n - 1}}u_{n - 1} +  C_{\ref{lem:IntComputation}} \epsth_{n-1} \right].
		\end{align}
		This gives the desired result.
		\end{proof}

		We define
		\begin{equation} \label{def: A1 and A2}
		\begin{aligned}
		    A_{1,n_0} &= 
		        e +  \frac{e \Big[C_{\ref{lem: Dn bounds},w}\|\Gamma_2\| C_{\ref{lemma: R_n w bound},c}(n_0) + C_{\ref{lem: Dn bounds},w} \|w_{n_0} - w^*\|\Big]}{\epsw_{n_0}}
		        + 
		        e^2 C_{\ref{lem: Dn bounds},w} \|\Gamma_2\| C_{\ref{lem:IntComputation}},\\
		    A_2 &=  e^{q_1 + 2(\alpha - \beta)} C_{\ref{lem: Dn bounds},w}  \|\Gamma_2\|C_{\ref{lemma: R_n w bound},b} \frac{2e^{q_2/2}}{q_2}.
		\end{aligned}
		\end{equation}

		\begin{proof}[Proof of Lemma~\ref{lem: bound on wn}]
		
		%
		
		Note that \eqref{eqn:thBd Main} follows immediately from \eqref{eqn:thBd}.
		
		Define now
		\begin{eqnarray}
		    A'_1 & = & 1 +  \frac{1}{\epsw_{n_0}} \left[C_{\ref{lem: Dn bounds},w}\|\Gamma_2\| C_{\ref{lemma: R_n w bound},c}(n_0) + C_{\ref{lem: Dn bounds},w} \|w_{n_0} - w^*\|\right]\\
		    A''_1 & = & C_{\ref{lem: Dn bounds},w} \|\Gamma_2\| C_{\ref{lem:IntComputation}}, \label{eq:defn: A''1}
		    \\ A'_2 & = & C_{\ref{lem: Dn bounds},w}  \|\Gamma_2\|C_{\ref{lemma: R_n w bound},b} \frac{2e^{q_2/2}}{q_2},
		\end{eqnarray}
		and observe that
		\begin{equation} 
		    A_{1,n_0} = e A'_1 + e^2 A''_1
		\end{equation}
		and
		\begin{equation}
		    A_2 =  e^{q_1 + 2(\alpha - \beta)} A'_2.
		\end{equation}
		
		For $n\geq n_0, $ observe that 
		%
		%
		\begin{align}
			\|w_{n+1} - w^*\| \leq {} & \|\Dnw + \Rnw +\Lnw  \| \label{eq: begin bound on wn}\\
		 \leq {} &  C_{\ref{lem: Dn bounds},w} \|w_{n_0} - w^*\| e^{-q_2 \sum_{j = n_0}^n \beta_j}  \label{eq: delta bound}\\ & +  C_{\ref{lem: Dn bounds},w} \|\Gamma_2\| \left[C_{\ref{lemma: R_n w bound},c}(n_0) e^{-q_{\min} \sum_{j = n_0 + 1}^n\alpha_j} 
		+ C_{\ref{lemma: R_n w bound},b} \frac{2e^{q_2/2}}{q_2} \frac{\alpha_{n - 1}}{\beta_{n - 1}}u_{n - 1} +  C_{\ref{lem:IntComputation}} \epsth_{n-1} \right] \label{eq: r bound}\\
		& + \epsilon_n^{(w)}  \label{eq: end bound on wn}.
		\end{align}
		Here, the first relation follows from \eqref{eq: wn components}. In the second relation, \eqref{eq: delta bound} follows from Lemma~\ref{lem: Delta bounds}, while \eqref{eq: r bound} follows from Lemma~\ref{lemma: R_n w bound}, third statement. As for \eqref{eq: end bound on wn}, it follows from our assumption that $\|L_n^{(w)}\| \indc[n \geq n_0 + 1] \leq \epsw_{n-1}.$
		
		Because of Lemma~\ref{lem: epsilon n domination}, for $n \geq n_0 \geq K_{\ref{lem: epsilon n domination},b},$ the above relation can be written as:
		\begin{equation}\label{eq: wn bound with exponentials}
		\|w_{n+1} - w^*\| \leq A^\prime_1 \epsw_n + A^{\prime\prime}_1 \epsth_{n-1} + A_2^\prime \frac{\alpha_{n-1}}{\beta_{n-1}} u_{n-1}.
		\end{equation}
		Again, from  Lemma~\ref{lem: epsilon n domination} and the fact $n_0 \geq K_{\ref{lem: epsilon  n domination},a}+1,$ we have  $\epsth_{n-1} \leq \epsw_{n-1}$ for any $n \geq n_0$.
		This implies
		%
		%
		\begin{equation}
		\|w_{n+1} - w^*\| \leq  A_1^\prime \epsw_n + A_1^{\prime\prime}\epsw_{n-1} +  A_2^\prime \frac{\alpha_{n-1}}{\beta_{n-1}} \uw_{n - 1}.
		\end{equation}
		Using Lemma~\ref{lem:IdxPullUp}, since $u$ is $\alpha$-moderate from $n_0  - 1$ onwards, we finally have
		%
		%
		\begin{equation}
		\|w_{n+1} - w^*\| \leq  A_1 \epsw_{n + 1} +  A_2 \frac{\alpha_{n+1}}{\beta_{n+1}} \uw_{n + 1},
		\end{equation}
		
		Lastly, notice that \eqref{eqn: wn Bound eps and u} also holds for $n=n_0-1$ because
		\begin{equation}
		\|w_{n + 1} - w^*\| = \|w_{n_0} - w^*\| \leq  \frac{C_{\ref{lem: Dn bounds},w}}{\epsw_{n_0}} \|w_{n_0} - w^*\| \epsw_{n_0} \leq  A_1 \epsw_{n_0} = A_1 \epsw_{n + 1} \leq A_1 \epsw_{n + 1} +  A_2 \frac{\alpha_{n+1}}{\beta_{n+1}} \uw_{n + 1},
		\end{equation}
		where the second relation holds because $C_{\ref{lem: Dn bounds},w} \geq 1$ by definition, and the third relation due to the definitions of $A_1$ and $A_1'.$
		\end{proof}

		\section{$\alpha-$moderateness, $\beta-$moderateness, and montonicity  of $\{u_n(\ell)\}$}
		\label{sec: Proof un cond 2}
		
		\begin{lemma}
		\label{lem: un satisfies condition 2}
		Let $\ell \leq \beta/[2(\alpha - \beta)]$ and assume that $u_n(\ell)$ is as in \eqref{eq: wn leq u} for some 
		constants $B_1, B_3 \geq 0$ and $B_2 \geq 1$ (these constants may depend on $\ell$), where at least one of $B_1$ and $B_3$ is strictly positive.
		Then, $u_n(\ell)$ is $\alpha$-moderate from $K_{\ref{eq: conditions for cond 2},\alpha}(\beta/2)$ onwards and $\beta$-moderate from $K_{\ref{eq: conditions for cond 2},\beta}(\beta/2)$ onwards. 
		Furthermore, $\{u_n(\ell)\}$ is monotonically decreasing from $e^{1/\beta}/B_2$ onwards.
		\end{lemma}
		\begin{proof}
		It is easy to see from Lemma~\ref{lem:Suff Cond. 2} that, for $z \geq \max\{\beta/2, (\alpha - \beta) \ell\},$ 
		\begin{equation}
		\label{eqn: Verify Cond 2}
		\frac{u_n(\ell)}{u_{n + 1}(\ell)} = \frac{B_1 (n+1)^{-\beta/2}\sqrt{\ln [B_2 (n + 1)]} + B_3 (n+1)^{-(\alpha - \beta) \ell}}{B_1 (n + 2)^{-\beta/2} \sqrt{\ln [B_2 (n + 2)]} + B_3 (n + 2)^{-(\alpha - \beta) \ell}} \leq \frac{(n+1)^{-z}}{(n+2)^{-z}}.
		\end{equation}
		Because $(\alpha - \beta) \ell \leq \beta/2,$ we can pick $z = \beta/2.$ Then, it remains to show that
		\begin{equation}
		\label{eqn: Int Term1 for Cond 2}
		\frac{(n+1)^{-\beta/2}}{(n+2)^{-\beta/2}} \leq \frac{\alpha_{n + 1}}{\alpha_n} \frac{\beta_n}{\beta_{n + 1}} e^{q_1/2 \; \alpha_{n + 1}}.
		\end{equation}
		But this indeed holds for $n\geq K_{\ref{eq: conditions for cond 2},\alpha}(\beta/2)$ due to Lemma~\ref{eq: conditions for cond 2}, Statement~1, since $z=\beta/2 \in [0, 1-(\alpha - \beta)]$. Hence, $u_n(\ell)$ is $\alpha$-moderate. Similarly, in order to establish that $u_n(\ell)$ is $\beta$-moderate, it suffices to show that
		\begin{equation} 
		\frac{(n+1)^{-\beta/2}}{(n+2)^{-\beta/2}} \leq \frac{\alpha_{n + 1}}{\alpha_n} \frac{\beta_n}{\beta_{n + 1}} e^{(q_2/2) \; \beta_{n + 2}},
		\end{equation}
		which indeed holds for $n \geq K_{\ref{eq: conditions for cond 2},\beta}(\beta/2)$ due to Lemma~\ref{eq: conditions for cond 2}, Statement~2.
		
		For monotonicity, let us first rewrite $u_n(\ell)$ in the form  $B_1\sqrt{f(n+1)}+B_3(n+1)^{-(\alpha-\beta)\ell}$, where $f(x) :=x^{-\beta}\ln[B_2x]$. The claimed monotonicity then follows from the monotonicity of $B_3(n+1)^{-(\alpha-\beta)\ell}$ and the fact that, for 
		$x \geq e^{1/\beta}/B_2$, 
		\begin{equation}
		f'(x)
		=
		-\beta x^{-\beta-1}\ln[B_2x]+x^{-\beta-1}
		=
		x^{-\beta-1} \left(1-\beta \ln[B_2x]\right)
		\leq
		0
		,
		\end{equation}
		as desired.
		\end{proof}

		\section{Domination of $\epsth_n, \epsw_n$}
		\label{sec: proof epsilon n domination}
		
		%
		
		\begin{lemma} \label{lem: epsilon n domination}
			The following statements are true. 
			\begin{enumerate}
				\item Let $K_{\ref{lem: epsilon n domination}, a} := ([L_\theta C_{\ref{lem: an bn upper bounds},\theta}]/[L_w C_{\ref{lem: an bn upper bounds},w}])^{1/(\alpha-\beta)}.$ For $n \geq K_{\ref{lem: epsilon n domination}, a},$
				%
		$
				\label{eq: bounding etheta with ew}
				\epsth_n \leq \epsw_n.
				$
				
				\item Let $K_{\ref{lem: epsilon n domination},b} := \left[1+ \frac{\alpha}{2q_{\min}}\right]^{1/(1 - \alpha)}.$ For $n\geq n_0 \geq K_{\ref{lem: epsilon n domination},b},$
				\begin{equation}
				\label{eq: poly vs exp}
				e^{-q_2\sum_{j=n_0}^n \beta_j} \leq {}  e^{-q_{\min}\sum_{j=n_0+1}^n \alpha_j} 
				\leq {} 
				\min\{\epsth_n/\epsth_{n_0}, \epsw_n/\epsw_{n_0}\}.
				\end{equation}

			\end{enumerate}
		\end{lemma}
		\begin{proof}
		The inequality in \eqref{eq: bounding etheta with ew} holds by the definition of $\epsth_n$ and $\epsw_n$ given in \eqref{eq: epsilon n new def}.
		
		The first relation in \eqref{eq: poly vs exp} holds trivially since $q_2 \geq q_{\min}, \alpha_{n_0} \geq 0,$ and $\alpha_j \leq \beta_j~\forall j.$ 
		
		Consider the second relation in \eqref{eq: poly vs exp}. Let $\gamma \in \{\alpha, \beta\}$ (the proof holds for both $\alpha$ and $\beta$).  Substituting $\epsth_{n}$ and $\epsw_{n}$ from \eqref{eq: epsilon n new def} and making use of the fact that $\log(4d^2(n + 1)/\delta)/\log(4d^2(n_0 + 1)/\delta) \geq 1,$ it follows that to prove this second relation we only need to show
		\begin{equation}
		    \label{eqn:exp vs PolyRatio}
		    e^{-q_{\min}\sum_{j=n_0+1}^n \alpha_j} \leq \left[\frac{n + 1}{n_0 + 1}\right]^{-\gamma/2}, \quad n \geq n_0.
		\end{equation}
		
		Observe that $\sum_{j = n_0 + 1}^{n} \alpha_j \geq \int_{n_0 + 1}^{n + 1} \frac{1}{(x + 1)^{\alpha}} \df x = \frac{1}{1 - \alpha} [(n + 2)^{1 -\alpha} - (n_0 + 2)^{1 - \alpha}].$ Hence, to establish \eqref{eqn:exp vs PolyRatio},  it suffices to show that
		\begin{equation}
		    \left[\frac{n + 1}{n_0 + 1}\right]^{-\gamma/2} \geq \exp\left[-q_{\min}[(n + 2)^{1 - \alpha} - (n_0 + 2)^{1 - \alpha}]/(1 - \alpha)\right].
		\end{equation}
		Equivalently, it suffices to show that 
		\begin{equation} \label{eq: poly g log}
		    f(x) := \frac{q_{\min}}{1 - \alpha} \left[(x + 2)^{1 - \alpha} - (n_0 + 2)^{1 - \alpha}\right] 
		    -
		    \frac{\gamma}{2}\left[\ln (x + 1) - \ln (n_0 + 1)\right] \geq 0
		\end{equation}
		for $x \geq n_0$.
		%
		To this end, note that
		\begin{align}
		    f'(x) 
		    =& 
		    q_{\min} (x + 2)^{-\alpha} - \frac{\gamma}{2(x + 1)}
		    \\=&
		    \frac{q_{\min}}{(x + 1)}\left[ \frac{x + 1}{(x + 2)^{\alpha}} - \frac{\gamma}{2q_{\min}}\right]
		    \\=&
		    \frac{q_{\min}}{(x + 1)}\left[ (x + 2)^{1-\alpha}-\frac{1}{(x + 2)^{\alpha}} - \frac{\gamma }{2q_{\min}}\right]
		    \\\geq&
		    \frac{q_{\min}}{(x + 1)}\left[ (x + 2)^{1-\alpha} - 1 - \frac{\gamma }{2q_{\min}}\right] ,
		\end{align}
		which is nonnegative when $x \geq K_{\ref{lem: epsilon n domination},b}.$
		This, combined with the facts that $f(n_0)=0$ and $n_0 \geq K_{\ref{lem: epsilon n domination},b}$ implies \eqref{eq: poly g log} and, thereby, concludes the proof.
		\end{proof}

		\section{Applications to Reinforcement Learning: GTD2 and TDC}
		\label{sec: gtd2 and tdc}
		{Here, we show with which constants Corollary~\ref{cor: RL} can be derived for GTD2 and TDC algorithms. This is done by validating the assumptions required and the constants involved.}
		\subsection{GTD2}
		The GTD2 algorithm~\citep{sutton2009fast} minimizes the objective function
		\begin{align}
		J^{\rm MSPBE}(\theta)
		&=
		\tfrac{1}{2}(b-A\theta)^\top C^{-1}(b-A\theta).
		\label{eq: MSPBE}
		\enspace
		\end{align}
		
		The update rule of the algorithm takes the form of Equations \eqref{eqn:tIter} and \eqref{eqn:wIter} with
		\begin{align*}
		h_1(\theta,w)  &= 
		A^\top w,
		\\
		h_2(\theta,w)  &= 
		b-A\theta -Cw,
		\end{align*}
		and
		\begin{align*}
		\Mt_{n+1} =& \left(\phi_n - \gamma \phi_n'\right)\phi_n^\top w_n - A^\top w_n \enspace,
		\\
		\Mw_{n+1} 
		=& r_n\phi_n + \phi_n[\gamma\phi_n'-\phi_n]^\top\theta_n - \phi_n\phi_n^\top w_n - [b-A\theta_n -Cw_n]\enspace.
		\end{align*}
		For GTD2, the relevant matrices are
		$\Tt = 0$, $\Wt = -A^\top$, $v_1 = 0$, and $\Tw=A$, $\Ww=C$, $v_2=b$.
		Additionally, $\Xt = \Tt - \Wt\Ww^{-1}\Tw = A^\top C^{-1} A$.
		By our assumptions, both $\Ww$ and $\Xt$ are symmetric positive definite matrices, and thus the real part of their eigenvalues are also positive.
		Additionally,
		\begin{eqnarray*}
		\|\Mt_{n+1}\| & \leq& (1+\gamma+\|A\|) \|w_n\| ,\\
		\|\Mw_{n+1}\| & = & \|r_n\phi_n-b+[A+\phi_n(\gamma\phi'_n-\phi_n)^\top]\theta_n - [\phi_n\phi_n^\top-C] w_n\|\\
		& \leq & 1+\|b\| + (1+\gamma+\|A\|)\|\theta_n\| + (1+\|C\|)\|w_n\|.
		\end{eqnarray*}
		Consequently, Assumption~\ref{assum:Noise} is satisfied with constants $\mt = (1+\gamma+\|A\|)$ and $\mw = 1 + \max(\|b\|,\gamma+\|A\|,\|C\|)$.

		\subsection{TDC}
		
		The TDC algorithm is designed to minimize (\ref{eq: MSPBE}), just like GTD2.
		However, its update rule takes the form of Equations \eqref{eqn:tIter} and \eqref{eqn:wIter} with
		\begin{align*}
		h_1\theta(\theta,w) &= 
		b-A\theta + [A^\top-C] w \enspace,
		\\
		h_2(\theta,w) &= 
		b-A\theta -Cw \enspace,
		\end{align*}
		and
		\begin{align*}
		\Mt_{n+1} 
		=& r_n\phi_n + \phi_n[\gamma\phi_n'-\phi_n]^\top\theta_n - \gamma \phi'\phi^\top w_n - [b-A\theta_n + [A^\top-C] w_n]
		\enspace,
		\\
		\Mw_{n+1} 
		=& r_n\phi_n + \phi_n[\gamma\phi_n'-\phi_n]^\top\theta_n - \phi_n\phi_n^\top w_n - [b-A\theta_n +Cw_n]
		\enspace.
		\end{align*}
		Thus, for TDC the relevant matrices in the update rules are
		$\Tt = A$, $\Wt = [C-A^\top]$, $v_1 = b$, and $\Tw=A$, $\Ww=C$, $v_2=b$.
		Additionally, $\Xt = \Tt - \Wt\Ww^{-1}\Tw = A - [C-A^\top] C^{-1} A = A^\top C^{-1} A$.
		By our assumptions, both $\Ww$ and $\Xt$ are symmetric positive definite matrices, and thus the real part of their eigenvalues are also positive.
		Additionally,
		\begin{align*}
		\|\Mt_{n+1}\| \leq&  2+(1+\gamma+\|A\|) \|\theta_n\| +(\gamma+\|A\|+\|C\|) \|w_n\|, \\
		\|\Mw_{n+1}\| =& 2+(1+\gamma+\|A\|) \|\theta_n\|+(1+\|C\|) \|w_n\| \enspace.
		\end{align*}
		As a result, Assumption~\ref{assum:Noise} is satisfied with constants $\mt = (2+\gamma+\|A\|+\|C\|)$ and $\mw = (2+\gamma+\|A\|+\|C\|)$.
		

		\begin{table}
		\begin{center}{
		\small
		 \begin{tabular}{|l|l|l|} 
		 \hline
		 & & \\[-2ex]
		 Parameter/Set & Definition & Source\\ [0.5ex]
		 \hline
		  $\nu(n; \gamma)$ &
		  $ (n+1)^{-\gamma/2} \sqrt{\ln{(4d^2(n+1)^p/\delta)}}$
		 & Section \ref{sec: proof outline} \\
		 \hline
		 $\epsth_n $ & $\sqrt{d^3 L_\theta C_{\ref{lem: an bn upper bounds},\theta}} \, \nu(n, \alpha)$ & \eqref{eq: epsilon n new def}
		 \\
		 \hline
		  $\epsw_n$& $\sqrt{d^3 L_w C_{\ref{lem: an bn upper bounds},w}} \, \nu(n, \beta)$& \eqref{eq: epsilon n new def}
		 \\
		 \hline
		 $u_n(\ell)$&$\left[A_{1,n_0}  \sum_{i=0}^{\ell-1} A_2^i\right] \epsw_n  + \left[A_3 A_2^\ell\right] {\left[\frac{\alpha_n}{\beta_n}\right]}^\ell$ & \eqref{eq: hypothesis}
		 \\
		 \hline 
		 $\ell^*$&$\left\lceil\frac{\beta}{2(\alpha-\beta)} \right\rceil$&\eqref{def: lS}
		 \\
		 \hline 
		 $\cU(n_0)$ & $\bigcap_{n \geq n_0}\big\{\|\theta_n - \thS\| \leq \Crth \Rprojth, \|L_{n + 1}^{(\theta)}\| \leq \epsth_n,  \|w_n - \wS\| \leq \Crw \Rprojw, \|L_{n + 1}^{(w)}\| \leq \epsw_n\big\}$ & \eqref{def: E_n0}
		 \\
		 \hline
		 $\cG^\prime_{n_0}$ & $\{\|\theta_{n_0} - \thS\| \leq \Rprojth, \|w_{n_0} - \wS\| \leq \Rprojw\}$ & \eqref{def: G_n_0'}
		 \\
		 \hline
		 $\cG_n$&$\bigcap_{k = n_0}^n\{\|\theta_k - \thS\| \leq \Crth\Rprojth, \|w_k - \wS\| \leq \Crw\Rprojw\}$&\eqref{def: Gn}
		  \\
		  \hline
		  $\cL_n$&$\bigcap_{k = n_0}^{n} \left\{\|L_{k + 1}^{(\theta)}\| \leq \epsth_k,  \|L_{k + 1}^{(w)}\| \leq \epsw_k\right\}$&\eqref{def: Ln}
		 \\
		 \hline
		 $\cW_n(u)$&$\{\|w_k - \wS\| \leq u_k \; \forall n_0 \leq k \leq n\}$&\eqref{def: wn}
		 \\
		 \hline
		\end{tabular} }
		\end{center}
		\caption{\label{tab: parameters} A summary of parameters and sets}
		\end{table}

		\newpage

		\setlength{\extrarowheight}{9.6pt}
		\begin{longtable}{| p{.080\textwidth} | p{.77\textwidth} | p{.12\textwidth}|} 
		\hline
		 Constant & Definition & Source \\ [0.5ex]
		 \hline
		      $q_1$ & $q_1 \in (0, \lambda_{\min}(X_1 + X_1^\tr)/2)$ & 
		      {\eqref{eq: q1 defn}}
		      \\
		     $q_2$ & $q_2 \in (0, \lambda_{\min}(W_2 + W_2^\tr)/2)$ & 
		     {\eqref{eq: q2 defn}}
		     \\
		     $\mu_1$ &  $ -\lambda_{\min}(X_1 + X_1^\tr) + \lambda_{\max}(X_1^\tr X_1)$ & Lemma~\ref{lem: Dn bounds}\\
		      $\mu_2$ &  $ -\lambda_{\min}(W_2 + W_2^\tr) + \lambda_{\max}(W_2^\tr W_2)$ & Lemma~\ref{lem: Dn bounds}\\
		     $K_{\ref{lem: Dn bounds},1}$ &  $\left(\frac{2\lambda_{\max}(X_1^\tr X_1)}{\lambda_{\min}(X_1 + X_1^\tr)}\right)^{1/\alpha}$ & Lemma~\ref{lem: Dn bounds}\\
		     $K_{\ref{lem: Dn bounds},2}$ &  $\left(\frac{2\lambda_{\max}(W_2^\tr W_2)}{\lambda_{\min}(W_2 + W_2^\tr)}\right)^{1/\beta}$ & Lemma~\ref{lem: Dn bounds}\\
		     $C_{\ref{lem: Dn bounds},\theta}$ &  $\sqrt{
		    \max_{\ell_1 \leq \ell_2 \leq K_{\ref{lem: Dn bounds},1}} \prod_{\ell = \ell_1}^{\ell_2} e^{\alpha_\ell(\mu_1 + 2q_1)}}$ & Lemma~\ref{lem: Dn bounds}\\
		    $C_{\ref{lem: Dn bounds},w}$ &  $\sqrt{
		    \max_{\ell_1 \leq \ell_2 \leq K_{\ref{lem: Dn bounds},2}} \prod_{\ell = \ell_1}^{\ell_2} e^{\alpha_\ell(\mu_2 + 2q_2)}}$ & Lemma~\ref{lem: Dn bounds}\\
		    $K_{\ref{lem: an bn upper bounds}}(p, \hat{q})$ &  $\min\{i|e^{-\hat{q} \sum_{k = 1}^{i - 1}(k + 1)^{-p}} \leq i^{-p}\}$ & Lemma~\ref{lem: an bn upper bounds}\\
		    $C_{\ref{lem: an bn upper bounds}}(p, \hat{q})$ &  $\max_{1 \leq i \leq K_{\ref{lem: an bn upper bounds}}} \{i^p e^{ -\hat{q} \sum_{k = 1}^{i - 1} (k+1)^{-p}}\}$ & Lemma~\ref{lem: an bn upper bounds}\\
		    $C_{\ref{lem: an bn upper bounds},\theta}$ &  $C_{\ref{lem: an bn upper bounds}}(\alpha,q_1) e^{q_1}/q_1$ & 
		    {Lemma~\ref{lem: an bn upper bounds}}
		    \\
		    $C_{\ref{lem: an bn upper bounds},w}$ & $C_{\ref{lem: an bn upper bounds}}(\beta,q_2) e^{q_2}/q_2$ & 
		    {Lemma~\ref{lem: an bn upper bounds}}
		    \\
		    $a_n$ & 
		    {$\sum_{k = 0}^{n - 1} \alpha_k^2 e^{-2q_1 \sum_{j=k+1}^{n - 1} \alpha_j}$}  & 
		    Lemma~\ref{lem: an bn upper bounds}
		    \\
		    ${b_n}$ & 
		    {$\sum_{k = 0}^{n - 1} \beta_k^2 e^{-2q_2 \sum_{j=k+1}^{n - 1} \beta_j}$}  & 
		    {Lemma~\ref{lem: an bn upper bounds}}
		    \\
		    $\Crth$ & 3 & \eqref{def: Crth}
		    \\
		    $\Crw$ &$3/2 + (e^{q_2} /q_2 \|\Gamma_2\| C_{\ref{lem: Dn bounds},w}) \Crth \frac{\Rprojth}{\Rprojw}$& \eqref{def: Crw}
		    \\
		     $L_\theta$ &  $2 \left[C_{\ref{lem: Dn bounds},\theta}  \left(1 + \Crw\Rprojw + \Crth\Rprojth + \|\thS\| + \|\wS\| \right) \left(m_2 + m_1 \|W_1\| \|W_2^{-1}\|\right)\right]^2$ & Lemma~\ref{eq: azuma-hoeffding}\\
		     $L_w$ &  $2 \left[C_{\ref{lem: Dn bounds},w}  m_2 \left(1+ \Crw \Rprojw + \Crth\Rprojth + \|\thS\| + \|\wS\| \right)\right]^2$ & Lemma~\ref{eq: azuma-hoeffding}\\
		      $p$ & $p \in (1,\infty)$ & Section~\ref{sec: proof outline}\\
		      $K_{\ref{lem: small eigenvalues},\alpha}$ & $\min\{i|\alpha_i \leq \frac{\lambda_{\min}(X_1 + X_1^\tr)}{2\lambda_{\max}(X_1^\tr X_1)}\}$ & Lemma~\ref{lem: small eigenvalues}\\
		      $K_{\ref{lem: small eigenvalues},\beta}$ & $\min\{i|\beta_i \leq \frac{\lambda_{\min}(W_2 + W_2^\tr)}{2\lambda_{\max}(W_2^\tr W_2)}\}$ & Lemma~\ref{lem: small eigenvalues}\\
		          $ K_{\ref{eq: conditions for cond 2},\alpha}(z)$ &  $\max\bigg\{\ceil[\bigg]{\left(\frac{q_1}{2(\alpha-\beta + z)}\right)^{1/\alpha}}, \ceil[\bigg]{\left(\frac{4(\alpha - \beta + z)}{q_1}\right)^{1/(1-\alpha)}} \bigg\}$ & Lemma~\ref{eq: conditions for cond 2}\\
		          $ K_{\ref{eq: conditions for cond 2},\beta}(z)$ &  $\max\bigg\{\ceil[\bigg]{\left(\frac{q_2}{\alpha-\beta + z}\right)^{1/\beta}}, \ceil[\bigg]{\left(\frac{4(\alpha - \beta + z)}{q_2}\right)^{1/(1-\beta)}} \bigg\}$ & Lemma~\ref{eq: conditions for cond 2}\\
		        $K_{\ref{lem: const bound on eps}, \alpha}$ & $\left[\frac{4d^3 L_\theta C_{\ref{lem: an bn upper bounds},\theta} p }{\alpha (\Rprojth)^2}\right]^{1/\alpha} \left[2\ln\left(2\frac{4d^3 L_\theta C_{\ref{lem: an bn upper bounds},\theta} p }{\alpha (\Rprojth)^2}\left[\frac{4 d^2}{\delta}\right]^{\alpha/p}\right)\right]^{1/\alpha}$ & Lemma~\ref{lem: const bound on eps}\\
		        $K_{\ref{lem: const bound on eps}, \beta}$ & $\left[\frac{4d^3 L_w C_{\ref{lem: an bn upper bounds},w} p }{\beta (\Rprojw)^2}\right]^{1/\beta}\left[2\ln \left(2\frac{4d^3 L_w C_{\ref{lem: an bn upper bounds},w} p }{\beta (\Rprojw)^2}\left[\frac{4 d^2}{\delta}\right]^{\beta/p}\right)\right]^{1/\beta}$ & Lemma~\ref{lem: const bound on eps}\\
		    $C_{\ref{lem: Rntheta bound},\theta}$ &  $\|W_1\| \|W_2^{-1}\| \frac{e^{q_2}}{q_2}\Crth \|\Gamma_2\| C_{\ref{lem: Dn bounds},w}$ & Lemma~\ref{lem: Rntheta bound}\\
		    $C_{\ref{lem: Tn bound}}$&$ [\|X_1\| + 2(\alpha - \beta) [1 + \|X_1\|]]$ & Lemma~\ref{lem: Tn bound}
		    \\
		    $C_{\ref{lem: Rntheta bound},w}$ &  $\|W_1\| \|W_2^{-1}\| \left[\frac{5}{2} + \frac{2e^{q_1/2}}{q_1} C_{\ref{lem: Tn bound}}  C_{\ref{lem: Dn bounds},\theta} \Crw \right]$ & Lemma~\ref{lem: Rntheta bound}
		    \\      $K_{\ref{lem:Large Theta n}}$ & $ \left[\frac{2}{3}C_{\ref{lem: Rntheta bound},\theta} + \frac{2}{3}C_{\ref{lem: Rntheta bound},w}\frac{\Rprojw}{\Rprojth}
		      \right]^{1/(\alpha - \beta)} $
		      & 
		      \eqref{defn: Large Theta n constant}
		    \\
		    $K_{\ref{lem:IntComputation},a}$ & $2^{1/(\alpha - \beta)}$ &
		    Lemma~{\ref{lem:IntComputation}}
		    \\
		    $K_{\ref{lem:IntComputation},b}$&
		    $(3\alpha/q_2)^{1/(1-\beta)} - 2$ & Lemma~\ref{lem:IntComputation}
		    \\
		    $C_{\ref{lem:IntComputation}}$ & $2e^{q_2/2}/q_2.$  &
		    Lemma~\ref{lem:IntComputation} 
		    \\
		    $C_{\ref{lemma: R_n w bound},a}$&
		    $C_{\ref{lem: Dn bounds},\theta} \max\{\|W_1 W_2^{-1}\|, 1\}$
		    & Lemma~\ref{lemma: R_n w bound}
		    \\
		    $C_{\ref{lemma: R_n w bound},b}$ & $\|W_1 W_2^{-1}\|\left(1 + \frac{2e^{q_1/2}}{q_1} C_{\ref{lem: Tn bound}} C_{\ref{lem: Dn bounds},\theta} \right) $ & Lemma~\ref{lemma: R_n w bound}
		    \\
		    $C_{\ref{lemma: R_n w bound},c}(n_0)$& 
		    $
		    \left[\beta_{n_0} \|\theta_{n_0} - \thS\| +  C_{\ref{lemma: R_n w bound},a}  e^{q_1} \frac{2}{q_{\min}}    \left[\|\theta_{n_0} - \thS\| + \frac{\alpha_{n_0}}{\beta_{n_0}}\|w_{n_0} - \wS\|\right]\right]$
		    & Lemma~\ref{lemma: R_n w bound}
		    \\
		    $A_{1,n_0}$ & $e +  \frac{e \left[C_{\ref{lem: Dn bounds},w}\|\Gamma_2\| C_{\ref{lemma: R_n w bound},c}(n_0) + C_{\ref{lem: Dn bounds},w} \|w_{n_0} - w^*\|\right]}{\epsw_{n_0}}
		    + 
		    e^2 C_{\ref{lem: Dn bounds},w} \|\Gamma_2\| C_{\ref{lem:IntComputation}}$ & 
		    \eqref{def: A1 and A2}
		    \\
		    $A_2$ & $e^{q_1 + 2(\alpha - \beta)} C_{\ref{lem: Dn bounds},w}  \|\Gamma_2\|C_{\ref{lemma: R_n w bound},b} 2e^{q_2/2}/q_2$ & \eqref{def: A1 and A2}
		    \\
		    $A_3$ & $\Crw \Rprojw$ & \eqref{defn: A3}
		    \\
		    $A_{4,n_0}$ & $A_{1,n_0} \sqrt{d^3 L_w C_{\ref{lem: an bn upper bounds},w}} \sum_{i=0}^{\ceil{\frac{\beta}{2(\alpha-\beta)}}-1} A_2^i  + A_3 A_2^{\ceil{\frac{\beta}{2(\alpha-\beta)}}}$ & \eqref{Defn:A4}
		    \\
		    $A_{5,n_0}$ & $\left[C_{\ref{lemma: R_n w bound},a}  \left[\Crth\Rprojth + \Crw \Rprojw \right] /\epsth_{n_0 - 1} + 1\right]\sqrt{4d^3 L_\theta C_{\ref{lem: an bn upper bounds},\theta}} +  C_{\ref{lemma: R_n w bound},b} A_4$ & \eqref{Defn:A5}
		    \\
		    $A_{4,C_1}$ & $d^3 L_w C_{\ref{lem: an bn upper bounds},w}\left(C_{\ref{lem: Dn bounds},w}\|\Gamma_2\| \left[\Rprojth +  C_{\ref{lemma: R_n w bound},a}   \frac{2e^{q_1}}{q_{\min}}    \left(\Rprojth + \Rprojw\right)\right] + C_{\ref{lem: Dn bounds},w} \Rprojw\right)e\sum_{i=0}^{\ceil{\frac{\beta}{2(\alpha-\beta)}}-1} A_2^i$ & \eqref{defn: A4C1}
		    \\
		    $A_{5,C_1}$ & $C_{\ref{lemma: R_n w bound},a}
		    \left[\Crth\Rprojth + \Crw \Rprojw \right]$
		    & \eqref{defn: A5C1}
		    \\
		    $A'_4$ & $A_{4,C_1} + 1$& \eqref{defn: A'4 and A'5}
		    \\
		    $A'_5$ & $2 + A_{5,C_1} + C_{\ref{lemma: R_n w bound},b} A_{4,C_1},$ &\eqref{defn: A'4 and A'5}
		    \\
		    $C_{\ref{thm:Rates Proj Iterates},\theta}$
		    &
		    $ A'_5 {(N_{\ref{thm:Rates Proj Iterates}}+1)^{\alpha/2}}/{\sqrt{\ln{(4d^2(N_{\ref{thm:Rates Proj Iterates}}+1)^p/\delta)}}
			}$
		    &
		    \eqref{eq: Rates Proj Iterates Cth}
		    \\
		    $C_{\ref{thm:Rates Proj Iterates},w}$
		    &
		    $ A'_4 {(N_{\ref{thm:Rates Proj Iterates}}+1)^{\beta/2}}/{\sqrt{\ln{(4d^2(N_{\ref{thm:Rates Proj Iterates}}+1)^p/\delta)}}
			}$
		    &
		    \eqref{eq: Rates Proj Iterates Cw}
		    \\
		    $K_{\ref{lemma: A4' A5'}
		    ,a}$ &
		    $\left[\frac{p }{\beta (A_{4,C_0})^2}\right]^{1/\beta}\left[2\ln \left(2\frac{p }{\beta (A_{4,C_0})^2}\left[\frac{4 d^2}{\delta}\right]^{\beta/p}\right)\right]^{1/\beta}$ &\eqref{defn: CA4'A5'Lemma a}
		    \\
		    $K_{\ref{lemma: A4' A5'},b}$ & $\left[\frac{p }{\alpha (\min\{C_{\ref{lemma: R_n w bound},b}A_{4,C_0},A_{5,C_0}\})^2}\right]^{1/\alpha}\left[2\ln \left(2\frac{p }{\alpha (\min\{C_{\ref{lemma: R_n w bound},b}A_{4,C_0},A_{5,C_0}\})^2}\left[\frac{4 d^2}{\delta}\right]^{\alpha/p}\right)\right]^{1/\alpha}$ & \eqref{defn: CA4'A5'Lemma b}
		    \\
		    $A''_1$ & $C_{\ref{lem: Dn bounds},w} \|\Gamma_2\| C_{\ref{lem:IntComputation}}$ &
		    \eqref{eq:defn: A''1}
		    \\
		    $A_{4,C_0}$ &$A_3 A_2^{\ceil{\frac{\beta}{2(\alpha-\beta)}}} + (e + e^2 A_1'')\sum_{i=0}^{\ceil{\frac{\beta}{2(\alpha-\beta)}}-1} A_2^i \sqrt{d^3 L_w C_{\ref{lem: an bn upper bounds},w}}$ & \eqref{defn:A4C0}
		    \\
		    $A_{5,C_0}$&
		    $\sqrt{4d^3 L_\theta C_{\ref{lem: an bn upper bounds},\theta}}$
		    &\eqref{defn:A5C0}
		    \\
		    $K_{\ref{lem: epsilon n domination}, a}$ &$([L_\theta C_{\ref{lem: an bn upper bounds},\theta}]/[L_w C_{\ref{lem: an bn upper bounds},w}])^{1/(\alpha-\beta)}$ & Lemma~\ref{lem: epsilon n domination}
		    \\
		    $K_{\ref{lem: epsilon n domination},b}$&$\left[1+ {\alpha}/{(2q_{\min})}\right]^{1/(1 - \alpha)}$ & Lemma~\ref{lem: epsilon n domination}
		    \\
		    $K_{\ref{thm:Rates Proj Iterates}, w}$&$[(A'_4/\Rprojw)^{2/\beta}]^{(A'_4/\Rprojw)^{2/\beta}} $ & \eqref{defn: Rates Proj Iterates Kw}
		    \\
		    $K_{\ref{thm:Rates Proj Iterates}, \theta}$&$ [(A'_5/\Rprojth)^{2/\alpha}]^{(A'_5/\Rprojth)^{2/\alpha}}$
		    & \eqref{defn: Rates Proj Iterates Kth}
		    \\
		\hline
		\caption{\label{tab: constants contd} Summary of all constants}
		\end{longtable}
		
		\end{document}